\documentclass[12pt]{article}
\usepackage[utf8]{inputenc}
\usepackage[sectionbib]{natbib}

\usepackage[margin=1.0in]{geometry}
\usepackage{graphicx, color}
\usepackage{booktabs}
\usepackage{xr}
\usepackage{pifont}
\usepackage{titlesec}
\usepackage{multirow}
\usepackage{multicol}
\usepackage{amsfonts}
\usepackage{textcomp}
\usepackage{comment}
\usepackage{amsthm}
\usepackage{mathtools}
\usepackage{algorithm}
\usepackage{algorithmic}
\usepackage{amsmath}
\usepackage{amssymb}
\usepackage[toc,page]{appendix}
\usepackage[mathscr]{euscript}

\let\counterwithin\relax
\usepackage{chngcntr}
\usepackage{apptools}
\usepackage{caption}
\usepackage{tabularx} %
\usepackage{array} %
\usepackage{soul}
\usepackage{tikz}
\usetikzlibrary{shapes.misc}
\usetikzlibrary{shapes}
\usepackage{underscore} 
\AtAppendix{\counterwithin{lemma}{section}}
\usepackage{indentfirst}
\usepackage{epstopdf}
\usepackage{diagbox}
\usepackage{lineno}
\usepackage{lipsum}
\usepackage{subcaption}
\usepackage{bbm}
\usepackage{lscape}
\usepackage{enumitem}
\usepackage{titling}
\newtheorem{assumption}{Assumption}

\newtheorem{definition}{Definition}
\newtheorem{lemma}{Lemma}

\newtheorem{proposition}{Proposition}
\newtheorem{theorem}{Theorem}
\newtheorem{corollary}{Corollary}

\newcolumntype{P}[1]{>{\centering\arraybackslash}p{#1}} %
\newcommand{\norm}[1]{\left\lVert#1\right\rVert}

\usepackage[mathscr]{euscript}
\usepackage{mathtools}
\usepackage{amsthm}
\usepackage{amsfonts}
\usepackage{amssymb}
\usepackage{color}
\usepackage{multirow}
\usepackage{makecell}
\usepackage{booktabs}
\usepackage{tikz}
\newcommand{\bm}{\boldsymbol}

\DeclareMathOperator*{\vectorize}{vec}

\DeclareMathOperator*{\argmin}{arg\,min}

\makeatletter
\newcommand*{\addFileDependency}[1]{
	\typeout{(#1)}
	\@addtofilelist{#1}
	\IfFileExists{#1}{}{\typeout{No file #1.}}
}
\makeatother

\definecolor{blue}{rgb}{0.0,0.0,1.0}
\definecolor{green}{rgb}{0.0,0.5,0.0}
\definecolor{red}{rgb}{1.0,0.0,0.0}
\definecolor{cyan}{rgb}{0.0,1.0,1.0}

\usepackage{lipsum}
\usepackage[colorlinks,linkcolor=blue,citecolor=blue]{hyperref}

\titlespacing*{\section}
{0pt}   
{1.2ex} 
{0.9ex}   

\titlespacing*{\subsection}
{0pt}
{1ex}
{0.7ex}

\newcommand{\Fr}{{\mathrm{F}}}

\definecolor{turquoise}{rgb}{0.03, 0.7, 0.87}

\newcommand{\bbm}{\boldsymbol}

\usetikzlibrary{shapes.geometric, arrows}
\tikzstyle{startstop} = [rectangle, rounded corners, minimum width=3cm, minimum height=1cm,text centered, draw=black, fill=purple!30]
\tikzstyle{stop} = [rectangle, rounded corners, minimum width=3.6cm, minimum height=1cm, draw=black, fill=purple!30, text width = 3.6cm]
\tikzstyle{start} = [rectangle, rounded corners, minimum width=2.7cm, minimum height=1cm, draw=black, fill=purple!30, text width = 2.7cm]
\tikzstyle{io} = [trapezium, trapezium left angle=70, trapezium right angle=110, minimum width=3cm, minimum height=1cm, draw=black, fill=blue!30]
\tikzstyle{process} = [rectangle, rounded corners, minimum width=3cm, minimum height=1cm, draw=black, fill=blue!10, text width=10cm]
\tikzstyle{process_new} = [rectangle, rounded corners, minimum width=3cm, minimum height=1cm, draw=black, fill=blue!10, text width=11.5cm]
\tikzstyle{process2} = [rectangle, rounded corners, minimum width=3cm, minimum height=1cm, draw=black, fill=blue!10, text width=10cm]
\tikzstyle{decision} = [diamond, minimum width=3cm, minimum height=1cm, text centered, draw=black, fill=green!30]
\tikzstyle{arrow} = [thick,->,>=stealth]
\DeclareRobustCommand\sampleline[1]{%
	\tikz\draw[#1] (0,0) (0,\the\dimexpr\fontdimen22\textfont2\relax)
	-- (2em,\the\dimexpr\fontdimen22\textfont2\relax);%
}

\title{ParaRNN: An Interpretable and Parallelizable Recurrent Neural Network for Time-Dependent Data}
\author{Yuxi Cai, Lan Li, Feiqing Huang, Guodong Li
	\\ \textit{Department of Statistics and Actuarial Science, University of Hong Kong}}

\begin{document}
\vspace{-3em}
\maketitle
\begin{abstract}
	The proliferation of large-scale and structurally complex data has spurred the integration of machine learning methods into statistical modeling. Recurrent neural networks (RNNs), a foundational class of models for time-dependent data, can be viewed as nonlinear extensions of classical autoregressive moving average models. Despite their flexibility and empirical success in machine learning, RNNs often suffer from limited interpretability and slow training, which hinders their use in statistics.
	This paper proposes the Parallelized RNN (ParaRNN), a novel model composed of multiple small recurrent units. 
	ParaRNN admits an additive representation that decouples recurrent dynamics into interpretable components, whose behavior can be characterized through recurrence features. This interpretability enables its applications in nonparametric regression for time-dependent data, while the design also allows efficient parallelization.
	The approximation capacity and non-asymptotic prediction error bounds in a nonparametric regression setting are established for ParaRNN. Empirical results on three sequential modeling tasks further demonstrate that ParaRNN achieves performance comparable to vanilla RNNs while offering improved interpretability and efficiency.
\end{abstract}
\textit{Keywords}:  nonparametric regression, parallel computing, recurrence feature, recurrent neural network, time series forecasting

\newpage
\section{Introduction}
Time-dependent data modeling is a longstanding problem in both statistics and machine learning communities. It encompasses a wide range of tasks, such as time series forecasting \citep{flunkert2017deepar}, language modeling \citep{mikolov2010recurrent}, machine translation \citep{sutskever2014sequence}, as well as action and speech recognition \citep{chan2016listen}. Despite addressing similar tasks, statistical and machine learning approaches differ primarily in their modeling objectives. Statistical methods are typically developed for structured data such as time series. They place strong emphasis on model interpretability in addition to predictive performance. In contrast, machine learning models employ highly flexible architectures, leveraging nonlinear activations, input-adaptive design, and deep layered structures. This flexibility enables them to accommodate more complex data and achieve superior predictive accuracy, but often undermines their interpretability and requires considerably larger sample sizes. 
As advances in technology continue to generate large-scale and structurally complex time-dependent data, the advantages of machine learning models become increasingly attractive, calling for their integration into statistics to balance flexibility and interpretability.

Among machine learning models for time-dependent data, the vanilla recurrent neural network (RNN) \citep{elman1990rnn} is the most foundational recurrent architecture. It forms the basis for many important variants, including long short-term memory (LSTM) networks \citep{hochreiter1997lstm}, gated recurrent units (GRU) \citep{cho2014gru}, and many others \citep{chang2018antisymmetricrnn, qiao2019stabilizing, gu2020improving,  erichson2021lipschitz, qin2023hierarchically}.
Recently, the inherent advantages of RNNs' recurrent structures have also been leveraged to address limitations of the popular Transformer architecture \citep{vaswani2017attention}, further expanding the landscape of recurrent models.
From a data modeling perspective, incorporating recurrent dynamics enables more sample-efficient extraction of sequential information \citep{shaw2018self}, leading to a growing trend of imposing RNN structures at the segment level of long sequences \citep{dai2019transformerxl, hutchins2022blockrecurrent} or on attention maps \citep{huang2023encoding}.
From an algorithmic perspective, RNNs enjoy linear scaling in time and memory costs with respect to sequence length, in contrast to the quadratic complexity of attention mechanisms in transformers. These differences have motivated the reformulation of attention into RNN-type designs for better scalability \citep{peng2021random,mao2022fine, peng2023rwkv}. 
The enduring popularity of RNNs underscores their empirical effectiveness in time-dependent data modeling.

Motivated by their empirical success, there has been growing interest in introducing RNNs into statistical modeling \citep{chen2025rnn}. A key challenge in this transfer is achieving interpretability, which depends critically on both model dimension and the choice of activation functions. When the activations are linear, RNNs have direct connections to classical time series models in statistics, including the commonly used autoregressive moving average (ARMA) and generalized autoregressive conditional heteroskedasticity (GARCH) models; see Section \ref{subsec:relation-ts} for details. 
In this setting, the recurrent mechanism of an RNN can be understood through established statistical frameworks, although the interpretation of ARMA models remains an open problem compared to that of simpler autoregressive (AR) models.
When nonlinear activation functions are used, the resulting recurrent dynamics become even more difficult to interpret, especially for large hidden dimensions. Hence, this paper aims to provide an alternative way to elucidate the dynamics of RNNs.

Section \ref{subsec:decouple} first revisits the recurrent dynamics of vanilla RNNs from a structural perspective.
Specifically, an RNN retains memory information by updating a fixed number of hidden states throughout the sequence, and its recurrent dynamics can therefore be characterized by the associated weight matrix, which this paper refers to as the \textit{recurrent matrix}. 
We observe that when this matrix is block diagonal, the hidden state admits an additive decomposition of the hidden states from smaller constituent RNNs, leading to a separation of recurrent dynamics into simpler and more interpretable parts. 
For general recurrent matrices, we may consider their block diagonal approximations in a denseness sense, which motivates a new model whose dynamics are decoupled while retaining expressive power.

\begin{figure}[t]
	\centering
	\begin{subfigure}{.55\textwidth}
		\centering
		\includegraphics[width=.95\linewidth]{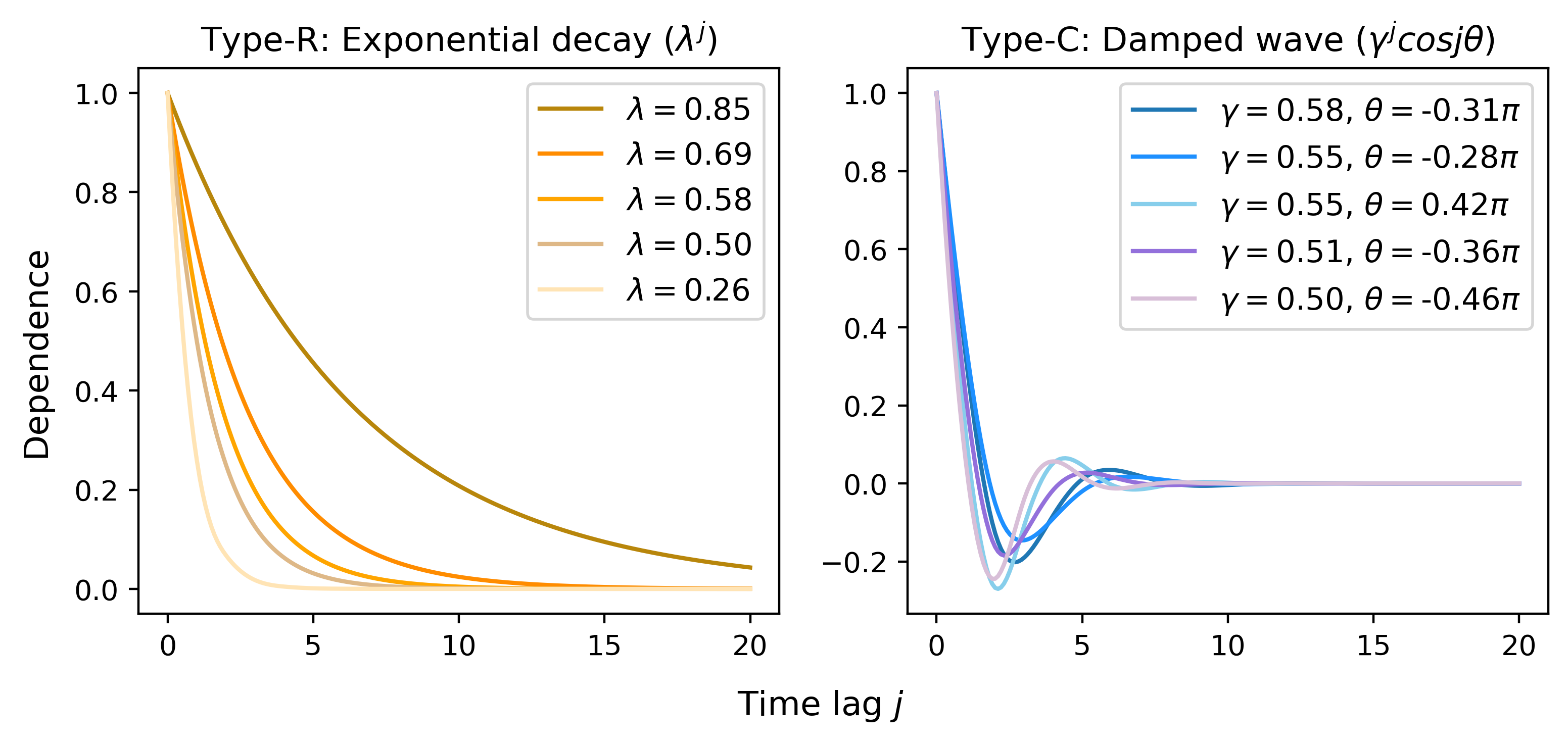}
		\caption{}
	\end{subfigure}%
	\begin{subfigure}{.45\textwidth}
		\centering
		\includegraphics[width=.95\linewidth]{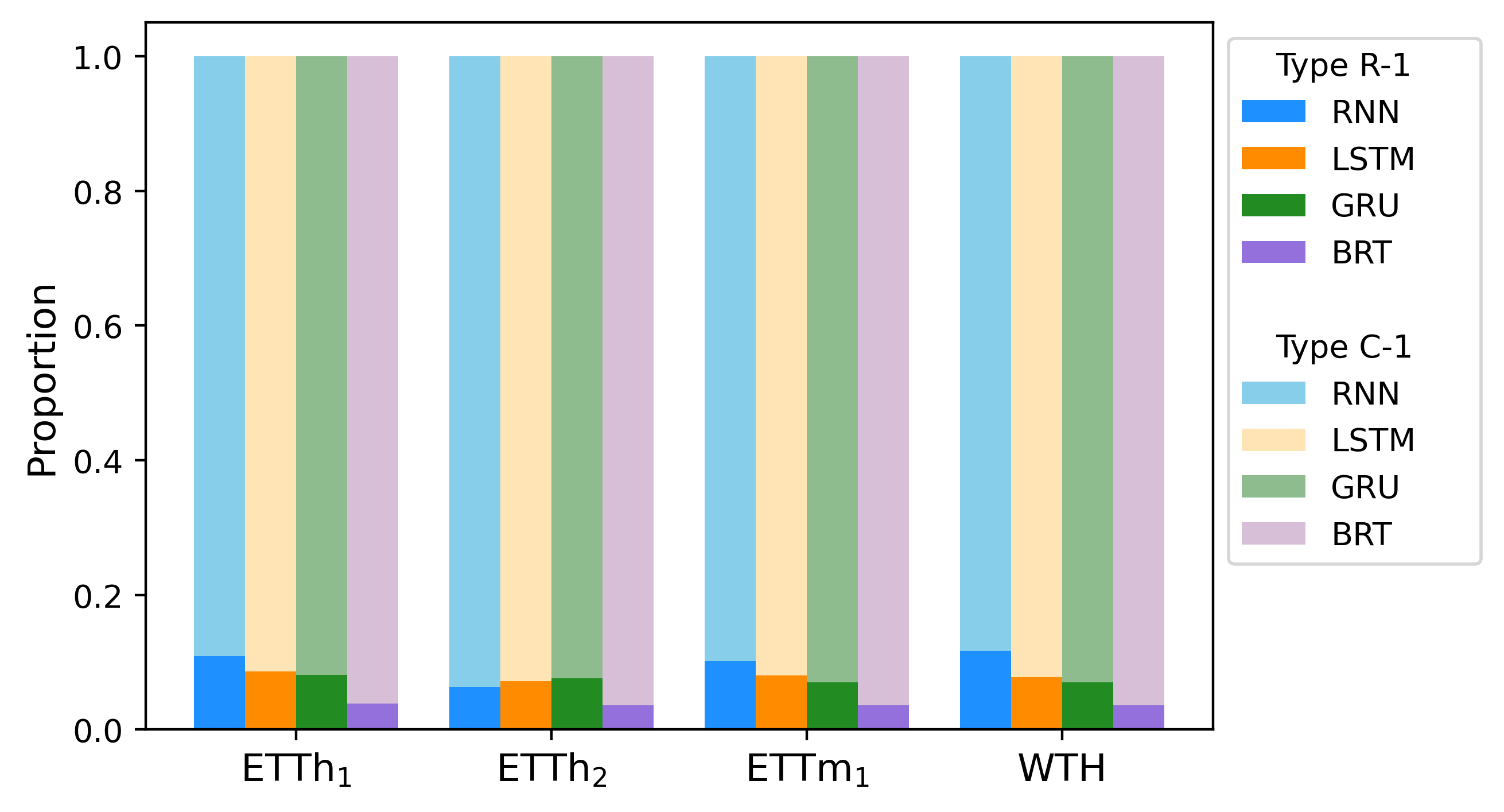}
		\caption{}
	\end{subfigure}
	\vspace{-1.2cm}
	\caption{(a) Two patterns of recurrence features that the vanilla RNN learns from real data: exponential decay (Type R) and damped wave decay (Type C). (b) Proportion of recurrence feature types learned by vanilla RNN, LSTM, GRU and BRT across two data sets, Electricity Transformer Temperature (ETT) and Weather (WTH). All features are first-order.}
	\label{fig:derivative}
\end{figure}

Building upon these observations, we propose an interpretable and computationally efficient recurrent network in Section \ref{sec:pararnn}, referred to as the \textit{Parallelized RNN} (ParaRNN). The model comprises multiple small RNNs operating in parallel, with an additional aggregation step to combine their hidden states. From an interpretability standpoint, the recurrent dynamics of ParaRNN are naturally decomposable. The behavior of each constituent RNN can be analyzed through the real Jordan decomposition \citep{horn2012matrix} of its recurrent matrix, which reveals that recurrent dynamics can be further dissected into a collection of elementary patterns, termed \textit{recurrence features}. Each feature is determined by either a real eigenvalue or a pair of conjugate complex eigenvalues of the recurrent matrix, giving rise to two broad types of temporal behavior. 
When the activation is linear, these types coincide with two renowned patterns in time series analysis \citep{fuller1996intro, cryer2008time}: exponential decay and damped sinusoidal wave decay; see the illustration in Figure \ref{fig:derivative}(a).
This hence sheds new light on ARMA models in statistics.

ParaRNN is also computationally appealing, as vanilla RNNs suffer from slow sequential training. Some previous works focus on parallelizing over the temporal dimension by convolutions \citep{bradbury2016quasirnn} or sequence splitting \citep{dennis2019shallowrnn}; other approaches include substituting the recurrent weight with a diagonal matrix \citep{li2018indrnn, martin2018plrn, rusch2021unicornn} or directly diagonalizing recurrent matrices over the complex plane and then excluding imaginary parts to obtain real-valued outputs \citep{orvieto2023resurrect}. However, they may sacrifice essential information. In particular, the last two approaches lose important features that capture periodic patterns, as suggested by Theorem \ref{thm:dense}.
In contrast, ParaRNN substantially accelerates computation along the hidden states' dimension but without significant loss in recurrence features. 
Our theoretical analysis and numerical studies demonstrate that recurrent dynamics are primarily influenced by low-order recurrence features; see Figure \ref{fig:derivative}(b). As a result, employing a series of small parallel RNNs is sufficient to effectively capture the original RNN's dynamics. 

Beyond interpretability and efficiency, theoretical justification is essential when applying RNNs to statistical modeling. \cite{chen2016generalized} and \cite{tu2020understanding} study the generalization ability; \cite{jiao2024approximation} derives approximation error and convergence properties of nonparametric regression with deep RNNs; and \cite{chen2025rnn} provides theoretical guarantees for RNNs trained on time series generated by nonlinear vector ARMA models. In Section \ref{subsec:theory}, we show that deep ParaRNNs can have the same approximation capacity as deep RNNs, and establish a non-asymptotic upper bound for prediction errors of our models in the context of nonparametric least squares regression. Additional theoretical analyses of RNNs have also been developed outside statistical settings, including expressive power \citep{khrulkov2018expressive}, memory capacity \citep{jasmine2017capacity, haviv2019understand}, and training dynamics \citep{alemohammad2021the, farrell2022gradient, cohen2023learning}.

In summary, this paper makes three main contributions: (i) a novel network, ParaRNN, is proposed to improve interpretability and computational efficiency; (ii) recurrence features are introduced to depict recurrent dynamics, offering insight into ParaRNNs as well as vanilla RNNs and classical ARMA models; (iii) theoretical justifications are established for the proposed methodology. In addition, extensions of our design to other recurrent models are demonstrated in Section \ref{sec:extension}, and their favorable numerical performance is validated through three sequential modeling tasks in Section \ref{sec:result}. Section \ref{sec:conclusion} concludes the paper with a discussion. All theoretical proofs and experiment details are deferred to the Supplementary Material.

\section{Motivation} \label{sec:motivation}
\subsection{Recurrent Neural Networks and Time Series Models}\label{subsec:relation-ts}
In machine learning, recurrent neural networks play a fundamental role in modeling time‑dependent data, whereas in statistics, classical time series models such as AR and ARMA models have long been developed for the same purpose. With the growing adoption of machine learning methods in statistical applications, it is therefore of interest to understand the relationship between these two modeling approaches.

Consider a general recurrent neural network layer with inputs $\bm{x}_t \in \mathbb{R}^{d_{\mathrm{in}}}$ and hidden states $\bm{h}_t \in \mathbb{R}^{d}$ where $1\leq t \leq T$. It has the form of
\begin{linenomath}
	\begin{align}
		\bm{h}_t=\sigma_h(\bm{W}_h\bm{h}_{t-1}+\bm{W}_x\bm{x}_t + \bm{b}),
		\label{eq:nonlinear-rnn}
	\end{align}
\end{linenomath}
where $\sigma_h(\cdot)$ is an element-wise activation function, $\bm{W}_h\in \mathbb{R}^{d\times d}$ and $\bm{W}_x \in \mathbb{R}^{d \times d_{\mathrm{in}}}$ are weight or parameter matrices, and $\bm{b}$ is the bias term.
In fact, this form encompasses many important time series models in statistics, such as ARMA and GARCH models.

We first verify that the ARMA model is a special case of model \eqref{eq:nonlinear-rnn}. Specifically, for a time series $\{\bm{y}_t\}$ with $\bm{y}_t \in \mathbb{R}^{d_\mathrm{in}}$, the ARMA(1,1) model and its AR$(\infty)$ form are given by
\begin{equation}
	\bm{y}_t = \bm{\Phi} \bm{y}_{t-1} + \bm{\epsilon}_t - \bm{\Theta} \bm{\epsilon}_{t-1} 
	\quad \text{or} \quad
	\bm{y}_t = \bm{h}_t +  \bm{\epsilon}_t \hspace{2mm}\text{with}\hspace{2mm}
	\bm{h}_t= \sum_{j=0}^{\infty} \bm{\Theta}^j (\bm{\Phi} - \bm{\Theta} )\bm{y}_{t-1-j} ,
	\label{eq:arma}
\end{equation}
respectively, where $\bm{\Phi}, \bm{\Theta} \in \mathbb{R}^{d_\mathrm{in} \times d_\mathrm{in}}$ are parameter matrices, $\bm{\epsilon}_t \in \mathbb{R}^{d_\mathrm{in}}$ is white noise, and the invertibility condition holds, i.e., the spectral radius of $\bm{\Theta}$ is less than one; see \cite{cryer2008time}.
Moreover, denote by $\mathcal{F}_t$ the $\sigma$-field generated by $\bm{y}_s$ with $s\leq t$, and it holds that $\bm{h}_t=\mathbb{E}(\bm{y}_t|\mathcal{F}_{t-1})$, i.e., model \eqref{eq:arma} attempts to fit $\bm{h}_t$.
When conducting estimation, the initial values of $\bm{y}_t$ with $t\leq 0$ are not observable, and the common practice is to set them as constants, say zeros.
As a result, after reparameterization of $\bm{W}_h =  \bm{\Theta}$ and $\bm{W}_x = \bm{\Phi} - \bm{\Theta}$, we have $\bm{h}_t=\sum_{j=0}^{t-1}\bm{W}_h^j\bm{W}_x\bm{x}_{t-j}\in\mathbb{R}^{d_\mathrm{in}}$ with $\bm{x}_t = \bm{y}_{t-1}$.
In the meanwhile, under the special case $\sigma_h(\bm{x})=\bm{x}$ and $\bm{b}=\bm{0}$, i.e., a linear RNN without bias, the output at \eqref{eq:nonlinear-rnn} can be rewritten into the same form.
Moreover, any ARMA$(p, q)$ model can be recast in a form with $p=q=1$, implying that RNNs are expressive enough to represent arbitrary ARMA models.

Note that all neural networks, including RNNs, from machine learning are for modeling only. From the perspective of time series modeling, RNNs are more general than ARMA models by including nonlinearities.  This additional flexibility enables RNNs to capture more complex temporal structures and can improve predictive performance, but it substantially obscures model interpretation, especially when the hidden dimension $d$ is large. This trade-off is analogous to the relationship between AR and ARMA models: while ARMA models often outperform AR models in terms of prediction by allowing richer autocorrelation structures, their interpretation is frequently complicated by a general form of $\bm{\Theta}$ and even nontrivial identification constraints \citep{chan2016large, wang2022high}. Therefore, when using RNNs in statistical contexts, it is important to explicitly account for issues of interpretability.


On the other hand, we consider the VECH-GARCH(1,1) model \citep{bollerslev1988capital} for a time series $\{\bm{y}_t\}$ with $\bm{y}_t \in \mathbb{R}^{d_\mathrm{out}}$. The model is specified as
\begin{align*}
	\bm{y}_t = \bm{H}_t^{1/2} \bm{\epsilon}_t
	\quad
	\text{with}\quad\bm{h}_t  = \text{vech}(\bm{H}_t) = \bm{\Theta} \text{vech}(\bm{y}_{t-1} \bm{y}_{t-1}^\top) + \bm{\Phi}\text{vech}(\bm{H}_{t-1}) + \bm{b},
\end{align*}
where the innovations $\{\bm{\epsilon}_t\}$ are independent and identically distributed with standard multivariate normality, 
$\bm{H}_t=\mathbb{E}(\bm{y}_{t} \bm{y}_{t}^\top|\mathcal{F}_{t-1}) \in \mathbb{R}^{d_\mathrm{out} \times d_\mathrm{out}}$ is the conditional variance matrix, and $\text{vech}(\cdot)$ denotes the half-vectorization operator that stacks the columns of the lower triangular part of a symmetric matrix.
Accordingly, the hidden state satisfies $\bm{h}_t\in\mathbb{R}^{d_\mathrm{out}(d_\mathrm{out} + 1)/2}$, and the parameter matrices $\bm{\Phi}$ and $\bm{\Theta}$ lie in $\mathbb{R}^{d_\mathrm{out}(d_\mathrm{out} + 1)/2 \times d_\mathrm{out}(d_\mathrm{out} + 1)/2}$.
By letting $d = d_\mathrm{in} = d_\mathrm{out}(d_\mathrm{out} + 1)/2$,  $\bm{W}_h = {\bm{\Phi}}$, $\bm{W}_x = {\bm{\Theta}}$, and $\bm{x}_t = \text{vech}(\bm{y}_{t-1} \bm{y}_{t-1}^\top)$, this model can be rewritten as the RNN in \eqref{eq:nonlinear-rnn} with a linear activation function. The resulting model can be trained by specifying a Gaussian likelihood loss and, by the same recasting principle, extends naturally to general GARCH$(p,q)$ models.
More broadly, there are many other multivariate GARCH models in the literature, and most of them can also be represented as RNNs whose weight matrices possibly have special low-dimensional structures.

\subsection{Decoupling Recurrent Dynamics}\label{subsec:decouple}
For both RNNs and time series models in the previous subsection, the matrix $\bm{W}_h$ governs the recurrent dynamics by controlling how subsequent hidden states are shaped by previous ones, and we therefore refer to it as the \textit{recurrent matrix}. Compared with AR models, the recurrent matrix in ARMA models enables persistent dependence, i.e., current observations can depend on those far in the past, which leads to significantly better forecasting performance in real applications \citep{cryer2008time}.
However, this increased expressive power also poses substantial challenges for interpreting fitted ARMA models, especially when $\bm{W}_h$ has a large size, and these difficulties are further exacerbated for RNNs.
Motivated by this issue, this subsection attempts to decompose the recurrent matrix such that we can break the recurrent dynamics of RNNs into simpler and more interpretable components.


We start by examining a special scenario where the recurrent matrix is block diagonal, i.e., $\bm{W}_h = \oplus_{k=1}^{K} \bm{W}_h^{(k)}\in\mathbb{R}^{d\times d}$, with $\oplus$ denoting the matrix direct sum, $\bm{W}_h^{(k)} \in \mathbb{R}^{d_k \times d_k}$ and $d=\sum_{k=1}^{K} d_k$. Then the RNN at \eqref{eq:nonlinear-rnn} can be decomposed into a series of smaller RNNs,
\begin{linenomath}
	\begin{align}
		\bm{h}_t^{(k)} = \sigma_h(\bm{W}_h^{(k)} \bm{h}_{t-1}^{(k)} + \bm{W}_x^{(k)} \bm{x}_t + \bm{b}^{(k)}) \in \mathbb{R}^{d_k}
		\quad \text{and} \quad
		\bm{h}_t = \text{Concat}[\bm{h}_t^{(1)}, \ldots,\bm{h}_t^{(K)}],
		\label{eq:nonlinear-rnn-diag}
	\end{align}  
\end{linenomath}
where Concat[$\cdot$] denotes concatenation, and $\bm{W}_x^{(k)}$ and $\bm{b}^{(k)}$ are obtained by partitioning $\bm{W}_x$ and $\bm{b}$ along rows, i.e., $\bm{W}_x^{(k)} = (\bm{W}_x)_{a_{k-1}+1: a_k, :}$, $\bm{b}^{(k)} = \bm{b}_{a_{k-1}+1: a_k}$, $a_0=0$, and $a_k = \sum_{i=1}^{k}d_i$ for $k>0$. In each constituent RNN, hidden neurons interact solely with those from the previous time step and the same group, remaining independent across groups.
In other words, the block diagonality of the recurrent matrix offers a way to separate recurrent dynamics.
Moreover, the hidden state $\bm{h}_t$ at \eqref{eq:nonlinear-rnn-diag} admits an additive representation, 
\begin{equation}
	\bm{h}_t = \sum_{k=1}^K f_{\text{RNN}}^{(k)} (\mathcal{F}_{t}; \bm{\Psi}^{(k)})
	\quad
	\text{with}
	\quad
	f_{\text{RNN}}^{(k)} (\mathcal{F}_{t}; \bm{\Psi}^{(k)}) = (\bm{0}_{d_k \times a_{k-1}}, \bm{I}_{d_k \times d_k}, \bm{0}_{d_k \times (d - a_k)})^{\top}\bm{h}_t^{(k)}.
	\label{eq:rnn-additive}
\end{equation}
The mapping $f_{\text{RNN}}^{(k)} (\mathcal{F}_{t}; \bm{\Psi}^{(k)})$ is induced by the $k$-th smaller RNN with parameter set $\bm{\Psi}^{(k)} = \{\bm{W}_h^{(k)}, \bm{W}_x^{(k)}, \bm{b}^{(k)}\}$. 
This additive structure brings benefits in both interpretability and computational efficiency. The RNN can be viewed through the lens of generalized additive and index models \citep{yuan2011on, chen2016generalized}, reducing interpretation from a single entangled recurrent mechanism to the analysis of temporal dynamics for each component separately. Meanwhile, the functions $f_{\text{RNN}}^{(k)} (\mathcal{F}_{t}; \bm{\Psi}^{(k)})$ can be evaluated in parallel, yielding computational gains that are unavailable for general RNNs with dense recurrent matrices.

In general, the recurrent matrix $\bm{W}_h$ is not necessarily block diagonal, and we may try to find its close approximation that has the block diagonal form, thereby allowing the recurrent dynamics to be deliberately disentangled.
Specifically, let $\mathbb{M}_d$ be the set of all $d \times d$ real matrices with rank at least two, and assume that $d$ is even without loss of generality. Denote similarity over $\mathbb{R}$ by $\sim$, and define the following subsets of $\mathbb{M}_d$: $\mathbb{M}_d^1 \coloneqq \{\bm{W} \in \mathbb{M}_d: \bm{W} \sim \text{Diag}(w^{(1)}, w^{(2)}, \ldots, w^{(d)}), w^{(k)}\in \mathbb{R}, 1\leq k \leq d\}$ and $\mathbb{M}_d^2 \coloneqq  \{\bm{W} \in \mathbb{M}_d: \bm{W} \sim \oplus_{k=1}^{d/2} \bm{W}^{(k)}, \bm{W}^{(k)} \in \mathbb{R}^{2 \times 2}\}$, where $\text{Diag}(\cdot)$ takes its arguments to construct a diagonal matrix.
\begin{theorem}
	\label{thm:dense}
	It holds that
	$\mathbb{M}_d^2$ is dense in $\mathbb{M}_d$, while $\mathbb{M}_d^1$ is not dense in $\mathbb{M}_d$.
\end{theorem}
Theorem \ref{thm:dense} implies that any matrix in $\mathbb{M}_d$ can be approximated arbitrarily closely by a block diagonal alternative with $2\times 2$ real blocks. The approximation is considered under similarity transformations which are standard in matrix analysis as they preserve linear dynamics  \citep{horn2012matrix}.

In fact, the block size of the alternative matrix can be generalized beyond $2\times 2$. Let $d_1, d_2, \ldots, d_K$ be even integers that satisfy $d = \sum_{k=1}^{K}d_k$, and define $\mathbb{M}_d^{\mathrm{even}} \coloneqq  \{\bm{W} \in \mathbb{M}_d: \bm{W} \sim \oplus_{k=1}^{K} \bm{W}^{(k)}, \bm{W}^{(k)} \in \mathbb{R}^{d_k \times d_k}\}$. Since $\mathbb{M}_d^2\subset \mathbb{M}_d^{\mathrm{even}} \subset \mathbb{M}_d$, and $\mathbb{M}_d^2$ is dense in $\mathbb{M}_d$, it follows that $\mathbb{M}_d^{\mathrm{even}}$ is also dense in $\mathbb{M}_d$. Thus, block diagonal matrices with even-dimensional blocks suffice to approximate a general matrix,
and we can directly consider $\bm{W}_h \in \mathbb{M}_d^{\mathrm{even}}$ for \eqref{eq:nonlinear-rnn}.
As a result, there exist an invertible matrix $\bm{B} \in \mathbb{R}^{d\times d}$ and a block diagonal matrix $\tilde{\bm{W}}_h = \oplus_{k=1}^{K}\tilde{\bm{W}}_h^{(k)}$ such that $\bm{W}_h = \bm{B} \tilde{\bm{W}}_h \bm{B}^{-1}$ and
\begin{linenomath}
	\begin{align}
		\bm{h}_t &= \sigma_h[\bm{B}(\tilde{\bm{W}}_h\bm{B}^{-1}\bm{h}_{t-1}+\bm{B}^{-1}\bm{W}_x\bm{x}_t + \bm{B}^{-1}\bm{b})] \notag \\
		&\approx \bm{B}\sigma_h(\tilde{\bm{W}}_h \bm{B}^{-1}\bm{h}_{t-1}+\bm{B}^{-1}\bm{W}_x\bm{x}_t + \bm{B}^{-1}\bm{b}),
		\label{eq:approx}
	\end{align}
\end{linenomath}
where the second step approximates the dynamics by moving the matrix $\bm{B}$ outside the activation function, and it becomes mathematically equivalent for linear activation functions. 
However, for nonlinear activation functions, the equivalence does not hold in general.
Nonetheless, it may serve as a useful modeling compromise to motivate a RNN surrogate: 
\begin{linenomath}
	\begin{align}
		\tilde{\bm{h}}_t=\sigma_h(\tilde{\bm{W}}_h\tilde{\bm{h}}_{t-1}+\tilde{\bm{W}}_x\bm{x}_t + \tilde{\bm{b}})
		\quad \text{and} \quad
		\bm{h}_t = \bm{B}\tilde{\bm{h}}_t,
		\label{eq:rnn-approx-main}
	\end{align}
\end{linenomath}
where $\tilde{\bm{W}}_x = \bm{B}^{-1}\bm{W}_x $ and $\tilde{\bm{b}}= \bm{B}^{-1}\bm{b}$. The recurrent matrix of this surrogate model has the desirable block diagonal form, inducing an additive model structure similar to \eqref{eq:rnn-additive} and enabling interpretation and parallelization through decoupled recurrent components.

\section{Parallelized Recurrent Neural Network}\label{sec:pararnn}
\subsection{Model Formulation} \label{subsec:pararnn-architecture}
\begin{figure}[t]
	\centering
	\begin{subfigure}{.3\textwidth}
		\centering
		\includegraphics[width=.92\linewidth]{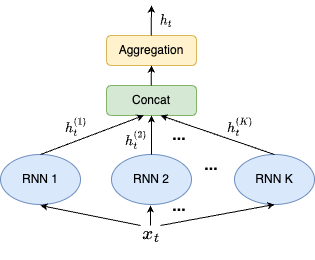}
		\caption{}
	\end{subfigure}%
	\begin{subfigure}{.3\textwidth}
		\centering
		\includegraphics[width=.98\linewidth]{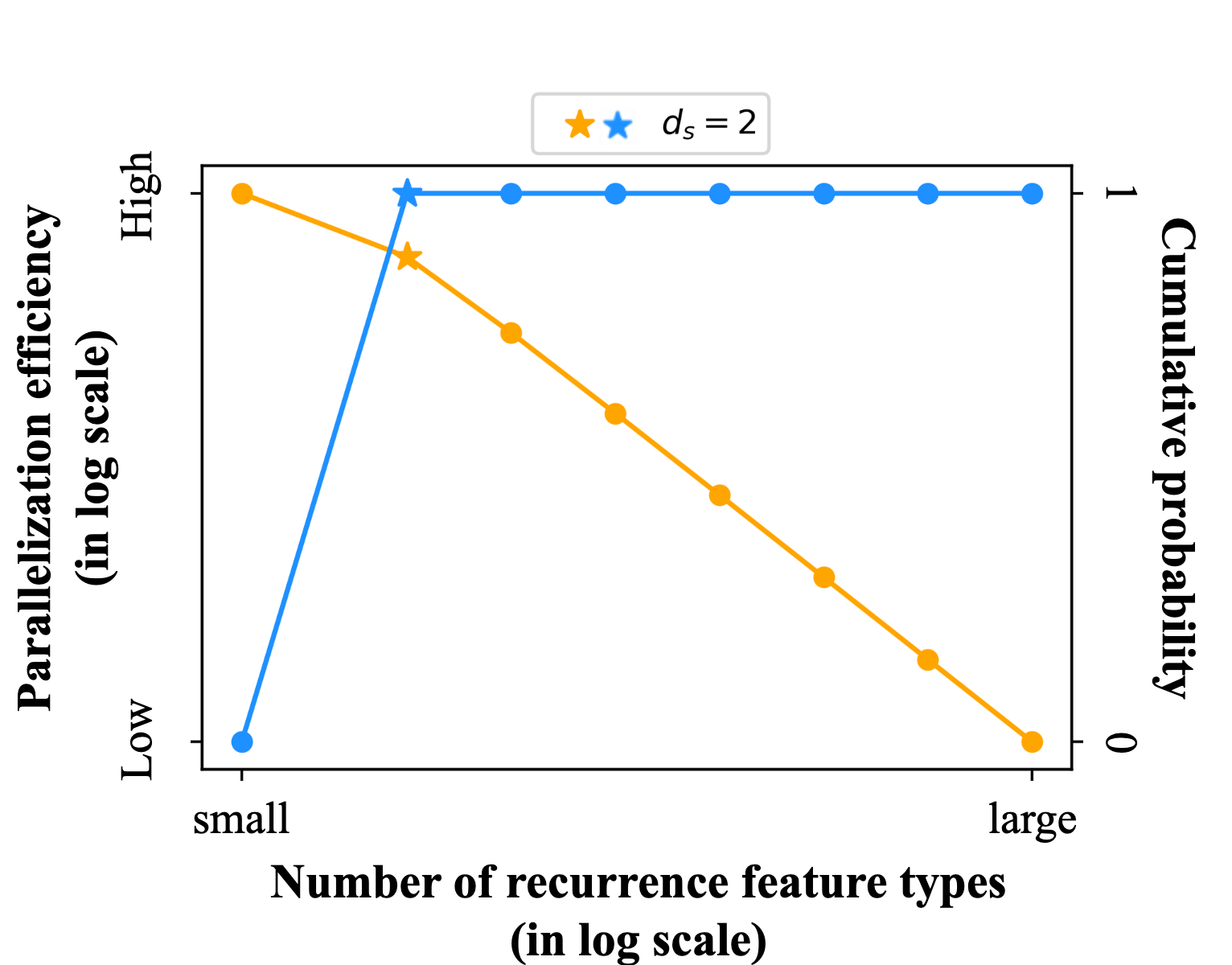}
		\caption{}
	\end{subfigure}%
	\begin{subfigure}{.33\textwidth}
		\centering
		\includegraphics[width=1.02\linewidth]{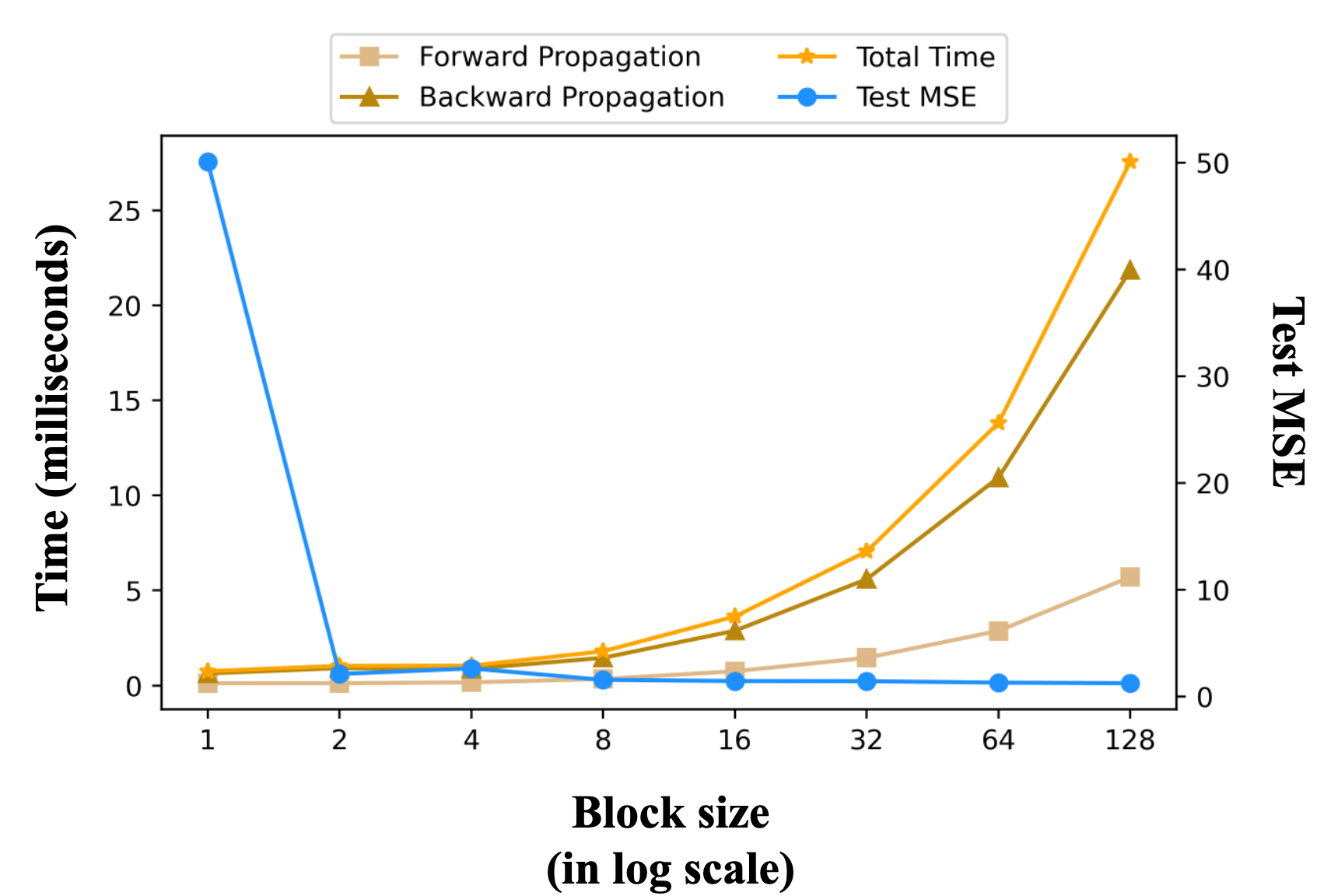}
		\caption{}
	\end{subfigure}
	\caption{(a) A ParaRNN layer at time $t$. (b) Number of recurrence feature types ($d=128$): its cumulative distribution from Proposition \ref{thm:recurrence-feature} (in blue) and the trade-off with parallelization efficiency from Theorem \ref{thm:complexity} (in orange). (c) Test MSE averaged over 25 simulation replicates (in blue), and average execution time over 100 forward and backward passes as well as their sum for a ParaRNN layer on a single V100 GPU.}
	\label{fig:two}
\end{figure}
We propose a new recurrent neural network that is composed of $K$ small RNNs with equal hidden size $d_s$ (i.e., $d_1 = d_2 = \cdots = d_K = d_s$):
\begin{linenomath}
	\begin{align}
		{\bm{h}}_{t}^{(k)} = \text{Recurrent-cell}(\bm{h}_{t-1}^{(k)}, \bm{x}_t) \in \mathbb{R}^{d_s}
		\quad \text{and} \quad
		\bm{h}_t =\bm{W}_f(\text{Concat}[{\bm{h}}_t^{(1)}, \ldots,{\bm{h}}_t^{(K)}]).
		\label{eq:pararnn}
	\end{align}
\end{linenomath}
Here $d_s$ and $K$ are hyperparameters, with $d_s$ restricted to be an even integer and the total hidden dimension given by $d=Kd_s$. $\text{Recurrent-cell}(\cdot)$ takes the same formulation as the first equality in \eqref{eq:nonlinear-rnn-diag}.
In addition, a parameter matrix, $\bm{W}_f \in \mathbb{R}^{d\times d}$, is utilized to aggregate all small hidden states, serving a similar purpose as the matrix $\bm{B}$ in \eqref{eq:rnn-approx-main}. 
In practice, the aggregation step can adopt alternative architectures, such as a fully-connected (FC) layer or a position-wise feedforward network \citep{vaswani2017attention}.
An illustration of the model architecture is presented in Figure \ref{fig:two}(a). Note that when $d_s=d$ and $K=1$, the proposed model is equivalent to the general RNN at \eqref{eq:nonlinear-rnn}.

Our design explicitly builds on the decoupling and approximation ideas developed in Section \ref{sec:motivation}. As in \eqref{eq:rnn-additive}, the hidden state of our network admits an additive form,
\begin{equation*}
	\vspace{-0.5em}
	\bm{h}_t = \sum_{k=1}^K f_{\text{RNN}}^{(k)} (\mathcal{F}_{t}; \bm{\Psi}^{(k)})
	\quad
	\text{with}
	\quad
	f_{\text{RNN}}^{(k)} (\mathcal{F}_{t}; \bm{\Psi}^{(k)}) =\bm{W}_f^{(k)} \bm{h}_t^{(k)},
	\vspace{0.1em}
\end{equation*}
where $\bm{W}_f^{(k)} = (\bm{W}_f)_{:, a_{k-1}+1: a_k}$ and $\bm{\Psi}^{(k)} = \{\bm{W}_h^{(k)}, \bm{W}_x^{(k)}, \bm{b}^{(k)}, \bm{W}_f^{(k)}\}$. As a result, the network is both more interpretable and computationally efficient: its recurrent dynamics are disentangled into components, and its forward and backward propagations can be parallelized across $K$ constituent RNNs. We refer to the network defined in \eqref{eq:pararnn} as \textbf{ParaRNN}.

\begin{theorem}
	Without the equal-sized constraint in \eqref{eq:pararnn}, the $T$-step time complexity of forward propagation for one ParaRNN layer is $\mathcal{O}(Td_\mathrm{max}^2+ T{d_\mathrm{max}}d_{\mathrm{in}} + d^2)$, and that of backward propagation is $\mathcal{O}(T^2 d_\mathrm{max}^3 + dd_\mathrm{max}^3)$, where $d_\mathrm{max} = \max_{1\leq k\leq K}{d_k}$.
	\label{thm:complexity}
\end{theorem}

Although ParaRNN can allow constituent RNNs to have heterogeneous hidden sizes $d_1, d_2, \ldots, d_K$, we recommend the equal-sized setting to maximize parallelization efficiency.
Theorem \ref{thm:complexity} shows that, in the absence of an equal-sized constraint, the time complexity of a ParaRNN layer is determined by the constituent RNN with the largest hidden size.
In comparison with a vanilla RNN of the same total hidden dimension $d$, which incurs forward and backward complexities of $\mathcal{O}(Td^2 + Tdd_{\mathrm{in}})$ and $\mathcal{O}(T^2 d^3)$, respectively, 
ParaRNN achieves substantial computational savings when all constituents have the same hidden size and $d_s$ is small. The aggregation step brings an additional cost of $\mathcal{O}(d^2)$ and $\mathcal{O}(dd_\mathrm{max}^3)$ to forward and backward propagation. However, these complexities do not scale with sequence length $T$ since the aggregation is parallel across time.

\begin{algorithm}[t]
	\caption{Forward propagation of $L$ ParaRNN layers} \label{alg:ParaRNN}
	{\setlength{\baselineskip}{0.73\baselineskip}
	\begin{algorithmic}[1]
		\STATE \textbf{Input:} $\{\bm{x}_t\}_{t=1}^{T}$, $\sigma_h(\cdot)$, $K$, $d_s$, and $L$ \\
		\STATE \textbf{Output}: the output hidden states of the $L$-th layer $\{\bm{h}_{t, L}\}_{t=1}^T$ \\
		\STATE \textbf{for} $l \leftarrow 1$ \textbf{to} $L$ \textbf{do} \\
		\STATE  \hspace{\algorithmicindent} Initialize ${\bm{h}}_{0,l}^{(k)}\leftarrow\bm{0}$ for $1 \leq k \leq K$ \\
		\STATE \hspace{\algorithmicindent} \textbf{for} $t \leftarrow 1$ \textbf{to} $T$ \textbf{do} \\
		\STATE \hspace{\algorithmicindent}\hspace{\algorithmicindent} \textbf{do in parallel}
		\STATE \hspace{\algorithmicindent}\hspace{\algorithmicindent}\hspace{\algorithmicindent} \textbf{if} $l=1$ \textbf{then}
		\STATE \hspace{\algorithmicindent}\hspace{\algorithmicindent}\hspace{\algorithmicindent}\hspace{\algorithmicindent} compute ${\bm{h}}_{t,l}^{(k)} \leftarrow \sigma_h({\bm{W}}_{h,l}^{(k)} {\bm{h}}_{t-1,l}^{(k)} + {\bm{W}}_{x,l}^{(k)} \bm{x}_t)$ for $1\leq k \leq K$
		\STATE \hspace{\algorithmicindent}\hspace{\algorithmicindent}\hspace{\algorithmicindent}
		\textbf{else}
		\STATE 
		\hspace{\algorithmicindent}\hspace{\algorithmicindent}\hspace{\algorithmicindent}\hspace{\algorithmicindent}
		compute ${\bm{h}}_{t,l}^{(k)} \leftarrow \sigma_h({\bm{W}}_{h,l}^{(k)} {\bm{h}}_{t-1,l}^{(k)} + {\bm{W}}_{x,l}^{(k)} \text{Concat}[{\bm{h}}_{t,l-1}^{(1)}, ..., {\bm{h}}_{t,l-1}^{(K)}])$ for $1\leq k \leq K$
		\STATE \hspace{\algorithmicindent}\hspace{\algorithmicindent}\hspace{\algorithmicindent} \textbf{end if}
		\STATE \hspace{\algorithmicindent} \textbf{end for}
		\STATE \textbf{end for}
		\STATE \textbf{do in parallel}
		\STATE \hspace{\algorithmicindent} compute $\bm{h}_{t,L} \leftarrow \text{Aggregate}(\text{Concat}[{\bm{h}}_{t, L}^{(1)},...,{\bm{h}}_{t, L}^{(K)}])$ for $1 \leq t \leq T$
	\end{algorithmic}
}
\end{algorithm}

Finally, same as standard recurrent architectures, ParaRNN can be extended to a deep network by stacking multiple layers together, which facilitates the modeling of more intricate temporal dynamics. The forward propagation for $L$ ParaRNN layers is described in Algorithm \ref{alg:ParaRNN}. 
To mitigate the computational overhead, the aggregation step is applied only once after the final ParaRNN layer; in the intermediate layers, it is implicitly incorporated into the linear projection of layer inputs that the recurrent computation already requires.

\subsection{Model Interpretation}\label{subsec:recurrence-features}
As established by the additive structure in Section \ref{subsec:pararnn-architecture}, the overall recurrent dynamics of ParaRNN already decompose into a sum of contributions from individual RNNs. Nevertheless, the interpretation of the temporal behavior within each component remains elusive. This subsection attempts to characterize the recurrent dynamics induced by each constituent recurrent matrix $\bm{W}_h^{(k)} \in \mathbb{R}^{d_s \times d_s}$. To this end, we employ a classical result from matrix analysis, namely the real Jordan decomposition \citep[Theorem 3.1.11]{horn2012matrix}.

\begin{proposition}[Real Jordan Decomposition]
	Suppose a matrix $\bm{W} \in \mathbb{R}^{d_s \times d_s}$ has $r$ distinct nonzero real eigenvalues $\{\lambda_j\}_{j=1}^r$ and $s$ distinct conjugate pairs of nonzero complex eigenvalues $\{(\lambda_{r+2k-1}, \lambda_{r+2k}) = (\gamma_k e^{i\theta_k}, \gamma_k e^{-i\theta_k}) \}_{k=1}^s$ with $\lambda_j \in \mathbb{R}$, $\gamma_k > 0$ and $\theta_k \in (-\pi/2, \pi/2)$. Assume that each nonzero eigenvalue has geometric multiplicity one. Then the matrix $\bm{W}$ is similar over $\mathbb{R}$ to a real block diagonal matrix 
	\begin{linenomath}
		\begin{align}
			\bm{J} = \bm{J}_{n_1}(\lambda_1) \oplus \cdot \cdot \cdot \oplus\bm{J}_{n_r}(\lambda_r) \oplus  \bm{C}_{n_{r+1}}(\gamma_1, \theta_1) \oplus \cdot \cdot \cdot \oplus \bm{C}_{n_{r+s}}(\gamma_s, \theta_s)
			\oplus \bm{0}_{\mathrm{nullity}(\bm{W} )},
			\label{eq:J-form-main}
		\end{align} 
	\end{linenomath}
	where the real Jordan blocks are defined by 
	\begin{align*} 
		\bm{J}_n(\lambda) = 
		\begin{pmatrix}
			\lambda & 1 &   &  \\
			& \ddots & \ddots & \\
			& & \lambda & 1 \\
			& & & \lambda \\
		\end{pmatrix} \in\mathbb{R}^{n\times n}
		\hspace{1.5mm} \text{and} \hspace{1.5mm}
		\bm{C}_{n}(\gamma, \theta) = 
		\begin{pmatrix}
			\bm{C}(\gamma, \theta) & \bm{I}_2&   &  \\
			& \ddots & \ddots & \\
			& &\ddots & \bm{I}_2\\
			& & &\bm{C}(\gamma, \theta)  \\
		\end{pmatrix}\in\mathbb{R}^{2n\times 2n}
	\end{align*}
	with 
	$\bm{C}(\gamma, \theta) =  \left( \begin{smallmatrix} \cos{\theta} & \sin{\theta} \\ -\sin{\theta} & \cos{\theta}  \end{smallmatrix} \right)\in \mathbb{R}^{2\times2}$
	and $\bm{I}_2 \in  \mathbb{R}^{2\times2}$ being the identity matrix. 
	The subscript $n_k \geq 1$ denotes the corresponding algebraic multiplicity of real eigenvalue $\lambda_k$ or complex conjugate pair $(\gamma_{k-r} e^{i\theta_{k-r}}, \gamma_{k-r} e^{-i\theta_{k-r}})$, and $\bm{0}_{\mathrm{nullity}(\bm{W} )} \in \mathbb{R}^{\mathrm{nullity}(\bm{W} )\times \mathrm{nullity}(\bm{W} )}$ is a zero matrix with $\mathrm{nullity}(\bm{W} )$ representing the dimension of $\bm{W} $'s null space.
	\label{prop:RJD}
\end{proposition}

Without loss of generality, suppose that $\bm{W}_h^{(k)}$ satisfies those eigenvalue conditions in Proposition \ref{prop:RJD}. Then it has the decomposition $\bm{W}_h^{(k)} = \bm{S}\bm{J}\bm{S}^{-1}$ with $\bm{S} \in \mathbb{R}^{d_s \times d_s}$ invertible and $\bm{J}$ of the form \eqref{eq:J-form-main}. Notably, each block of $\bm{J}$ is determined by the eigenvalues of $\bm{W}_h^{(k)}$ and their respective algebraic multiplicities. Putting this decomposition back into Recurrent-cell($\cdot$) and applying the same approximation argument in \eqref{eq:approx}, we obtain the surrogate dynamics
\begin{linenomath}
	\begin{align*}
		\tilde{\bm{h}}_t^{(k)}=\sigma_h\big(\bm{J}\tilde{\bm{h}}_{t-1}^{(k)}+\tilde{\bm{W}}_x^{(k)}\bm{x}_t + \tilde{\bm{b}}^{(k)}\big) \quad \text{and} \quad \bm{h}_t^{(k)} = \bm{S}\tilde{\bm{h}}_t^{(k)},
	\end{align*}
\end{linenomath}
where $\tilde{\bm{W}}_x^{(k)} = \bm{S}^{-1}\bm{W}_x^{(k)}$ and $\tilde{\bm{b}}^{(k)}= \bm{S}^{-1}\bm{b}^{(k)}$. 
This representation further isolates the recurrent behavior of each constituent RNN through the canonical blocks of $\bm{J}$.

As shown in \eqref{eq:J-form-main}, matrix $\bm{J}$ has three types of blocks: $\bm{J}_{n}(\lambda)$, $\bm{C}_{n}(\gamma, \theta)$ and $\bm{0}$. The term $\bm{0}$ produces no recurring pattern, while the others induce two types of recurrent dynamics: 
\begin{linenomath}
	\begin{equation}\label{eq:rnn-basic}
		{\bm{h}}_t=\sigma_h \big(\bm{J}_{n}(\lambda) {\bm{h}}_{t-1} + \bm{W}_x \bm{x}_t + {\bm{b}} \big) \in\mathbb{R}^n \hspace{3mm}\text{and}\hspace{3mm}
		{\bm{h}}_t=\sigma_h \big(\bm{C}_{n}(\gamma,\theta) {\bm{h}}_{t-1} + \bm{W}_x \bm{x}_t + {\bm{b}} \big) \in\mathbb{R}^{2n}.
	\end{equation}
\end{linenomath}
These dynamics cannot be further decoupled, since the recurrent matrices $\bm{J}_n(\lambda)$ and $\bm{C}_{n}(\gamma, \theta)$ are irreducible under the real Jordan decomposition. 
Recognizing that the two RNNs in \eqref{eq:rnn-basic} are the fundamental units to constitute the recurrent dynamics of a ParaRNN, we formally define their dynamics as \textit{recurrence features}.
Furthermore, for an algebraic multiplicity $n$, $\bm{J}_{n}(\lambda)$ and $\bm{C}_{n}(\gamma, \theta)$ characterize distinct temporal dependencies between hidden states and inputs through whether the eigenvalue is real or complex. We hence refer to the corresponding dynamics as Types R-$n$ and C-$n$ recurrence features, respectively. 

When $\sigma_h(\bm{x}) = \bm{x}$, as in ARMA-type models, and $|\lambda|, |\gamma| < 1$,
Type R recurrence features provide a regular exponential decay pattern of temporal dependence, whereas Type C features exhibit a damped sinusoidal wave-like behavior and can be used to capture periodic patterns \citep{fuller1996intro, cryer2008time}. 
For instance, we have 
\begin{align*}
	\bm{h}_t =  \sum_{j=0}^{t-1}\lambda^j \bm{W}_x  \bm{x}_{t-j}
	\hspace{3mm} \text{or} \hspace{3mm} 
	\bm{h}_t =  \sum_{j=0}^{t-1}
	\gamma^j 
	\begin{pmatrix}
		\cos{j\theta} & \sin{j\theta} \\ -\sin{j\theta} & \cos{j\theta}  
	\end{pmatrix}\bm{W}_x \bm{x}_{t-j}
\end{align*}
for a R-1 feature and a C-1 feature, respectively.

When $\sigma_h(\cdot)$ is nonlinear, interpreting recurrence features becomes less direct, as the nonlinearity confounds the dynamics encoded by $\bm{J}_{n}(\lambda)$ and $\bm{C}_{n}(\gamma, \theta)$. Taking a R-1 feature as an example, the local dependence of the current hidden state on the previous one can be quantified as
$\partial \bm{h}_t/\partial \bm{h}_{t-1} = \sigma_h'(\bar{\bm{h}}_{t}) \cdot \lambda$,
where $\sigma_h'(\cdot)$ is the first derivative of $\sigma_h(\cdot)$ and $\bar{\bm{h}}_t=\lambda \bm{h}_{t-1}+\bm{W}_x\bm{x}_t + \bm{b} \in \mathbb{R}$ is the pre-activation state. 
The impact of nonlinearity is data-dependent. 
For commonly used activation functions such as Tanh, Sigmoid, and ReLU, the nonlinearity may preserve exponential decay, accelerate forgetting, or entirely reset memory, thereby altering the linear recurrence behavior. 

\subsection{Hyperparameter Selection}\label{subsec:dominant-feature}
Compared to a general RNN with hidden size $d$ in \eqref{eq:nonlinear-rnn}, a ParaRNN with the same total hidden dimension can be regarded as imposing an equal-sized block diagonal constraint on its recurrent matrix, i.e., $\bm{W}_h = \oplus_{k=1}^{K} \bm{W}_h^{(k)}$ with $\bm{W}_h^{(k)} \in \mathbb{R}^{d_s \times d_s}$. When $d_s < d$, this leads to differences in the recurrence features that the two models are able to provide. Specifically, a general RNN can have at most Types R-$\{n\}_{n=1}^d$ and C-$\{n\}_{n=1}^{d/2}$ recurrence features. In contrast, a ParaRNN can provide at most Types R-$\{n\}_{n=1}^{d_s}$ and C-$\{n\}_{n=1}^{d_s/2}$ features, which form a subset of those attainable by a general RNN and are controlled by the size $d_s$. Increasing $d_s$ expands the range of recurrence feature types that ParaRNN can present, but this comes at the expense of higher time complexity, as indicated by Theorem \ref{thm:complexity} and the orange curve in Figure \ref{fig:two}(b). Consequently, under a limited computational budget, the choice of $d_s$ involves a trade-off between recurrence feature richness and parallelization efficiency.

\begin{figure}[t]
	\centering
	\includegraphics[width=\linewidth]{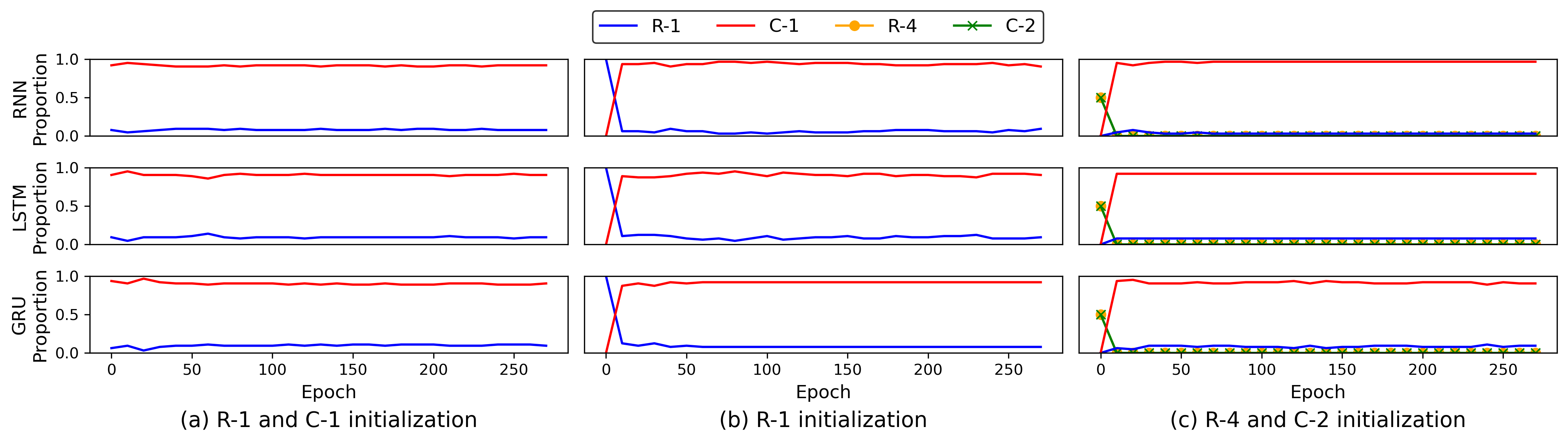}
	\caption{Evolution of recurrence feature types across training for vanilla RNN, LSTM, and GRU models from the permuted sequential MNIST task in Section \ref{sec:result-mnist}. Types other than R-1, C-1, R-4, or C-2 have zero occurrence.}
	\label{fig:recurrence-feature-type}
\end{figure}

To guide the selection of $d_s$, we first examine the empirical prevalence of different recurrence feature types in the real-world examples of Section \ref{sec:result}, where the recurrent models are overparameterized. 
Figure \ref{fig:recurrence-feature-type} tracks the shift of recurrence feature types during training for RNN, LSTM, and GRU models on the permuted sequential MNIST classification task in Section \ref{sec:result-mnist}. 
Despite different initializations, the trained models rapidly concentrate on low-order (R-1 and C-1) features.
Besides, C-1 features occur more frequently than R-1 features, a trend consistent in the time series forecasting example in Section \ref{sec:result-ts} as depicted by Figure \ref{fig:derivative}(b). 
Note that these features and their corresponding eigenvalues undergo active changes at the initial stages of training and stabilize as optimization converges. 
A simulation setting under the asymptotic framework with the data generating process (DGP) known to originate from an RNN model has also been explored in Section \ref{appendix:sec-feature-simulation} of Supplementary Material. While the trained recurrent matrix may converge to true values from the DGP that comprises R-4 and C-2 features only, the recurrence features are still predominantly low-order. 

\begin{proposition}
	Let $\bm{W}_h = (\xi_{ij})_{1 \leq i,j \leq d}$ in \eqref{eq:nonlinear-rnn}, and assume that $\{\xi_{ij}\}$ have a continuous joint distribution. Then, with probability one, the RNN contains only recurrence features of Types R-1, R-2 and C-1. Moreover, if $\{\xi_{ij}\}$ are independent standard normal random variables, the probability that the RNN contains only Type R-1 features is at most $1 / 2^{d(d-1)/4}$.
	\label{thm:recurrence-feature}
\end{proposition}

Under an idealized random matrix setting, Proposition \ref{thm:recurrence-feature} attempts to theoretically explain these observations. It implies that high-order recurrence features are almost surely absent, and purely R-1 dynamics occur only with exponentially small probability; see the blue curve in Figure \ref{fig:two}(b) for an illustration. 
On the other hand, Theorem \ref{thm:dense} provides a complementary perspective.
The denseness of $\mathbb{M}_d^2$ in $\mathbb{M}_d$ implies that setting $d_s$ to 2 is sufficient to approximate the recurrent dynamics of a general RNN. 
In contrast, diagonal recurrent matrices belonging to $\mathbb{M}_d^1$, as considered in \cite{li2018indrnn, martin2018plrn, rusch2021unicornn}, are strictly more limited. 
They fail to engender Type C recurrence features, consequently missing the periodic dynamics prevalent in prior empirical examples.

Taken together, these empirical and theoretical observations motivate setting $d_s = 2$ as a default configuration for ParaRNN. It covers the most common recurrence feature types and does not compromise recurrent dynamics in an approximation sense while maintaining high computational efficiency.
The numerical results in Section \ref{sec:result} also support its sufficiency in most cases.
When richer recurrent dynamics are demanded by specific datasets, larger values of $d_s$ and $K$ can be adopted to allow high-order and more diverse recurrence features.

\section{Theoretical Justifications}\label{subsec:theory}
\subsection{Approximation Theory}
This subsection investigates the approximation capacity of ParaRNNs.
We start by mathematically defining a function class based on the multi-layer ParaRNNs proposed in Algorithm \ref{alg:ParaRNN}, where the rectified linear unit (ReLU) activation function is applied and additional transformation layers are incorporated to accommodate different input and output dimensions. 
Specifically, for an input sequence of length {$T$, $\bm{X}_{T} = (\bm{x}_1, \dots, \bm{x}_{T}) \in \mathbb{R}^{d_\mathrm{in} \times T}$}, we consider a deep ParaRNN $\mathcal{N}$ consisting of: 
(i) one linear input transformation layer $\mathcal{P}: \mathbb{R}^{d_\mathrm{in} \times { T}} \mapsto \mathbb{R}^{d \times { T}}$, where $\mathcal{P}(\bm{X}_{{T}}) = (\bm{P}\bm{x}_1, \dots, \bm{P}\bm{x}_{T})$ for a given matrix $\bm{P} \in \mathbb{R}^{d \times d_\mathrm{in}}$; 
(ii) $L$ recurrent layers $\mathcal{R}_1, \dots, \mathcal{R}_L: \mathbb{R}^{d \times {T}} \mapsto \mathbb{R}^{d \times {T}}$, each following the form of \eqref{eq:pararnn} but without $\bm{W}_f$, using a block size of $d_s$, and setting $\sigma_h$ as ReLU; (iii) one position-wise FC layer $\mathcal{F}: \mathbb{R}^{d \times {T}} \mapsto \mathbb{R}^{d \times {T}}$ with ReLU activation for aggregation; and (iv) one linear output transformation layer $\mathcal{Q}: \mathbb{R}^{d \times {T}}  \mapsto \mathbb{R}^{d_\mathrm{out} \times {T}}$, defined analogously to $\mathcal{P}$. 
A function class comprising the outputs at the last time step of these ParaRNNs is defined as:
\begin{align*}
	\mathcal{F}^{({T})}_{d_\mathrm{in}, d_\mathrm{out}, d, d_s, L, U} =  \{\mathcal{N}(\bm{X}_{{T}})[{T}]: 
	\mathbb{R}^{d_\mathrm{in} \times {T}}  \mapsto \mathbb{R}^{d_\mathrm{out}}, \hspace{1mm}
	&\mathcal{N}(\bm{X}_{{T}}) =  \mathcal{Q} \circ \mathcal{F} \circ \mathcal{R}_L \circ \cdots \circ \mathcal{R}_1 \circ \mathcal{P} (\bm{X}_{{T}}) \\ \in\mathbb{R}^{d_\mathrm{out} \times {T}},
	\text{ with } & 
	\sup_{\bm{X}_{{T}} \in \mathbb{R}^{d_\mathrm{in} \times {T}}, t\in \{1, \dots, {T}\}} \norm{\mathcal{N}(\bm{X}_{{T}})[t]}_\infty \leq U
	\},
\end{align*}
where $\mathcal{N}(\bm{X}_{{T}})[{t}]\in\mathbb{R}^{d_\mathrm{out}}$ denotes the output at time ${t}$, with all entries assumed to be bounded. 

An important question that is worth investigating is the capacity of this ParaRNN function class to approximate an unknown function $f^{({T})}:\mathbb{R}^{d_\mathrm{in} \times {T}} \mapsto \mathbb{R}^{d_\mathrm{out}}$. The approximation error depends on the smoothness of $f^{({T})}$, and various assumptions have been placed in approximation theory literature, such as requiring $f^{({T})}$ to reside in a Sobolev ball \citep{farell2021deep}, Korobov spaces \citep{mohri2018foundations}, or Besov spaces \citep{suzuki2018adaptivity}. In line with \cite{shen2020deep, lu2021deep, jiao2023deep, jiao2024approximation}, we assume that $f^{({T})}$ belongs to a H\"{o}lder class with a smoothness index $\beta > 0$, the definition of which is provided below.
\begin{definition}[H\"{o}lder class]
	Let $\Omega \subset \mathbb{R}^{d_\mathrm{in}}$ and $\beta > 0$ with $\beta = k + \omega$, where $k \in \mathbb{N}_0$, $\mathbb{N}_0 = \mathbb{N} \cup \{0\}$ and $\omega \in (0, 1]$. A function is said to be $\beta$-smooth if all its partial derivatives up to order $k$ exist and are bounded, and the partial derivatives of order $k$ are $\omega$-H\"{o}lder continuous. For $d_\mathrm{in}, d_\mathrm{out} \in \mathbb{N}$, the H\"{o}lder class with smoothness index $\beta$ is defined as
	\begin{equation*}
		\begin{split}
			\mathcal{H}_{d_\mathrm{in}, d_\mathrm{out}}^\beta (\Omega, M) &= 
			\Big\{
			f = (f_1, \dots, f_{d_\mathrm{out}})^\top: \Omega \mapsto \mathbb{R}^{d_\mathrm{out}}, \\
			\sum_{\bm{n}: \norm{\bm{n}}_1 < \beta} &\norm{\partial^{\bm{n}} f_i}_{L^\infty(\Omega)} + 
			\sum_{\bm{n}: \norm{\bm{n}}_1 = k} \sup_{\bm{x}, \bm{y} \in \Omega, \bm{x} \neq \bm{y}}
			\frac{\lvert \partial^{\bm{n}} f_i(\bm{x} ) - \partial^{\bm{n}} f_i(\bm{y}) \rvert}{\norm{\bm{x} - \bm{y}}^\omega} \leq M,
			\hspace{3mm} i = 1, \dots, d_\mathrm{out}
			\Big\},
		\end{split}
	\end{equation*}
	where $\partial^{\bm{n}}= \partial^{n_1} \dots \partial^{n_{d_\mathrm{in}}}$ with $\bm{n} = (n_1, \dots, n_{d_\mathrm{in}}) \in \mathbb{N}_0^{d_\mathrm{in}}$ and $\norm{\bm{n}}_1 = \sum_{i=1}^{d_\mathrm{in}} n_i$. 
\end{definition}

\begin{theorem}[Approximation error] \label{proposition:approx-rate-single-step}
	Assume that  $f^{({T})} \in \mathcal{H}^{\beta}_{d_\mathrm{in} \times {T}, d_\mathrm{out}}$ $([0, 1]^{d_\mathrm{in} \times {T}}, U)$. Then for any $I, J \in \mathbb{N}^+$, there exists a ParaRNN-based function $\phi \in \mathcal{F}^{({T})}_{d_\mathrm{in}, d_\mathrm{out}, d, d_s, L, U}$ 	such that 
	\begin{align*}
		\sup_{\bm{X} \in [0, 1]^{d_\mathrm{in} \times {T}}}
		\norm{\phi(\bm{X})  - f^{({T})}(\bm{X})}_\infty \leq 19 U (\lfloor \beta \rfloor + 1)^2 (d_\mathrm{in}{T})^{\lfloor \beta \rfloor + (\beta  \vee 1) / 2} (JI)^{-2\beta / (d_\mathrm{in}{T})},
	\end{align*}
	where the depth $L= 42( \lfloor \beta \rfloor + 1)^2 I \lceil \log_2(8I) \rceil + 6d_\mathrm{in} T$, width $d = 
	76 (\lfloor \beta \rfloor + 1)^2 3^{d_\mathrm{in}T} d_\mathrm{in}^{\lfloor \beta \rfloor + 2} d_\mathrm{out}$ $T^{\lfloor \beta \rfloor + 1} J \lceil \log_2(8J)\rceil + d_s$, and $d$ is divisible by the block size $d_s$ without loss of generality.
\end{theorem}

Theorem \ref{proposition:approx-rate-single-step} gives the capacity of the function class $\mathcal{F}^{({T})}_{d_\mathrm{in}, d_\mathrm{out}, d, d_s, L, U}$ to approximate H\"{o}lder smooth functions.
The error bound has the approximation rate $(JI)^{-2\beta / (d_\mathrm{in}{T})}$, and it holds for arbitrary network width and depth specified by $I$ and $J$. 
Compared to the approximation results of RNNs in Lemma 10 of \cite{jiao2024approximation}, our ParaRNNs achieve the same rate with a slightly larger width.
The proof draws inspiration from \cite{song2023minimal,jiao2024approximation}, yet establishing the equivalence between deep ParaRNNs and deep feedforward neural networks (FNNs) is a nontrivial task; see Section \ref{subsec:approx-proof-general} of the Supplementary Material. 

\subsection{Non-Asymptotic Prediction Error Bounds}
This subsection analyzes ParaRNNs in the context of nonparametric least squares regression. We first establish an estimation error bound for ParaRNN-based estimators obtained through empirical risk minimization. Combining this result with their approximation capacity, we further derive a non‑asymptotic upper bound on the prediction error for ParaRNNs.

Consider random variables $(\bm{x}_1,\ldots,\bm{x}_T, z_T)$ for a nonparametric regression model, where the predictors satisfy $\bm{x}_t \in [0, 1]^{d_\mathrm{in}}$ and the response $z_{T} \in \mathbb{R}$, i.e., the output dimension is set to $d_\mathrm{out}=1$.  Note that this formulation includes time series forecasting as a special case when $z_T=\bm{x}_{T+h}$ with $h>0$, as well as classification problems when $z_T$ takes discrete values. The function $f^{({T})}_{0} (\bm{X}_{T}) = \mathbb{E}[z_{T} | \bm{x}_1, \dots, \bm{x}_{T}]$ is unknown, and our objective is to construct an estimator of $f^{({T})}_{0}$ based on the ParaRNN function class $\mathcal{F}^{({T})}_{d_\mathrm{in}, 1, d, d_s, L, U}$.

For any function $\phi: \mathbb{R}^{d_\mathrm{in} \times {T}} \mapsto \mathbb{R}$, define its $L_2$ risk as
\begin{align*}
	\mathcal{R}(\phi) = \mathbb{E}_{(\bm{X}_{T}, z_{{T}})} [\phi(\bm{X}_{{T}}) - z_{{T}}]^2,
\end{align*}
where the expectation is taken with respect to $(\bm{X}_{T}, z_{{T}})$. 
Consider a random sample $\mathcal{S} = \{ ({ \bm{X}_{i, T}, z_{i, T}}),$ $ i=1, \dots, N\}$, which are independently and identically distributed (i.i.d.) as $({ \bm{X}_T, z_T})$. Accordingly, we then can define the empirical risk with respect to a function $\phi$:
\begin{align*}
	\mathcal{R}_N(\phi) = \frac{1}{N} \sum_{i=1}^{N} \left[ \phi(\bm{X}_{i, {T}}) - z_{i, {T}}\right]^2,
\end{align*}
where $\bm{X}_{i, {T}} = (\bm{x}_{i, 1}, \dots, \bm{x}_{i, {T}})$.
As a result, the estimator of $f^{({T})}_{0}$ can be constructed by minimizing the empirical risk within the ParaRNN function class $\mathcal{F}^{({T})}_{d_\mathrm{in}, 1, d, d_s, L, U}$,
\begin{align*}
	\widehat{f}^{({T})} = \argmin_{\phi \in\mathcal{F}^{({T})}_{d_\mathrm{in}, 1, d, d_s, L, U}} 
	\mathcal{R}_N(\phi),
\end{align*}
which is referred to as the empirical risk minimizer (ERM).

One can show that $f^{({T})}_{0} = \argmin_{\phi} \mathcal{R}(\phi)$. 
Denote by $\bar{f}^{({T})} = \argmin_{\phi \in\mathcal{F}^{({T})}_{d_\mathrm{in}, 1, d, d_s, L, U}} \mathcal{R}(\phi)$ the minimizer of population risk over $\mathcal{F}^{({T})}_{d_\mathrm{in}, 1, d, d_s, L, U}$, and then the approximation error of the ParaRNN function class can be quantified by $\mathcal{R}(\bar{f}^{({T})})- \mathcal{R}(f^{({T})}_{0}) = \mathbb{E}_{\bm{X}_{T}}\{[\bar{f}^{({T})}(\bm{X}_{T}) - f_0^{({T})}(\bm{X}_{T})]^2\}$. 
Note that the prediction error is $\mathcal{R}(\widehat{f}^{({T})})$, and we next consider the excess risk of the ParaRNN-based ERM below,
\begin{align*}
	\mathcal{R}(\widehat{f}^{({T})}) - \mathcal{R}(f^{({T})}_{0}) =  \left[\mathcal{R}(\widehat{f}^{({T})}) - \mathcal{R}(\bar{f}^{({T})})\right] + 	\left[\mathcal{R}(\bar{f}^{({T})})- \mathcal{R}(f^{({T})}_{0}) \right],
\end{align*}
where the first term represents the estimation error, and the second term, i.e., the approximation error, can be readily bounded using Theorem \ref{proposition:approx-rate-single-step}.
We next proceed to derive the estimation error bound, preceded by the introduction of some notations.
For any two sequences $a_n$ and $b_n$, denote $a_n \lesssim b_n$ (or $a_n \gtrsim b_n$) if there exists an absolute constant $C > 0$ such that $a_n \leq Cb_n$ (or $a_n \geq Cb_n$). Write $a_n \asymp b_n$ if $a_n \lesssim b_n$ and $a_n \gtrsim b_n$.
We first state a condition on the distribution of the response $z_T$, and then the estimation error bound.
\begin{assumption}\label{assump:Y}
	The response variable $z_{{T}}$ is sub-exponentially distributed, i.e., there exists a constant $\sigma_z > 0$ such that $\mathbb{E} \exp(\sigma_z |z_{{T}}|) < \infty$.
\end{assumption}
\begin{theorem}[Estimation error]\label{proposition:estimation-error}
	Suppose that Assumption \ref{assump:Y} holds and $\lVert{f^{({T})}_0}\rVert_\infty \leq U$ for some $U \geq 1$. Then, for $N \gtrsim d^2 L^2 \log \max\{d, L\}$,
	\begin{align*}
		\mathbb{E}_\mathcal{S}[\mathcal{R}(\widehat{f}^{({T})}) - \mathcal{R}(\bar{f}^{({T})})]
		\leq 
		C U^5  (\log N)^5  \frac{1}{N} d^2 L^2 \log \max\{d, L\} + 0.5 \left[\mathcal{R}(\bar{f}^{({T})})- \mathcal{R}(f^{({T})}_{0}) \right],
	\end{align*}
	where the expectation is taken with respect to the sample $\mathcal{S}$, and $C>0$ is a constant independent of $d_\mathrm{in}, d_s, d, L, U, N$.
\end{theorem}
The estimation error in Theorem \ref{proposition:estimation-error} is mainly bounded by the metric entropy of $\mathcal{F}^{({T})}_{d_\mathrm{in}, 1, d, d_s, L, U}$, which can be quantified by its covering number. Note that deriving the covering number for ParaRNNs is also nontrivial as it is related to the equivalence between deep ParaRNNs and FNNs; see the detailed proof in Section \ref{subsec:estimation-proof} of the Supplementary Material. Moreover, the bound also relies on the approximation error slightly. Actually, the factor of 0.5 at the second term can be replaced by an arbitrarily small but fixed number.

Lastly, combining Theorems \ref{proposition:approx-rate-single-step} and \ref{proposition:estimation-error}, we are now ready to present the non-asymptotic prediction error bound for a ParaRNN-based ERM.

\begin{theorem}\label{thm:prediction-error}
	Suppose that Assumption \ref{assump:Y} holds, $f^{({T})}_0 \in \mathcal{H}_{d_\mathrm{in} \times {T}, 1}^{\beta} $ $([0, 1]^{d_\mathrm{in} \times {T}}, U)$ for some $U \geq 1$, and the function class $\mathcal{F}^{({T})}_{d_\mathrm{in}, 1, d, d_s, L, U}$ has width $d = 
	76 (\lfloor \beta \rfloor + 1)^2 3^{d_\mathrm{in}T} d_\mathrm{in}^{\lfloor \beta \rfloor + 2} T^{\lfloor \beta \rfloor + 1} J \lceil \log_2(8J)\rceil + d_s$ and depth 
	$L = {42( \lfloor \beta \rfloor + 1)^2 I \lceil \log_2(8I) \rceil}$ $ + 6d_\mathrm{in} T$ for any $I, J \in \mathbb{N}^{+}$.
	If $N \gtrsim d^2 L^2 \log \max\{d, L\}$, then the prediction error of the ERM $\widehat{f}^{({T})}$ satisfies 
	\begin{align}
		& \mathbb{E}_\mathcal{S} \left[ \mathcal{R}(\widehat{f}^{({T})}) - \mathcal{R}(f^{({T})}_{0})\right] \notag \\
		\leq  & C U^5  (\log N)^5  \frac{1}{N} d^2 L^2 \log \max\{d, L\} + 
		542 U^2 (\lfloor \beta \rfloor + 1)^4 (d_\mathrm{in}{T})^{2 \lfloor \beta \rfloor + (\beta  \vee 1)} (JI)^{-4\beta / (d_\mathrm{in}{T})},
		\label{eq:expected-excess-risk}
	\end{align} 
	where $C > 0$ is a constant not depending on $d_\mathrm{in}, d, L, U, \beta, N, I$ or $J$.
\end{theorem}

As in \cite{jiao2023deep, jiao2024approximation}, the prefactor $542 U^2 (\lfloor \beta \rfloor + 1)^4 (d_\mathrm{in}{T})^{2 \lfloor \beta \rfloor + (\beta  \vee 1)}$ in the approximation term has a polynomial dependency on the input dimension $d_\mathrm{in}$ rather than an exponential one in \cite{schmidt2020non} and \cite{lu2021deep}.
This makes it possible to consider higher-dimensional inputs.
Achieving the best error rate requires balancing the two terms in \eqref{eq:expected-excess-risk}: as the complexity of $\mathcal{F}_{d_\mathrm{in}, 1, d, d_s, L, U}^{({T})}$ grows with larger $d$ and $L$, the first term depending on its metric entropy will enlarge, while the second term for its approximation error will diminish.
The following corollary gives the upper bound after balancing.
\begin{corollary}\label{corollary:prediction-error}
	Suppose that Assumption \ref{assump:Y} holds, $f^{({T})}_0 \in \mathcal{H}_{d_\mathrm{in} \times {T}, 1}^{\beta} $ $([0, 1]^{d_\mathrm{in} \times {T}}, U)$ for some $U \geq 1$,
	and the function class $\mathcal{F}^{({T})}_{d_\mathrm{in}, 1, d, d_s, L, U}$ has width and depth, 
	\begin{align*}
		d \asymp N^{\eta} \log N  \hspace{3mm} \text{and} \hspace{3mm} L  \asymp N^{\frac{d_\mathrm{in} {T}}{2d_\mathrm{in} {T} + 4\beta} - \eta} \log N,
	\end{align*}
	for fixed $\eta \in [0, {d_\mathrm{in} {T}}/ ({2 d_\mathrm{in} {T} + 4\beta})]$. 
	If $N \gtrsim  d^2 L^2 \log \max\{d, L\}$, then the ERM $\widehat{f}^{({T})}$ satisfies
	\begin{align*}
		\mathbb{E}_\mathcal{S}  \left[ \mathcal{R}(\widehat{f}^{({T})}) - \mathcal{R}(f^{({T})}_{0})\right] \lesssim N^{-\frac{2\beta}{d_\mathrm{in} {T} + 2\beta}} (\log N)^{10}.
	\end{align*}
\end{corollary}

From the above corollary,  the ERM $\widehat{f}^{({T})}$ based on deep and wide ParaRNNs can attain the optimal minimax rate $N^{-{2\beta}/({d_\mathrm{in} {T} + 2\beta})}$ established by \cite{stone1982optimal} for nonparametric regression, up to a logarithmic factor; see also the RNN-based ERM in \cite{jiao2024approximation}. 
The results can be straightforwardly extended to scenarios where ParaRNNs are deep with fixed width or wide with fixed depth.
Note that the theoretical properties in this paper rely on settings that the number of sequences $N$ may diverge, while that of time points $T$ is fixed, i.e. we focus on typical machine learning tasks.
It is challenging to derive theoretical properties for classical time series problems with diverging $T$, and we leave it for future research.

\section{Extension to Other Recurrent Networks}\label{sec:extension}
Our framework can be easily extended to the latest recurrent networks, including gated and attention-based models, since they share similar representations of recurrent matrices and we can identify them through locating the weights on previous hidden states. 

To illustrate this, we take the LSTM, a widely used gated variant, as our first example. 
Recall there are forget, input, output gates and candidate cell state in LSTM, given by 
\begin{align*}
	\begin{bmatrix}
		\bm{f}_t \\  \bm{i}_t \\ \bm{o}_t
	\end{bmatrix}
	= \text{Sigmoid}\Bigl(
	\begin{bmatrix}
		\bm{W}_f \\ \bm{W}_i \\ \bm{W}_o
	\end{bmatrix}
	\bm{h}_{t-1} + 
	\begin{bmatrix}
		\bm{U}_f \\ \bm{U}_i \\ \bm{U}_o
	\end{bmatrix}
	\bm{x}_t + 
	\begin{bmatrix}
		\bm{b}_f \\  \bm{b}_i \\ \bm{b}_o
	\end{bmatrix}
	\Bigr) 
	\hspace{2mm}
	\text{and}
	\hspace{2mm}
	\tilde{\bm{c}}_t =\text{Tanh}(\bm{W}_c\bm{h}_{t-1} +\bm{U}_c\bm{x}_t + \bm{b}_c).
\end{align*}
Here all $\bm{W}$'s, $\bm{U}$'s and $\bm{b}$'s are weight matrices and bias.
Importantly, the updates of $\bm{f}_t$, $\bm{i}_t$, $\bm{o}_t$ and $\tilde{\bm{c}}_t$ have a similar form as the vanilla RNN, which highlights the relevance of analyzing \eqref{eq:nonlinear-rnn}. The recurrent dynamics are jointly determined by four weight matrices, $\bm{W}_f$, $\bm{W}_i$, $\bm{W}_o$ and $\bm{W}_c$, as they shape how the previous hidden state $\bm{h}_{t-1}$ is incorporated. Thus these matrices can be regarded as recurrent matrices as well.

The attention-based recurrent cell serves as another example where the weights on the old memory play a comparable role as $\bm{W}_h$ in vanilla RNNs. 
Generally, the attention-based recurrent cell has a memory state denoted by $\bm{h}_t \in \mathbb{R}^{d\times 1}$. The memory state aims to preserve long-term memory in the sequence and it is typically updated by cross attention to the old memory state $\bm{h}_{t-1}$ and $N$ input tokens at the current chunk $\bm{X}_t \in \mathbb{R}^{d \times N}$. In detail, the recurrent cell first generates the query $\bm{Q}_t \in \mathbb{R}^{d \times 1}$ from $\bm{h}_{t-1}$. It also extracts from both $\bm{h}_{t-1}$ and $\bm{X}_t$ to get the key $\bm{K}_t = [
\bm{K}_{t,h} \hspace{2mm} \bm{K}_{t,x} ] \in \mathbb{R}^{d \times (1+N)}$ and similarly for the value $\bm{V}_t = [
\bm{V}_{t,h} \hspace{2mm} \bm{V}_{t,x} ] \in \mathbb{R}^{d \times (1+N)}$, where 
\begin{align*}
	\begin{bmatrix}
		\bm{Q}_t \\ \bm{K}_{t,h} \\ \bm{V}_{t,h}
	\end{bmatrix}
	= 
	\begin{bmatrix}
		\bm{W}_Q \\ \bm{W}_{K} \\ \bm{W}_{V}
	\end{bmatrix}
	\bm{h}_{t-1},   
	\quad \text{and} \quad
	\begin{bmatrix}
		\bm{K}_{t, x} \\ \bm{V}_{t, x} 
	\end{bmatrix}
	=
	\begin{bmatrix}
		\bm{W}_{K} \\ \bm{W}_{V}
	\end{bmatrix}
	\bm{X}_{t}.
\end{align*}
Then the current memory state is obtained from $\bm{h}_t^\top = \text{softmax}\big(\bm{Q}_t^\top \bm{K}_t / \sqrt{d}\big) \bm{V}_t^\top$. Notably, the weight matrices $\bm{W}_Q$, $\bm{W}_K$, $\bm{W}_V\in \mathbb{R}^{d\times d}$ decide how the previous memory state contributes to the memory update, shaping the recurrent dynamics of attention-based recurrent cell. Hence we also consider them as the recurrent matrices.

Once recurrent matrices are identified, the decomposition of recurrent dynamics 
for these variants follows directly from the vanilla RNN case.
Furthermore, the ParaRNN framework, comprising segregation, parallelization, and aggregation, can be readily applied to these networks by block diagonalizing recurrent matrices and parallelizing the resulting constituents.

\section{Numerical Studies}
\label{sec:result}
We first empirically verify the trade-off between performance and training speed in ParaRNN through a simulation, and then validate the sufficiency of $d_s=2$ in three tasks: time series forecasting, sequential image classification, and long-sequence genomics classification. Four baseline models are considered: RNN, LSTM, GRU, and the Block Recurrent Transformer (BRT) \citep{hutchins2022blockrecurrent}.
To distinguish between the baselines and our models, we use the suffixes ``Vanilla" and ``Para". Unless otherwise stated, the block size $d_s$ is uniformly set to two for all our models. All experiments are performed on a single V100 GPU and their training scheme can be found in Section \ref{sec:appendix-training} of the Supplementary Material.

\subsection{Simulations for Verifying the Trade-off in ParaRNN} 
\label{sec:result-simulation}
Firstly, we assess the performance of a single ParaRNN layer on data generated by a vanilla RNN of size $d=128$. 
We generate $N=50,000$ samples, with the $i$-th sample represented as $\big(x_{i, 1:T} \: , \bm{z}_i\big)$ and following the form:
\begin{linenomath}
	\begin{align*}
		\bm{z}_i= \bm{W}_y \bm{h}_{i, T} + \bm{b}_y + \bm{\epsilon}_i \in \mathbb{R}^{10}
		\quad \text{and} \quad 
		\bm{h}_{i, t}= \mathrm{Tanh}(\bm{W}_h \bm{h}_{i, t-1} + \bm{W}_x x_{i, t}+ \bm{b}_h) \in \mathbb{R}^{128}
	\end{align*}
\end{linenomath}
for $1 \leq t \leq T, 1 \leq i \leq N$.
In particular, the inputs $x_{i, 1:T}$ of length $T=128$ are generated from the ARMA$(1,1)$ process with the AR and MA coefficients being 0.7 and 0.3, respectively. 
All entries of the additive errors $\bm{\epsilon}_i$, the weights $\bm{W}$'s, and the biases $\bm{b}$'s are independently sampled from $N(0, 1)$.
For modeling, we employ one ParaRNN layer with a fully-connected layer. 
We vary the hidden size $d_s$ of the constituent RNNs among $\{2^{j}| j\in \mathbb{N}: 0 \leq j \leq 7 \}$ and let $K = d/d_s$, with $d_s=d=128$ corresponding to the vanilla RNN.
The blue curve in Figure \ref{fig:two}(c) shows the test mean squared errors (MSE) averaged over 25 replicates, whereas the corresponding standard deviations are reported in Section \ref{sec:appendix-simulation} of the Supplementary Material. Notably, there is a substantial drop in loss when $d_s$ changes from $1$ to $2$, whereas the change in performance for $d_s \geq 2$ is not significant. 
The same observation holds true even when considering a larger variance for $\bm{\epsilon}_i$, as elaborated in Section \ref{sec:appendix-simulation}.
This suggests that ParaRNN with $d_s=2$ can mostly recover the recurrent dynamics of the vanilla model.

Next, we evaluate the training speed of one ParaRNN layer under the same setting of $T$, $d$, and $d_s$.
It is worth mentioning that the time complexity presented in Theorem \ref{thm:complexity} is theoretical; while for implementation, operations among matrix rows can be further parallelized by CUDA. Therefore, the empirical execution time increases linearly with larger $d_s$ in Figure \ref{fig:two}(c). The corresponding standard deviations are provided in Section \ref{sec:appendix-simulation}.

\subsection{Time Series Forecasting} 
\label{sec:result-ts}
The Electricity Transformer Temperature (ETT) and Weather (WTH) data sets \citep{zhou2021informer} are employed in this multivariate time series forecasting task.
The ETT data set includes seven features related to the long-term electric power development, collected over two years, with hourly records in ETTh$_1$ and ETTh$_2$, and 15‑minute records in ETTm$_1$.
The task is to predict the target value ``oil temperature" using all features. The WTH data set contains twelve climatological features collected from around 1,600 U.S. locations between 2010 and 2013, where the goal is to predict the ``Wet Bulb" temperature.

\begin{table}[t!]
	\centering
	\caption{(a) The total winning count for time series forecasting. (b)-(e) The classification accuracies for permuted sequential MNIST / pixel-by-pixel CIFAR-10 / noise-padded CIFAR-10 / EigenWorms tasks, respectively. Better performances are in bold. The ``$\ast$" denotes the use of chrono-initialization \citep{tallec2018can} in the corresponding model.}
	\renewcommand{\arraystretch}{0.9}
	\resizebox{1.\textwidth}{!}{
		\begin{tabular}{cc|cc|cc|cc}
			\specialrule{1pt}{0.1pc}{0.1pc}  
			Vanilla RNN&ParaRNN&Vanilla LSTM&ParaLSTM&Vanilla GRU&ParaGRU&Vanilla BRT&ParaBRT\\
			\specialrule{1pt}{0.1pc}{0.1pc}  
			\multicolumn{8}{c}{(a) Time Series Forecasting (Count)}\\
			\specialrule{0.5pt}{0.1pc}{0.1pc}  
			7 & \textbf{27} & 7 & \textbf{26} & 8 & \textbf{26} & 15 & \textbf{21}\\
			\specialrule{1pt}{0.1pc}{0.1pc}  
			\multicolumn{8}{c}{(b) Permuted Sequential MNIST Classification (Accuracy in \%)}\\
			\specialrule{0.5pt}{0.1pc}{0.1pc}  
			90.31 & \textbf{93.93} & \textbf{94.20} & 94.10 & 94.44 & \textbf{94.67}& 97.87 & \textbf{97.90} \\
			\specialrule{1pt}{0.1pc}{0.1pc}  
			\multicolumn{8}{c}{(c) Pixel-by-pixel CIFAR-10 Classification (Accuracy in \%)}\\
			\specialrule{0.5pt}{0.1pc}{0.1pc}  
			31.80 & \textbf{36.37} & 66.35 & \textbf{66.40} & \textbf{70.61} & 70.06 & 74.01 & \textbf{74.06} \\
			\specialrule{1pt}{0.1pc}{0.1pc}  
			\multicolumn{8}{c}{(d) Noise-padded CIFAR-10 Classification (Accuracy in \%)}\\
			\specialrule{0.5pt}{0.1pc}{0.1pc}  
			/ & / & 56.71$^\ast$ & \textbf{57.07}$^\ast$ & 52.90$^\ast$ & \textbf{53.40}$^\ast$ & 68.85 & \textbf{68.91} \\ 
			\specialrule{1pt}{0.1pc}{0.1pc}  
			\multicolumn{8}{c}{(e) EigenWorms Classification (Accuracy in \%)}\\
			\specialrule{0.5pt}{0.1pc}{0.1pc}  
			43.33 & \textbf{43.59} & 42.82 & \textbf{43.59 }&  42.82 & \textbf{43.33} & 58.40 & \textbf{60.97} \\ 
			\specialrule{1pt}{0.1pc}{0.1pc}  
	\end{tabular}}
	\label{tab:exp}
\end{table}

Here and in Section \ref{sec:result-mnist}, we especially include the attention-based model, BRT. Different from the traditional transformers, BRT introduces a set of state vectors, $\bm{M} = (\bm{m}_1, \dots, \bm{m}_S) \in \mathbb{R}^{S \times d}$, where $S$ denotes the number of states and $d$ refers to the dimension of each state. These state vectors serve a similar role as the hidden states in RNNs by preserving memory from previous sequences. They are updated by self-attention to themselves and cross-attention to the input tokens, and thus operate in a recurrent way; please refer to the architecture details in \cite{hutchins2022blockrecurrent}. To apply our framework, we replace the multi-head self-attention of state vectors by the following: 
\begin{linenomath}
	\begin{align*}
		\text{ParaAttention}(\bm{M}) &= \text{Concat}[\text{head}_1, \dots, \text{head}_K] \bm{W}_O,\\
		\text{where}\hspace{3mm}\text{head}_k &= \text{Attention}(\bm{M}^{(k)}\bm{W}_Q^{(k)}, \bm{M}^{(k)}\bm{W}_K^{(k)}, \bm{M}^{(k)}\bm{W}_V^{(k)}),
	\end{align*} 
\end{linenomath}
and $\bm{M}^{(k)}= \bm{M}_{[:, \,(k-1)d_s + 1\,:\, kd_s]} \in \mathbb{R}^{S \times d_s}$. Here $\bm{W}$'s are projection matrices and we set $d_s = 2$. Compared to the multi-head attention, the attention of each head in our framework involves different source tokens $\bm{M}^{(k)}$, which are obtained by splitting $\bm{M}$'s feature dimension $d$. The number of heads $K$ is decided by $d = Kd_s$, which is different from the entirely heuristic choice in multi-head attention. As a result, the number of parameters involved in ParaAttention is significantly smaller than that in the vanilla multi-head attention. Following \cite{hutchins2022blockrecurrent} and \cite{huang2023encoding}, we insert one recurrent layer to the transformer, while the rest layers follow the Transformer-XL \citep{dai2019transformerxl} style; see the choices of hyperparameters in Table \ref{tab:setting} of Section \ref{sec:appendix-brt} in the Supplementary Material.
As for RNNs, LSTMs, and GRUs, we set the number of recurrent layers and the total hidden size to 2 and 128, respectively, and use the feedforward network to aggregate the hidden states. 

Table \ref{tab:exp}(a) reports the total winning count across all data sets, prediction horizons and evaluation metrics. Detailed results, including mean squared error (MSE) and mean absolute error (MAE) obtained from three replicates, are listed in Table \ref{tab:time-series} of Section \ref{sec:appendix-ts} in Supplementary Material.
Overall, adopting our framework incurs no significant loss; in some cases, it even outperforms the baselines, possibly due to the ease of training.
Besides, to support Theorem \ref{thm:dense}, we count the recurrence feature types learned by vanilla models in Figure \ref{fig:derivative}(b), where the more opaque and transparent colors correspond to R-1 and C-1 features, respectively. Notably, no other types are observed in our experiments with vanilla models.


\subsection{Sequential Image Classification} \label{sec:result-mnist}
We consider three multi-class sequential image classification tasks. The first, permuted sequential MNIST, is based on the MNIST dataset \citep{lecun1998grad} that comprises $28\times 28$ grayscale images of handwritten digits from 0 to 9. For this task, the inputs are flattened to sequences with length $T=784$ and input dimension $d_\mathrm{in}=1$, and then permuted using a predetermined random order. The second and third tasks, pixel-by-pixel CIFAR-10 and noise-padded CIFAR-10, employ the CIFAR-10 dataset \citep{krizhevsky2009learning} which contains $32\times 32$ color images across 10 object classes. The pixel-by-pixel CIFAR-10 task takes flattened images with RGB channels as inputs, resulting in sequences with $T=1024$ and $d_\mathrm{in}=3$. In noise-padded CIFAR-10, images are processed row-wise and flattened along RGB channels, producing 96-dimensional sequences, each of length 32. Furthermore, a random noise is padded after the first 32 inputs to create sequences with $T=1000$ and $d_\mathrm{in}=96$. This task poses the greatest challenge as it requires retaining information over long sequences.

For the permuted sequential MNIST and pixel-by-pixel CIFAR-10 tasks, we employ 128-dimensional ParaRNN, ParaLSTM and ParaGRU and their corresponding baselines. A FC layer or position-wise feedforward network is used during the aggregation step, with an additional linear projection layer applied to obtain outputs.
Detailed hyperparameter settings are provided in Section \ref{sec:appendix-brt} of the Supplementary Material. We also include BRT as a baseline, where the architecture design is the same as in Section \ref{sec:result-ts} and the hyperparameter settings are included in Table \ref{tab:setting} of Section \ref{sec:appendix-brt}.
For the noise-padded CIFAR-10 task, given that baseline RNN and LSTM models perform only marginally better than random guess, as reported in \cite{rusch2022long, rusch2021unicornn, chang2018antisymmetricrnn}, we exclude RNNs and follow \cite{rusch2022long} to adopt the chrono-initialization method from \cite{tallec2018can} for LSTMs and extend its applications to GRUs as well. 

The final test accuracies for both baseline and proposed models are presented in Table \ref{tab:exp}(b)-(d). ParaRNN surpasses its baseline, suggesting that it has already offered adequate recurrent dynamics to fit the data while the additional aggregation step may further aid learning. ParaLSTM, ParaGRU and ParaBRT exhibit comparable performance with their respective baselines, affirming that our framework generalizes well to various RNN variants.

Besides, Figure \ref{fig:recurrence-feature-type} presents the shift of recurrence feature types during training for vanilla RNN, LSTM and GRU. Specifically, their recurrent matrices are initialized using three distinct methods: random entries from a uniform distribution on $(-1/\sqrt{d}, 1/\sqrt{d})$, resulting in R-1 and C-1 features exclusively; an identity matrix, yielding solely R-1 features; and a block diagonal matrix with a mix of R-4 and C-2 features. Note that although these models have more than one recurrent matrix, they exhibit similar patterns of dominant low-order features as in Section \ref{subsec:dominant-feature}, and we thus just visualize one of them without loss of generality. 

\subsection{Very Long Sequences for Genomics Classification}
To assess the model's ability to classify very long sequences, we consider the EigenWorms dataset \citep{anthony2018the} that includes 259 sequences of length $T = 17984$. These sequences depict the motion of a worm with dimension $d_\mathrm{in} = 6$, and our target is to classify a worm as either wild-type or one of four mutant types. Two-layer RNN, LSTM, GRU, each with a hidden size of 32 following \cite{rusch2022long}, and four-layer BRT models are considered. Unlike in previous experiments, we adapt all BRT layers to be recurrent due to the excessive length of the sequences, preventing the out-of-memory issue; additional hyperparameter configurations can be found in Table \ref{tab:setting}. The average test accuracies over 10 random initializations for both baseline and proposed models are reported in Table \ref{tab:exp}(e). ParaRNN, ParaLSTM and ParaGRU perform comparably to their counterparts, whereas an improvement can be observed for ParaBRT. We speculate that since this task handles extended sequences with a limited sample size, the training of ParaBRT may benefit from the block diagonal structure of recurrent matrices without compromising recurrent dynamics.

\section{Conclusion and Discussion} \label{sec:conclusion}
This paper revisits the recurrent dynamics of vanilla RNNs and shows that they can be decoupled in an approximation sense. Motivated by this observation, we introduce an alternative recurrent model, ParaRNN, whose additive structure separates recurrent dynamics into tractable components characterized by recurrence features. This structure also naturally enables parallelization and improves computational efficiency. A guideline for hyperparameter selection is provided for practitioners to help them balance the trade-off between recurrence feature richness and parallelization efficiency. We establish approximation properties and non‑asymptotic prediction error bounds for ParaRNN in a nonparametric regression setting, demonstrating that it achieves expressive power comparable to vanilla RNNs. Empirical results on three sequential modeling tasks further confirms its effectiveness.

The proposed methodology suggests three possible extensions. Firstly, our framework is highly extensible and not limited to verified examples in Section \ref{sec:extension}. Other recurrent models, such as more RNN variants and combinations with state-of-the-art models like Transformer, can also leverage our framework to achieve more interpretable recurrent dynamics, thereby making them more suitable for statistical modeling.
Providing theoretical guarantees for these extensions also remains an important open problem.
Secondly, the theoretical analysis in Section \ref{subsec:theory} is conducted under a standard machine learning setup, where the sample size corresponds to the number of independent sequences $N$. Extending the results to classical time‑series settings, where the number of time points $T$ diverges, is left for future work.
Finally, to accommodate high‑dimensional inputs, additional low‑dimensional structures, such as sparsity or low-rankness, may be imposed on the input weight matrix $\bm{W}_x$.

\setlength{\bibsep}{5.5pt}

\clearpage
\hypersetup{pageanchor=false}
\providecommand{\theHpage}{}
\renewcommand{\theHpage}{supp.\arabic{page}}
\setcounter{page}{1}
\setcounter{section}{0}
\setcounter{subsection}{0}
\setcounter{subsubsection}{0}
\setcounter{equation}{0}
\setcounter{figure}{0}
\setcounter{table}{0}
\setcounter{algorithm}{0}
\setcounter{proposition}{3}
\numberwithin{equation}{section}
\numberwithin{lemma}{section}
\numberwithin{theorem}{section}
\numberwithin{assumption}{section}
\numberwithin{proposition}{section}
\numberwithin{definition}{section}
\numberwithin{figure}{section}
\renewcommand{\thesection}{\Alph{section}}
\providecommand{\theHsection}{}
\providecommand{\theHsubsection}{}
\providecommand{\theHsubsubsection}{}
\providecommand{\theHequation}{}
\providecommand{\theHfigure}{}
\providecommand{\theHtable}{}
\providecommand{\theHalgorithm}{}
\renewcommand{\theHsection}{supp.\Alph{section}}
\renewcommand{\theHsubsection}{supp.\Alph{section}.\arabic{subsection}}
\renewcommand{\theHsubsubsection}{supp.\Alph{section}.\arabic{subsection}.\arabic{subsubsection}}
\renewcommand{\theHequation}{supp.\Alph{section}.\arabic{equation}}
\renewcommand{\theHfigure}{supp.\Alph{section}.\arabic{figure}}
\renewcommand{\theHtable}{supp.\arabic{table}}
\renewcommand{\theHalgorithm}{supp.\arabic{algorithm}}
\renewcommand\thealgorithm{S.\arabic{algorithm}}
\renewcommand{\thetable}{S.\arabic{table}}
\begingroup
\renewcommand\maketitlehookc{\vspace{-10ex}}
\title{Supplementary Material for  ``ParaRNN: An Interpretable and Parallelizable Recurrent Neural Network for Time-Dependent Data''}
\date{\vspace{2.25em}\today}
\setlength{\droptitle}{-6em}
\maketitle
\endgroup
\setlength{\parindent}{16pt}

\begin{abstract}
	This supplementary material consists of four sections. Section \ref{appendix:sec-feature-simulation} presents simulation results supporting the prevalence of low-order recurrence features as discussed in Section \ref{subsec:dominant-feature} of the main paper.
	Section \ref{appendix-sec:proofl} gives technical proofs for theoretical results in Sections \ref{sec:motivation} and \ref{sec:pararnn} of the main paper, while Section \ref{appendix:sec:new-theory} includes the proofs for the justifications in Section \ref{subsec:theory} of the main paper. 
	Section \ref{appendix-sec:exp} provides more descriptions on the visualization of Figure \ref{fig:derivative}, along with further details on model settings, training scheme and results of Section \ref{sec:result} in the main paper. An additional empirical example, namely the adding problem, is also included in Section \ref{appendix-sec:exp}. 
\end{abstract}

\section{Additional Simulation Results for Section \ref{subsec:dominant-feature}}\label{appendix:sec-feature-simulation}
In this setting, we consider an RNN model as the data generating process (DGP). Given an input sequence $\bm{X} = (\bm{x}_1, \cdots,\bm{x}_T) \in \mathbb{R}^{T \times d_\mathrm{in}}$, the output $\bm{z} \in \mathbb{R}^d$ is generated by
\begin{align}
	\bm{h}_t = \text{Tanh}(\bm{W}_h^{\ast} \bm{h}_{t-1} + \bm{W}_x^{\ast} \bm{x}_t) \in \mathbb{R}^{d}, 
	\hspace{3mm} \text{and }
	\bm{z} = \bm{h}_T + \bm{\epsilon},
	\label{eq:simulation-dgp}
\end{align} 
where we do not introduce any other transformation between $\bm{h}_T$ and $\bm{z}$ for model identifiability.
We set $T=128$, $d_\mathrm{in}=1$, $d=128$, $\bm{h}_0=\bm{0}$, and generate entries of matrices $\bm{W}_h^\ast$, $\bm{W}_x^\ast \in \mathbb{R}^{d \times d_\mathrm{in}}$ and noise $\bm{\epsilon} \in \mathbb{R}^d$ independently from  $N(0, 0.1^2)$, $N(2, 2^2)$ and $N(0, 0.01^2)$. Besides, we fix the training sample size $N=50,000$.

To fit the data, we adopt an RNN model with the same recurrent layer as in \eqref{eq:simulation-dgp} and directly take $\bm{h}_T$ as the output. The number of parameters is $d^2 + dd_\mathrm{in} = 16,512$, which is smaller than the sample size. Three different initializations of recurrence matrices are considered: random entries from a uniform distribution on $(-1/\sqrt{d}, 1/\sqrt{d})$, resulting in R-1 and C-1 features exclusively; an identity matrix, yielding solely R-1 features; and a block diagonal matrix with a mix of R-4 and C-2 features. 
The training objective is to minimize the mean squared error (MSE) loss and the Adam optimizer is applied with an initial learning rate of $10^{-4}$. A reducing learning rate scheduler that drops the learning rate by a factor of 2 if validation performance no longer improves for three consecutive epochs is applied. Note that we have also explored a constant learning rate scheduler, and since it yields the same results as the reducing one, we only present one of them for conciseness.

\begin{figure}
	\vspace{-1cm}
	\centering
	\includegraphics[width=0.85\linewidth]{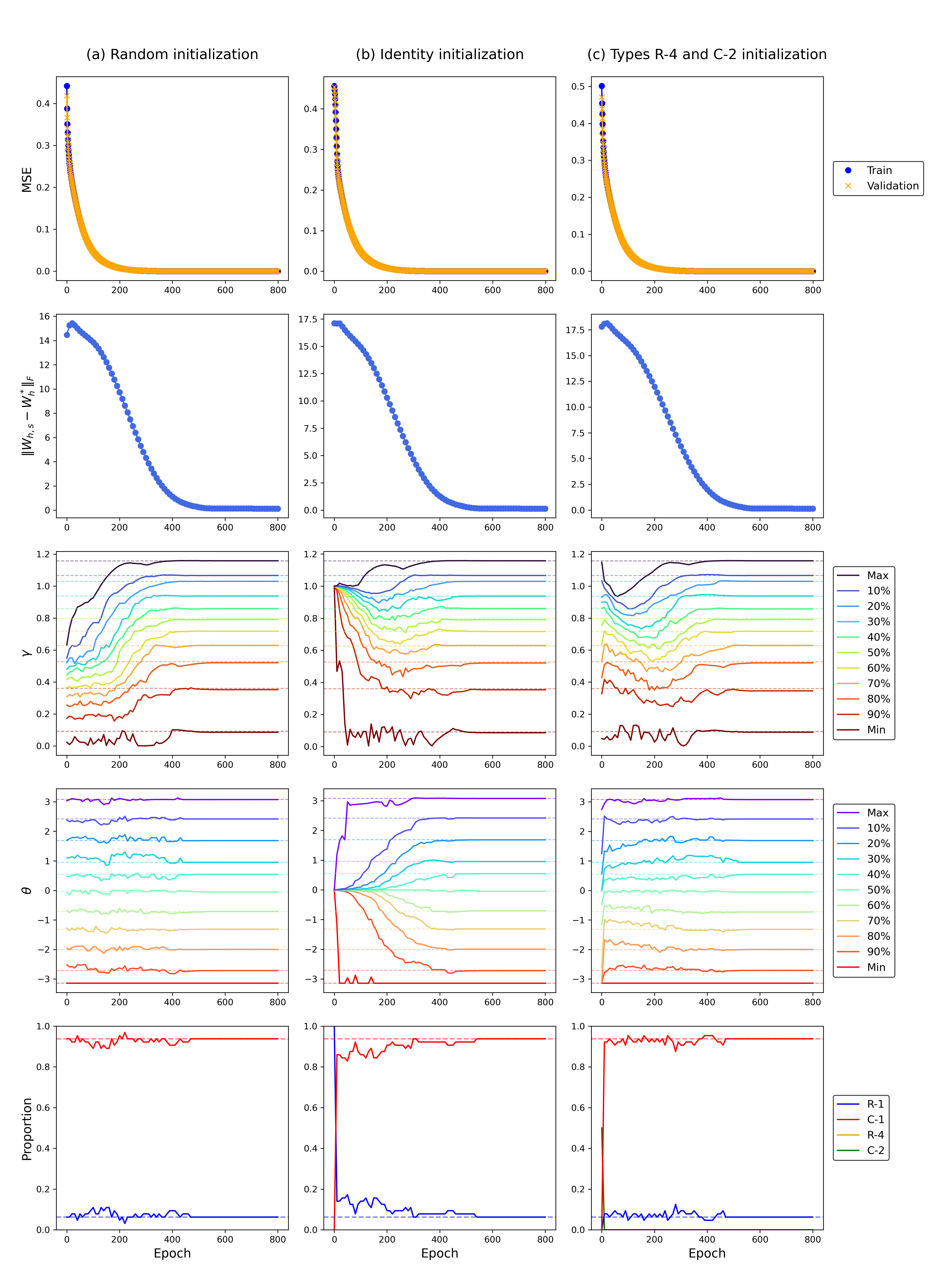}
	\caption{Simulation setting with entries of $\bm{W}_h^\ast$ in DGP from normal distribution: One-layer {RNN} with a {reducing} learning rate schedule. The first row displays the training and validation MSEs, while subsequent rows detail the learning of recurrent matrix in terms of $\lVert \bm{W}_{h, s}-\bm{W}_h^\ast \rVert_\mathrm{F}$, the percentiles of $\gamma$ and $\theta$, and the distribution of recurrence feature types. Three columns correspond to different initialization methods.}
	\label{fig:stat-random-adaptive}
\end{figure}

\begin{figure}
	\centering
	\includegraphics[width=\linewidth]{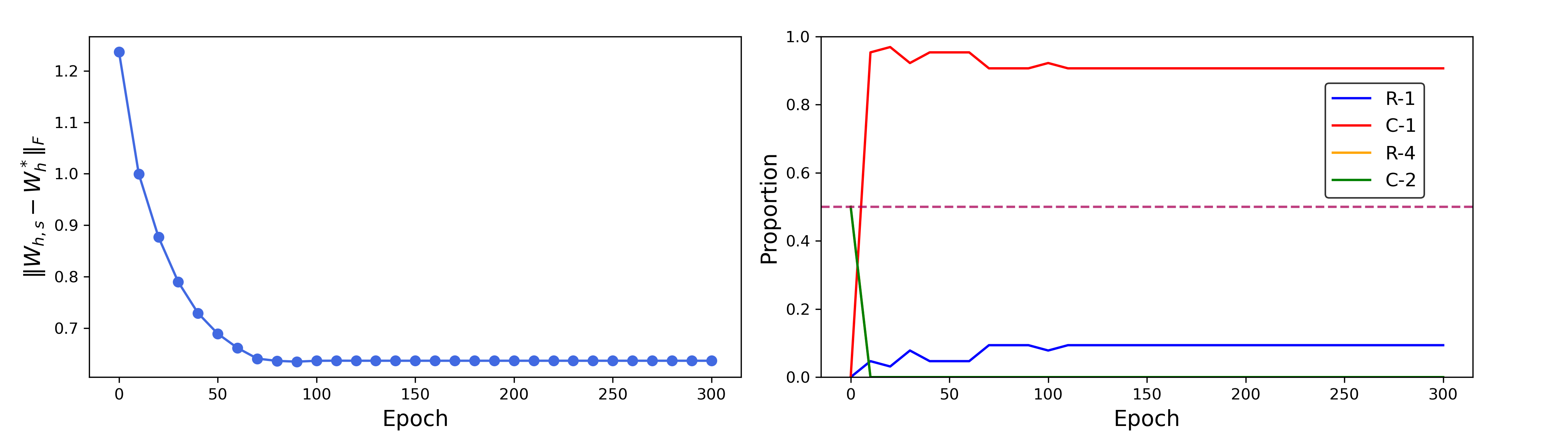}
	\caption{Simulation setting with $\bm{W}_h^\ast$ in DGP comprising R-4 and C-2 features: One-layer RNN with a {reducing} learning rate schedule. The left figure shows the learning of recurrent matrix in terms of $\lVert \bm{W}_{h, s}-\bm{W}_h^\ast \rVert_\mathrm{F}$, whereas the right presents the distribution of recurrence feature types.}
	\label{fig:stat-rnn-r4init}
\end{figure}

Under the asymptotic framework, overfitting is not an issue, as confirmed by overlapping training and validation loss curves in the first row of Figure \ref{fig:stat-random-adaptive}. Despite the non-convex nature of this optimization problem, it is possible for training to converge to a global minimum with a sufficiently large sample size. In this case, since the model \eqref{eq:simulation-dgp} is identifiable, the updated recurrence matrices from each epoch during training, denoted as $\{\bm{W}_{h, 1}, \bm{W}_{h, 2}, \dots, \bm{W}_{h, s}, \dots\}$ where $s$ is the epoch index, can converge to its true value $\bm{W}_h^\ast$. The second row of Figure \ref{fig:stat-random-adaptive} presents the difference between $\bm{W}_{h, s}$ and $\bm{W}_h^\ast$ in terms of the Frobenius norm across epochs, confirming the trend of convergence. 
Furthermore, by Theorem 1.1 from \cite{stewart1990matrix}, the eigenvalues of these updated recurrent matrices, which drive the recurrent dynamics of trained RNNs, can also converge to those of $\bm{W}_h^\ast$. 
The third and fourth rows of Figure \ref{fig:stat-random-adaptive} demonstrate that the percentiles of $\gamma, \theta$ approach the true values. Finally, while the recurrent matrix can converge, recurrence feature types may not align with those of $\bm{W}_h^\ast$. Since by the denseness of low-order features in our Theorem \ref{thm:dense}, trained RNNs tend to provide these features, even if $\bm{W}_h^\ast$ yields higher-order features. To verify this, we introduce a different DGP model where $\bm{W}_h^\ast$ in \eqref{eq:simulation-dgp} comprises a mix of R-4 and C-2 features. The corresponding recurrent matrix and recurrence feature types after gradient descent updates are provided in Figure \ref{fig:stat-rnn-r4init}. Notably, the learned recurrence features transition to low-order. Besides, both Figure \ref{fig:stat-rnn-r4init} and the last row of Figure \ref{fig:stat-random-adaptive} show a substantial proportion of Type C features. These are consistent with our findings in the real-world setting.


\section{Proof of Theoretical Results in Sections \ref{sec:motivation} and \ref{sec:pararnn}} \label{appendix-sec:proofl}
\subsection{Proof of Theorem \ref{thm:dense}}
\begin{proof}
	The denseness of $\mathbb{M}_d^2$ in $\mathbb{M}_d$ is directly given by Theorem 1 of \cite{hartfiel1995dense}. To prove that $\mathbb{M}_d^1$ is not dense in $\mathbb{M}_d$, it is sufficient to construct a matrix $\bm{W} \in \mathbb{M}_d$ and an $\epsilon > 0$ such that there exists no matrix $\tilde{\bm{W}} \in \mathbb{M}_d^1$ satisfying $\lVert{\tilde{\bm{W}} -\bm{W}}\rVert_\Fr \leq \epsilon$.
	
	Without loss of generality, we construct a matrix $\bm{W}$ with rank $R > 1$ in the following form:
	\begin{align*}
		\bm{W} = \begin{cases}
			\bm{C}_1(\gamma, \theta) \oplus \bm{0}_{d-R} & \text{for } R=2, \\ 
			\bm{C}_1(\gamma, \theta) \oplus \bm{J}_1(\lambda_1) \oplus \cdots \oplus \bm{J}_1(\lambda_{R-2}) \oplus \bm{0}_{d-R} & \text{for } R > 2,
		\end{cases}
	\end{align*}
	where $\lambda_1, \dots, \lambda_{R-2} \in \mathbb{R}$ are arbitrary and we let $\gamma=\sqrt{2}$ and $\theta = -\pi/4$.
	Note that $\bm{W}$ has nonzero complex eigenvalues $\gamma e^{i\theta}$ and  $\gamma e^{-i\theta}$ and for $R>2$, real eigenvalues $\lambda_1, \dots, $ $\lambda_{R-2}$.
	By \cite{meyer2023matrix}, the eigenvalues of $\bm{W}$ must vary continuously with the entries of $\bm{W}$. It implies that for sufficiently small $\epsilon>0$, any ${\bm{W}}^\prime$ satisfying $\lVert {\bm{W}}^\prime -\bm{W} \rVert_\Fr \leq \epsilon$ must retain at least one nonzero complex eigenvalue. Since any $\tilde{\bm{W}} \in \mathbb{M}_d^1$ can only have real eigenvalues, we have no matrix $\tilde{\bm{W}} \in \mathbb{M}_d^1$ satisfying $\lVert{\tilde{\bm{W}} -\bm{W}}\rVert_\Fr \leq \epsilon$. The proof is thus complete.
\end{proof}

\subsection{Proof of Proposition \ref{thm:recurrence-feature}}
\begin{proof}
	Suppose $\bm{W}_h = (\xi_{ij})_{1 \leq i,j \leq d} \in \mathbb{R}^{d \times d}$ is a random matrix.
	For clarity, we first denote a series of events, among which our target is to find the probabilities of events $C$ and $E$:
	\begin{itemize}
		\item Event $A$: All eigenvalues of $\bm{W}_h$ are distinct.
		\item Event $B$: All eigenvalues of $\bm{W}_h$ are real.
		\item Event $C$: A RNN whose recurrent matrix is $\bm{W}_h$ only provides Type R-1 recurrence features.
		\item Event $D$: A RNN whose recurrent matrix is $\bm{W}_h$ only provides Types R-1 and C-1 recurrence features.
		\item Event $E$: A RNN whose recurrent matrix is $\bm{W}_h$ only provides Types R-1, R-2, and C-1 recurrence features.
	\end{itemize}
	By \cite{tao2017random}, $P(A) = 1$ when $\xi_{ij}$'s has a continuous joint distribution. As demonstrated in Section \ref{subsec:recurrence-features}, which type of recurrence features that an RNN is able to provide is determined by the eigenvalues of the recurrent matrix $\bm{W}_h$. Therefore, the events A and D are equivalent and $P(D) = P(A)= 1$. Since $D \subseteq E$, $P(E) \geq P(D) = 1$; and thus we prove that with probability $1$, the RNN whose recurrent matrix is $\bm{W}_h$ only provides Types R-1, R-2, and C-1 recurrence features.
	
	By Corollary 7.1 of \cite{alan1997matrix}, when $\xi_{ij} \overset{\text{i.i.d}}{\sim} N(0, 1)$, $P(B) =  1 / 2^{d(d-1)/4}$. Since $C = A \cap B$, $P(C) \leq P(B) =1 / 2^{d(d-1)/4}$. Hence with probability at most $1 / 2^{d(d-1)/4}$, the RNN whose recurrent matrix is $\bm{W}_h$ only provides Type R-1 recurrence features.
\end{proof}

\subsection{Proof of Theorem \ref{thm:complexity}}
\begin{proof}
	Suppose $\bm{A} \in \mathbb{R}^{m \times n}$, $\bm{B} \in \mathbb{R}^{n \times p}$ and $\bm{v} \in \mathbb{R}^{n \times 1}$ are dense. We take the time complexity of computing $\bm{A}\bm{B}$ and $\bm{A}\bm{v}$ as $\mathcal{O}(mnp)$ and $\mathcal{O}(mn)$, respectively. Note that this is a naive version of counting the complexity and it can be replaced when a more efficient algorithm for matrix multiplication is used. Moreover, since the $K$ small RNNs in \eqref{eq:nonlinear-rnn-diag} are run in parallel, it suffices to analyze the branch with the largest hidden size $d_\mathrm{max}$, indexed by $k := \arg\max_i d_i$.
	
	\textbf{Forward Propagation.} For any $1 \leq t \leq T$, the complexities of computing $\bm{W}_h^{(k)} \bm{h}_{t-1}^{(k)}$ and $\bm{W}_x^{(k)} \bm{x}_t$ are $\mathcal{O}(d_\mathrm{max}^2)$ and $\mathcal{O}(d_\mathrm{max}d_\mathrm{in})$, respectively. Thus, the $T$-step time complexity is $\mathcal{O}(T d_\mathrm{max}^2 + T d_\mathrm{max} d_\mathrm{in})$. In addition, ParaRNN applies a linear projection to the concatenated hidden states, which incurs a matrix-vector multiplication cost of $\mathcal{O}(d^2)$. Then the overall time complexity of forward propagation is $\mathcal{O}(T d_\mathrm{max}^2 + T d_\mathrm{max} d_\mathrm{in} + d^2)$.
	
	\textbf{Backward Propagation.} For the backward propagation, we use the complexity of deriving $\partial \bm{h}_T^{(k)} / \partial \text{vec} \big( \bm{W}_h^{(k)} \big)$ as the time measure, which is representative since its derivation needs to traverse through the longest time steps. Denote $\bar{\bm{h}}_t^{(k)} = \bm{W}_h^{(k)} \bm{h}_{t-1}^{(k)} + \bm{W}_x^{(k)} \bm{x}_t + \bm{b}^{(k)}$ as the pre-activation hidden state. Then
	\begin{align*}
		\frac{\partial \bm{h}_T^{(k)}}{\partial \text{vec} \big( \bm{W}_h^{(k)} \big)} = &  
		\frac{\partial \bm{h}_T^{(k)}}{\partial \bar{\bm{h}}_T^{(k)}} 
		\Bigl(
		\frac{\partial \bar{\bm{h}}_T^{(k)}}{\partial \text{vec} \big( \bm{W}_h^{(k)} \big)} + 
		\frac{\partial \bar{\bm{h}}_T^{(k)}}{\partial \bm{h}_{T-1}^{(k)}} 
		\frac{\partial \bm{h}_{T-1}^{(k)}}{\partial \text{vec} \big( \bm{W}_h^{(k)} \big)}
		\Bigr) \\
		= & 
		\text{Diag}\big( \sigma_h'(\bar{\bm{h}}_T^{(k)}) \big) 
		\Bigl\{
		( \bm{h}_{T-1}^{(k)} \otimes \bm{I}_{d_\mathrm{max}})^\top + 
		\bm{W}_h^{(k)} 
		\frac{\partial \bm{h}_{T-1}^{(k)}}{\partial \text{vec} \big( \bm{W}_h^{(k)} \big)}
		\Bigr\} \\
		= & 
		\text{Diag}\big( \sigma_h'(\bar{\bm{h}}_T^{(k)}) \big)  
		( \bm{h}_{T-1}^{(k)} \otimes \bm{I}_{d_\mathrm{max}})^\top \\
		+ & 
		\sum_{j=1}^{T-1}
		\Bigl[
		\prod_{i=0}^{j-1}
		\text{Diag}\big( \sigma_h'(\bar{\bm{h}}_{T-i}^{(k)}) \big)
		\bm{W}_h^{(k)} 
		\Bigr]
		\text{Diag}\big( \sigma_h'(\bar{\bm{h}}_{T-j}^{(k)}) \big) 
		( \bm{h}_{T-j-1}^{(k)} \otimes \bm{I}_{d_\mathrm{max}})^\top. 
	\end{align*}
	For the simplicity of notations, we define two terms $\bm{B}_t := \text{Diag}\big( \sigma_h'(\bar{\bm{h}}_t^{(k)}) \big)  ( \bm{h}_{t-1}^{(k)} \otimes \bm{I}_{d_\mathrm{max}})^\top$ and $\bm{C}_t := \text{Diag}\big( \sigma_h'(\bar{\bm{h}}_{t}^{(k)}) \big)\bm{W}_h^{(k)}$. Consequently, we can represent the above derivative as $\partial \bm{h}_T^{(k)} / \partial \text{vec} \big( \bm{W}_h^{(k)} \big) = \bm{B}_T + \sum_{j=1}^{T-1} (\prod_{i=0}^{j-1} \bm{C}_{T-i}) \bm{B}_{T-j}$. Since both $\text{Diag}\big( \sigma_h'(\bar{\bm{h}}_t^{(k)}) \big)$ and $( \bm{h}_{t-1}^{(k)} \otimes \bm{I}_{d_\mathrm{max}})^\top$ are sparse and have special structures, the complexities of computing $\bm{B}_t$ and $\bm{C}_t$ are $\mathcal{O}(d_\mathrm{max}^2)$. The summation, $\sum_{j=1}^{T-1} (\prod_{i=0}^{j-1} \bm{C}_{T-i}) \bm{B}_{T-j}$ , has a complexity of $\sum_{j=1}^{T-1} j d_\mathrm{max}^3$, which accounts for the major complexity since $\bm{C}_t$'s are dense. As a result, the T-step time complexity of backward propagation is $\mathcal{O}(T^2 d_\mathrm{max}^3)$.

	Next, we consider the aggregation step in ParaRNN. During the backward propagation, we need to compute $\partial \bm{h}_T / \partial \text{vec} \big( \bm{W}_h^{(k)} \big) = \partial \bm{h}_T/  \partial \bm{h}_T^{(k)}  \cdot \partial \bm{h}_T^{(k)} / \partial \text{vec} \big( \bm{W}_h^{(k)} \big)$. This introduces an additional cost of order $\mathcal{O}(dd_{\mathrm{max}}^3)$. Combining the recurrent and aggregation costs, the total $T$-step backward complexity of a single ParaRNN layer is $\mathcal{O}(T^2 d_\mathrm{max}^3 + dd_\mathrm{max}^3)$.
\end{proof}


\section{Proof of Theoretical Results in Section \ref{subsec:theory}}\label{appendix:sec:new-theory}
\subsection{Proof of Approximation Error Bound in Theorem \ref{proposition:approx-rate-single-step}} \label{subsec:approx-proof-general}
This subsection proves the approximation error bound for ParaRNNs as defined in Section \ref{subsec:theory}. Specifically, the equivalence of ParaRNNs and FNNs is first illustrated in Section \ref{subsec:equivalence}, where the formal statement can be found in Proposition \ref{prop:equivalence-pararnn-fnn}. Then the proof of Theorem \ref{proposition:approx-rate-single-step} is given based on this equivalence in Section \ref{subsec:approx-proof}.

\subsubsection{Equivalence of ParaRNNs and FNNs}\label{subsec:equivalence}
To begin with, we first define the ParaRNN-based and FNN-based function classes.
\paragraph{Class of ParaRNNs} 
Denote $\mathcal{PRNN}$ as a class of ParaRNNs with bounded outputs:
\begin{align*}
	\mathcal{PRNN}^{(T)}_{d_\mathrm{in}, d_\mathrm{out}} (d, d_s, L, U) =  \{&\mathcal{N}: \mathcal{N}(\bm{X}_T) =  \mathcal{Q} \circ \mathcal{F} \circ \mathcal{R}_L \circ \cdots \circ \mathcal{R}_1 \circ \mathcal{P} (\bm{X}_T) \\
	&\text{ with } 
	\sup_{\bm{X}_T \in \mathbb{R}^{d_\mathrm{in} \times T}, t\in \{1, \dots, T\}} \norm{\mathcal{N}(\bm{X}_T)[t]}_\infty \leq U
	\},
\end{align*}
where the definitions of  $\mathcal{P}, \mathcal{R}_1, \dots, \mathcal{R}_L, \mathcal{F}, \mathcal{Q}$ are given in Section \ref{subsec:theory} of the main paper. Note that the function class $\mathcal{F}^{(T)}_{d_\mathrm{in}, d_\mathrm{out}, d, d_s, L, U}$ in the main paper can be equivalently written as $\mathcal{F}^{(T)}_{d_\mathrm{in}, d_\mathrm{out}, d, d_s, L, U} = \{\mathcal{N}(\bm{X}_T)[T]: \mathcal{N} \in 	\mathcal{PRNN}^{(T)}_{d_\mathrm{in}, d_\mathrm{out}} (d, d_s, L, U) \}$. 

For the clarity of notations, we elaborate more on the formulation of these layers. Given an input sequence $\bm{X}_T= (\bm{x}_{1}, \dots, \bm{x}_{T}) \in \mathbb{R}^{d \times T}$, each ParaRNN layer $\mathcal{R}_l$ takes the form of
\begin{align}
	\label{eq:rnn-block-diag}
	\mathcal{R}_l(\bm{X}_T)[t_0] = \sigma_h(\bm{W}_{h, l}	\mathcal{R}_l(\bm{X}_T)[t_0-1] + \bm{W}_{x, l} \bm{x}_{t_0} + \bm{b}_{l}), \hspace{3mm}\text{for }1\leq t_0 \leq T.
\end{align}
Here $\bm{W}_{h, l} = \oplus_{k=1}^K \bm{W}_{h, l}^{(k)} \in \mathbb{R}^{d\times d}$, $\bm{W}_{x, l} = ( {\bm{W}_{x, l}^{(1)}}^\top, \dots,  {\bm{W}_{x, l}^{(K)}}^\top)^\top \in \mathbb{R}^{d \times d}$, $\bm{b}_l = ({\bm{b}_l^{(1)}}^\top, \dots,$ $ {\bm{b}_l^{(K)}}^\top)^\top \in \mathbb{R}^d$, where $\bm{W}_{h, l}^{(k)} \in \mathbb{R}^{d_s \times d_s}, \bm{W}_{x, l}^{(k)} \in \mathbb{R}^{d_s \times d}, \bm{b}_l^{(k)} \in \mathbb{R}^{d_s}$ are the weights and bias of the recurrent cell defined in \eqref{eq:pararnn} of the main paper. Besides, the FC layer $\mathcal{F}$, and the linear transformation maps $\mathcal{P}, \mathcal{Q}$ are all applied position-wise. Therefore, we may write $\mathcal{F}(\bm{X}_T)[t_0] = \sigma(\bm{W}_f \bm{x}_{t_0}  + \bm{b}_f) \in \mathbb{R}^d$ for the given weight matrix $\bm{W}_f\in \mathbb{R}^{d\times d}$ and bias vector $\bm{b}_f \in \mathbb{R}^d$, $\mathcal{P}(\bm{X}_T)[t_0] = \bm{P}\bm{x}_{t_0}\in \mathbb{R}^d$ for the given weight matrix $\bm{P} \in \mathbb{R}^{d \times d_\mathrm{in}}$, and $\mathcal{Q}(\bm{X}_T)[t_0] = \bm{Q}\bm{x}_{t_0} \in \mathbb{R}^{d_\mathrm{out}}$ for the given weight matrix $\bm{Q} \in \mathbb{R}^{d_\mathrm{out} \times d}$.

\paragraph{Class of Feedforward Neural Networks}
We denote $\mathcal{FNN}$ as a class of feedforward neural networks:
\begin{align*}
	\mathcal{FNN}_{d_\mathrm{in}, d_\mathrm{out}} (d, L) = \{\mathcal{N}: \mathcal{N}(\bm{x} ) = \mathcal{F}_L \circ \cdots \circ \mathcal{F}_1(\bm{x}) \},
\end{align*}
where $\mathcal{F}_l$ represents a single feedforward layer of the form: $\mathcal{F}_l(\bm{x}) = \sigma(\bm{A}_l \bm{x} + \bm{b}_l)$ for $1 \leq l \leq L-1$ and $\mathcal{F}_L(\bm{x}) = \bm{A}_L \bm{x} + \bm{b}_L$. Here $\bm{A}_l \in \mathbb{R}^{d_l \times d_{l-1}}$ with $d_0 = d_\mathrm{in}$ and $d_L = d_\mathrm{out}$ are the weight matrices and $\bm{b}_l \in \mathbb{R}^{d_l}$ are the bias terms. Besides, the notation $d$ denotes the maximum width of the network, i.e., $d = \max_l d_l$. Note that to let FNN use the same input as the ParaRNN, we may vectorize the input $\bm{X}_T \in \mathbb{R}^{d_\mathrm{in} \times T}$ by stacking its columns, which results in $\vectorize(\bm{X}_T) = (\bm{x}_1^\top, \dots, \bm{x}_{T}^\top)^\top \in \mathbb{R}^{d_{\mathrm{in}}T}$.  

In addition to these two function classes, we also introduce a function class of the so-called modified ParaRNNs, which plays an intermediary role in proving the equivalence of $\mathcal{PRNN}$ and $\mathcal{FNN}$.

\paragraph{Class of Modified ParaRNNs} Define the modified recurrent layer of ParaRNN as
\begin{align*}
	\bar{\mathcal{R}}(\bm{X}_{T})[t_0]_i &= \sigma_I(\bm{W}_h \bar{\mathcal{R}}(\bm{X}_{T})[t_0-1] + \bm{W}_x \bm{x}_{t_0} + \bm{b})_i \\
	&= \begin{cases}
		\sigma_h(\bm{W}_h \bar{\mathcal{R}}(\bm{X}_{T})[t_0-1] + \bm{W}_x \bm{x}_{t_0}+ \bm{b})_i & \hspace{3mm}\text{if } i \in I\\
		(\bm{W}_h \bar{\mathcal{R}}(\bm{X}_{T})[t_0-1] + \bm{W}_x \bm{x}_{t_0} + \bm{b})_i & \hspace{3mm} \text{if } i \notin I
	\end{cases},
\end{align*}
where $\sigma_I$ is a partially activation function defined as
\begin{align*}
	\sigma_I(\bm{x})_i = \begin{cases}
		\sigma_h(x_i) & \hspace{3mm} \text{if } i \in I \\
		x_i   & \hspace{3mm} \text{if } i \notin I
	\end{cases}.
\end{align*}
In other words, its only difference with the ParaRNN layer is that the activation function is now partially activated on a specified set.
Similarly, we can also define the modified position-wise FC layer as 
\begin{align*}
	\bar{\mathcal{F}}(\bm{X}_{T})[t_0]_i = \sigma_I(\bm{W}_f \bm{x}_{t_0} + \bm{b}_f).
\end{align*}
Then denote a class of modified ParaRNNs with bounded outputs as $\mathcal{MPRNN}^{(T)}_{d_\mathrm{in}, d_\mathrm{out}}$ $(d, d_s, L, U)  = \{\mathcal{N}: \mathcal{N}(\bm{X}_{T}) =  \mathcal{Q} \circ \bar{\mathcal{F}}\circ \bar{\mathcal{R}}_L \circ \cdots \circ \bar{\mathcal{R}}_1 \circ \mathcal{P} (\bm{X}_T) \text{ with }
\sup_{\bm{X}_T \in \mathbb{R}^{d_\mathrm{in} \times T}, t\in \{1, \dots, T\}} $ $\norm{\mathcal{N}(\bm{X}_T)[t]}_\infty \leq U
\}.$

Proposition \ref{prop:equivalence-pararnn-fnn} proves the equivalence of $\mathcal{PRNN}$ and $\mathcal{FNN}$ by showing that any FNN/ ParaRNN can be represented by a deeper and wider ParaRNN/FNN. In particular, the proof of the former is established by demonstrating that any modified ParaRNN can be represented by a ParaRNN (Lemma \ref{lemma:prnn-mprnn}) and any FNN can be represented by a modified ParaRNN (Lemma \ref{lemma:mprnn-fnn}).
Without the loss of generality, we consider $\mathcal{PRNN}_{d_\mathrm{in}, d_\mathrm{out}} (d, d_s, L, U) :=\mathcal{PRNN}^{(T)}_{d_\mathrm{in}, d_\mathrm{out}} (d, d_s, L, U)$ and $\mathcal{MPRNN}_{d_\mathrm{in}, d_\mathrm{out}} (d, d_s, L, U) :=\mathcal{MPRNN}^{(T)}_{d_\mathrm{in}, d_\mathrm{out}} (d, d_s, L, U)$.
 
\begin{proposition}
	\label{prop:equivalence-pararnn-fnn}
	Let $t_0 \in \{1, \cdots, T\}$, $\bm{X} = (\bm{x}_1, \dots, \bm{x}_{T})$ and $\bm{X}_{t_0} = (\bm{x}_1, \dots, \bm{x}_{t_0})$. 
	\begin{itemize}
		\item[(i)] For any FNN $\tilde{\mathcal{N}} \in \mathcal{FNN}_{d_\mathrm{in}t_0, d_\mathrm{out}}(d, L)$, there exists a ParaRNN $\mathcal{N} \in \mathcal{PRNN}_{d_\mathrm{in}, d_\mathrm{out}} ((d_\mathrm{in} + 1)d + d_s,{d_s}, 2(T+L-1), U)$ with $d_s$ being divisible by $(d_\mathrm{in} + 1)d$ such that 
		\begin{align*}
			\tilde{\mathcal{N}}(\bm{X}_{t_0}) = \mathcal{N}(\bm{X})[t_0]
			\hspace{3mm}\text{for}\hspace{3mm} \bm{X}\in[0, 1]^{d_\mathrm{in} \times T}.
 		\end{align*}
 		\item[(ii)] For any ParaRNN $\mathcal{N} \in \mathcal{PRNN}_{d_\mathrm{in}, d_\mathrm{out}}(d, d_s, L, U)$, there exists an FNN $\tilde{\mathcal{N}} \in \mathcal{FNN}_{d_\mathrm{in}t_0, d_\mathrm{out}}$ $((2t_0-1)d, (t_0 + 1)L+2)$ such that
 		\begin{align*}
 		\mathcal{N}(\bm{X})[t_0]=\tilde{\mathcal{N}}(\bm{X}_{t_0} )
 		\hspace{3mm}\text{for}\hspace{3mm} \bm{X}\in \mathbb{R}^{d_\mathrm{in} \times T}. 
 		\end{align*}
	\end{itemize}
\end{proposition}

\begin{proof}[Proof of Proposition \ref{prop:equivalence-pararnn-fnn}(i)]
	The conclusion can be obtained by combining Lemmas \ref{lemma:prnn-mprnn} and \ref{lemma:mprnn-fnn}.
\end{proof}

\begin{proof}[Proof of Proposition \ref{prop:equivalence-pararnn-fnn} (ii)]
	This proof closely follows Lemma 9 of \cite{jiao2024approximation}. We first consider a single ParaRNN recurrent layer $\mathcal{R}$:
	\begin{align*}
		\mathcal{R}(\bm{X})[t] = \sigma_h(\bm{W}_h \mathcal{R}(\bm{X})[t-1] + \bm{W}_x \bm{x}[t] + \bm{b})
	\end{align*} 
	with $\bm{W}_h$ being a block diagonal matrix and each block of size $d_s \times d_s$; 
	and prove that there exists a FNN $\tilde{\mathcal{N}} \in \mathcal{FNN}_{t_0d, t_0d}((2t_0-1)d, t_0+1)$ satisfying that 
	\begin{align*}
		\tilde{\mathcal{N}}
		\begin{pmatrix}
			\bm{x}[1] \\ \bm{x}[2] \\ \vdots \\ \bm{x}[t_0]
		\end{pmatrix}
		= 
		\begin{pmatrix}
			\mathcal{R}(\bm{X})[1] \\ \mathcal{R}(\bm{X})[2] \\ \vdots \\ \mathcal{R}(\bm{X})[t_0]
		\end{pmatrix},
	\end{align*}
where $\bm{x}[i] \in \mathbb{R}^d$ is the $i$-th column of the matrix $\bm{X}$. 
As suggested by \cite{jiao2024approximation}, one can construct a FNN $\tilde{\mathcal{N}} = \mathcal{F}_{t_0+1} \circ \mathcal{F}_{t_0} \circ \cdots \circ \mathcal{F}_{1}$ and set their weight matrices and bias terms as 
\begin{align*}
	\tilde{\bm{A}}_1 = \begin{pmatrix}
		\bm{W}_x & \bm{0}_{d \times d} & \cdots & \bm{0}_{d \times d} \\
		\bm{0} _{d \times d}& \bm{I}_d & \cdots & \bm{0}_{d \times d} \\
		\bm{0}_{d \times d} & -\bm{I}_d & \cdots & \bm{0}_{d \times d} \\
		\vdots & \vdots & \ddots & \vdots \\
		\bm{0}_{d \times d} & \bm{0}_{d \times d} & \cdots & \bm{I}_d \\
		\bm{0}_{d \times d} & \bm{0}_{d \times d} & \cdots & -\bm{I}_d \\
	\end{pmatrix} \in \mathbb{R}^{(2t_0 - 1)d  \times t_0 d},
	\hspace{2mm}
	\tilde{\bm{b}}_1 = \begin{pmatrix}
		\bm{b} \\ 0 \\ 0 \\ \vdots \\ 0
	\end{pmatrix} \in \mathbb{R}^{(2t_0 - 1)d},
\end{align*}
for $l=2, \dots, t_0$, 
\begin{align*}
	\tilde{\bm{A}_l}  = \begin{pmatrix}
		\bm{I}_d & & & & & & & &  \\
		& \ddots & & & & & & & \\
		&  & \bm{I}_d & & & & & & \\
		&  & & \bm{I}_d & \bm{0}_{d \times d} &  \bm{0}_{d \times d}  & & & \\
		&  & & -\bm{I}_d & \bm{0}_{d \times d} &  \bm{0}_{d \times d}  & & & \\
		&  & & \bm{W}_h & \bm{W}_x &  -\bm{W}_x & & & \\
		& & & & & & \bm{I}_d  & &  \\
		&  & & & & & & \ddots & \\
		&  &  & & & & & & \bm{I}_d \\
	\end{pmatrix} \in \mathbb{R}^{(2t_0 - 1)d  \times t_0 d},
	\hspace{2mm}
	\tilde{\bm{b}}_l = \begin{pmatrix}
		0 \\ \vdots \\ 0 \\ \bm{b} \\ 0 \\ \vdots \\ 0
	\end{pmatrix}  \in \mathbb{R}^{(2t_0 - 1)d},
\end{align*}
and 
\begin{align*}
	\tilde{\bm{A}}_{t_0 + 1} =
	\begin{pmatrix}
		\bm{I}_d & -\bm{I}_d & & & & & & \\
		 & & \bm{I}_d & -\bm{I}_d & & & & \\
		 & &  & & \ddots & & & \\
		 & &  & &  &\bm{I}_d & -\bm{I}_d & \\
		  & &  & &  & & & \bm{I}_d \\
	\end{pmatrix} \in \mathbb{R}^{(2t_0 - 1)d  \times t_0 d},
	\hspace{2mm}
	\tilde{\bm{b}}_{t_0+1} = \bm{0}_{(2t_0 - 1)d} \in \mathbb{R}^{(2t_0 - 1)d}.
 \end{align*}
Note that $\tilde{\bm{A}}_l$ has $2l-4$ identity matrices in the upper left corner and $2t_0 - 2l$ in the lower right corner, $\tilde{\bm{b}}_l$ has $(2l-2)d$ zero rows before $\bm{b}$, and it does not matter whether $\mathcal{F}_{t_0+1}$ includes the ReLU activation. It can be verified that 
\begin{align*}
	\tilde{\mathcal{N}}
	\begin{pmatrix}
		\bm{x}[1] \\ \bm{x}[2] \\ \vdots \\ \bm{x}[t_0]
	\end{pmatrix}
	= 
	\mathcal{F}_{t_0 + 1}
	\begin{pmatrix}
		\sigma_h(\mathcal{R}(\bm{X})[1]) \\ \sigma_h(-\mathcal{R}(\bm{X})[1]) \\
		\sigma_h(\mathcal{R}(\bm{X})[2]) \\ \sigma_h(-\mathcal{R}(\bm{X})[2]) \\ \cdots \\
		\sigma_h(\mathcal{R}(\bm{X})[t_0 - 1]) \\ \sigma_h(-\mathcal{R}(\bm{X})[t_0 - 1]) \\ \mathcal{R}(\bm{X})[t_0]
	\end{pmatrix}
  = \begin{pmatrix}
  	 \mathcal{R}(\bm{X})[1] \\  \mathcal{R}(\bm{X})[2] \\ \vdots \\  \mathcal{R}(\bm{X})[t_0]
  \end{pmatrix},
\end{align*}
where we utilize the fact that $\sigma_h(\bm{x}) - \sigma_h(-\bm{x}) = \bm{x}$.

Further, for a tokenwise FNN layer $\mathcal{F}(\bm{X})[t] = \sigma(\bm{W} \bm{x}_t  + \bm{b})$, one can construct a FNN layer $\tilde{\mathcal{F}}$ with $\tilde{\bm{A}} = \bm{I}_{t_0} \otimes \bm{W}_f$ and $\tilde{\bm{b}} = \bbm{1}_{t_0} \otimes \bm{b}_f$ such that 
\begin{align*}
	\tilde{\mathcal{F}} 	\begin{pmatrix}
		\bm{x}[1] \\ \bm{x}[2] \\ \vdots \\ \bm{x}[t_0]
	\end{pmatrix} = 
	 \begin{pmatrix}
		\mathcal{F}(\bm{X})[1] \\  \mathcal{F}(\bm{X})[2] \\ \vdots \\  \mathcal{F}(\bm{X})[t_0]
	\end{pmatrix}.
\end{align*}

Then for any ParaRNN $\mathcal{N} = \mathcal{Q} \circ \mathcal{F} \circ \mathcal{R}_L \circ \cdots \circ \mathcal{R}_1 \circ \mathcal{P} \in \mathcal{PRNN}_{d_\mathrm{in}, d_\mathrm{out}}(d, d_s, L, U)$, we can find $\tilde{\mathcal{N}}_1, \dots, \tilde{\mathcal{N}}_L \in  \mathcal{FNN}_{t_0d, t_0d}((2t_0-1)d, t_0+1)$, a FNN layer $\tilde{\mathcal{F}}$ and linear mappings $\tilde{\mathcal{P}}$ and $\tilde{\mathcal{Q}}$ such that
\begin{align*}
	\tilde{\mathcal{Q}} \circ \tilde{\mathcal{F}} \circ \tilde{\mathcal{N}}_L \circ \cdots \circ (\tilde{\mathcal{N}}_1 \tilde{\mathcal{P}})
	(\vectorize (\bm{X}_{t_0})) =  \mathcal{Q} \circ \mathcal{F} \circ \mathcal{R}_L \circ \cdots \circ \mathcal{R}_1 \circ \mathcal{P}(\bm{X})[t_0] = \mathcal{N}(\bm{X})[t_0].
\end{align*}
It can be checked that $\tilde{\mathcal{Q}} \circ\tilde{ \mathcal{F}} \circ \tilde{\mathcal{N}}_L \circ \cdots \circ (\tilde{\mathcal{N}}_1 \tilde{\mathcal{P}}) \in \mathcal{FNN}_{d_\mathrm{in}t_0, d_\mathrm{out}}((2t_0-1)d, (t_0 + 1)L+2)$, and the proof is thus finished.
\end{proof}

\begin{lemma}\label{lemma:prnn-mprnn}
	For any modified ParaRNN $\bar{\mathcal{N}} \in \mathcal{MPRNN}_{d_\mathrm{in}, d_\mathrm{out}}(d, d_s, L, U)$, there exists a ParaRNN $\mathcal{N} \in \mathcal{PRNN}_{d_\mathrm{in}, d_\mathrm{out}}(d+d_s, d_s, 2L, U)$ such that
	\begin{align*}
		\bar{\mathcal{N}}(\bm{X})[t] = \mathcal{N}(\bm{X})[t] \hspace{3mm}\text{for } \bm{X} \in [0, 1]^{d_\mathrm{in} \times T}, t \in \{1, \dots, T\}.
	\end{align*}
\end{lemma}

\begin{proof}[Proof of Lemma \ref{lemma:prnn-mprnn}]
	This proof adapts from Lemma 7 of \cite{jiao2024approximation}.
	For any modified ParaRNN $\bar{\mathcal{N}} \in \mathcal{MPRNN}_{d_\mathrm{in}, d_\mathrm{out}}(d, d_s, L, U)$, it is constructed by $L$ modified recurrent layers $\bar{\mathcal{R}}_1, \dots, \bar{\mathcal{R}}_L$ with $\bar{\mathcal{R}}_l(\bm{X})[t] = \sigma_I(\bar{\bm{W}}_{h, l} \bar{\mathcal{R}}_l(\bm{X})[t-1] + \bar{\bm{W}}_{x, l} \bm{x}_t + \bar{\bm{b}}_l)$, a position-wise FC layer $\bar{\mathcal{F}}(\bm{X})[t] = \sigma_I (\bar{\bm{W}}_f\bm{x}_t + \bar{\bm{b}}_f)$, and tokenwise input and output projection  maps: $\bar{\mathcal{P}}: \mathbb{R}^{d_\mathrm{in} \times T} \mapsto \mathbb{R}^{d \times T}$ and $\bar{\mathcal{Q}}: \mathbb{R}^{d \times T} \mapsto \mathbb{R}^{d_\mathrm{out} \times T}.$ This can be equivalently written as 
	\begin{align*}
		\bar{\mathcal{N}} = \bar{\mathcal{Q}} \circ \bar{\mathcal{F}} \circ \bar{\mathcal{R}}_L \circ \cdots \circ \bar{\mathcal{R}}_1 \circ \bar{\mathcal{P}}.
	\end{align*}
	Without loss of generality, $d_\mathrm{in} = d_\mathrm{out} = d$ is assumed. 
	
	Let us first consider the case where $L=1$. We show that for any compact subset $\Omega \subset \mathbb{R}^{d\times T}$, there exist ParaRNN recurrent layers $\mathcal{R}_1, \mathcal{R}_2: \mathbb{R}^{(d+d_s) \times T} \mapsto \mathbb{R}^{(d+d_s) \times T}$ where the initial hidden states of $\mathcal{R}_2$ are non-zeros, a position-wise FC layer $\mathcal{F}: \mathbb{R}^{(d+d_s) \times T} \mapsto \mathbb{R}^{(d+d_s) \times T}$, and tokenwise projection maps: $\mathcal{P}: \mathbb{R}^{d \times T} \mapsto \mathbb{R}^{(d+d_s) \times T}$ and $\mathcal{Q} \in \mathbb{R}^{(d+d_s) \times T} \mapsto  \mathbb{R}^{d \times T}$ such that for any $\bm{X} \in \Omega$,
	$\bar{\mathcal{Q}} \circ \bar{\mathcal{F}} \circ \bar{\mathcal{R}}_1 \circ \bar{\mathcal{P}}(\bm{X}) = {\mathcal{Q}} \circ {\mathcal{F}} \circ {\mathcal{R}_2} \circ \mathcal{R}_1 \circ {\mathcal{P}}(\bm{X})$. 
	
	Without loss of generality, we may assume $\bar{\mathcal{P}}$ is an identity map and the index set $I = \{1, 2, \dots, k\}$. We define the input projection map $\mathcal{P}$ as 
	\begin{align*}
		\mathcal{P}(\bm{X})[t] = \begin{pmatrix}
			\bm{I}_d \\ \bm{0}_{d_s, d}
		\end{pmatrix} \bm{x}_t = \begin{pmatrix}
		\bm{x}_t \\ \bm{0}_{d_s}
	\end{pmatrix},
	\end{align*}
	which introduces $d_s$ additional dimensions.
	
	Further, for the ReLU activation $\sigma$, we have for a fixed point $z_0 > 0$, $\sigma(z_0 + \delta x) = z_0 + \delta x$ if $\delta$ is sufficiently small. Then for $\delta > 0$, we can construct 
	\begin{align*}
		\mathcal{R}_1^\delta(\bm{X})[t] = \sigma_h
		\left(\delta\begin{pmatrix}
			\bar{\bm{W}}_{x, 1} & \\ & \bm{0}_{d_s, d_s} 
		\end{pmatrix} \bm{X}[t] + \delta \begin{pmatrix}
		\bar{\bm{b}}_1 \\ \bm{0}_{d_s}
	\end{pmatrix}
	+ z_0 \begin{pmatrix}
		 \bm{1}_{d+1} \\ \bm{0}_{d_s - 1}
	\end{pmatrix}
	\right) \in \mathbb{R}^{d + d_s}.
	\end{align*}
	Since its recurrent matrix is $\bm{0}_{d+d_s, d+d_s}$, $\mathcal{R}_1^\delta$ can be considered as a ParaRNN recurrent layer. Combined with the input projection map, we have
	\begin{align*}
		\mathcal{R}_1^\delta \circ \mathcal{P}(\bm{X}) [t] = \sigma_h
		\begin{pmatrix}
			\delta 	\bar{\bm{W}}_{x, 1} \bm{x}_t + \delta \bar{\bm{b}}_1 + z_0 \bm{1}_d \\
			z_0 \\ \bm{0}_{d_s - 1}
		\end{pmatrix}
		= \begin{pmatrix}
			 z_0 \bm{1}_d + \delta(	\bar{\bm{W}}_{x, 1} \bm{x}_t  + \bar{\bm{b}}_1 ) \\ z_0 \\ \bm{0}_{d_s - 1}
		\end{pmatrix}.
	\end{align*}
	
	Next, we construct another recurrent layer
	\begin{align*}
		&\mathcal{R}_2^\delta(\bm{X})[t] \\ = & \sigma_h \left(
		\tilde{\bm{W}}_{h, 1} \mathcal{R}_2^\delta(\bm{X})[t-1]  + 
		\begin{pmatrix}
			\delta^{-1} \bm{I}_k & \\ & \bm{I}_{d+d_s-k} 
		\end{pmatrix} \bm{X}[t] +
		\begin{pmatrix}
			-\delta^{-1} z_0 \bm{1}_k  - \delta^{-1} z_0 [\bar{\bm{W}}_{h, 1}]_{1:k, k+1:d} \bm{1}_{d-k}  \\ 
			- z_0 [\bar{\bm{W}}_{h, 1}]_{k+1:d, k+1:d} \bm{1}_{d-k} \\
			\bm{0}_{d_s}
		\end{pmatrix}
		\right),
	\end{align*}
	and let the recurrent matrix 
	\begin{align*}
		\tilde{\bm{W}}_{h, 1} = 
		\begin{pmatrix}
			\bm{I}_{k} & \\ & \delta \bm{I}_{d+d_s - k} 
		\end{pmatrix}
		\begin{pmatrix}
			\bar{\bm{W}}_{h, 1} & \\ & \bm{0}_{d_s, d_s}
		\end{pmatrix}
		\begin{pmatrix}
			\bm{I}_k & & \\ & \delta^{-1}\bm{I}_{d-k} & \\ & & \bm{0}_{d_s, d_s}
		\end{pmatrix}
	\end{align*}
	and the initial hidden state $	\mathcal{R}_2^\delta(\bm{X})[0] = \begin{pmatrix}
	\bm{0}_k \\ z_0 \bm{1}_{d-k} \\ \bm{0}_{d_s}
	\end{pmatrix}$. Note that it can be easily verified that $\tilde{\bm{W}}_{h, 1}$ is a block diagonal matrix with each block of size $d_s \times d_s$.
	We next proceed to find out the form of $\mathcal{R}_2^\delta \circ \mathcal{R}_1^\delta \circ \mathcal{P}(\bm{X})[t]$ by mathematical induction.
	When $t=1$, we have
	\begin{align*}
		\mathcal{R}_2^\delta \circ \mathcal{R}_1^\delta \circ \mathcal{P}(\bm{X})[1] & = 
		\sigma_h \left(
		\tilde{\bm{W}}_{h, 1} 
		\begin{pmatrix}
			\bm{0}_k \\ z_0 \bm{1}_{d-k} \\ \bm{0}_{d_s}
		\end{pmatrix} \right.  +
		\begin{pmatrix}
			\delta^{-1} \bm{I}_k & \\ & \bm{I}_{d+d_s-k} 
		\end{pmatrix}
		 \begin{pmatrix}
			z_0 \bm{1}_d + \delta(	\bar{\bm{W}}_{x, 1} \bm{x}_1  + \bar{\bm{b}}_1 ) \\ z_0 \\ \bm{0}_{d_s - 1} 
		\end{pmatrix}  \\ &  \hspace{4cm} \left. +
		\begin{pmatrix}
			-\delta^{-1} z_0 \bm{1}_k  - \delta^{-1} z_0 [\bar{\bm{W}}_{h, 1}]_{1:k, k+1:d} \bm{1}_{d-k}  \\ 
			- z_0 [\bar{\bm{W}}_{h, 1}]_{k+1:d, k+1:d} \bm{1}_{d-k} \\
			\bm{0}_{d_s}
		\end{pmatrix}
		\right) \\
		& = \sigma_h 
		\begin{pmatrix}
			(\bar{\bm{W}}_{x, 1} \bm{x}_1 + \bar{\bm{b}}_{1})_{1:k} \\
			z_0 \bm{1}_{d-k} + \delta (\bar{\bm{W}}_{x, 1} \bm{x}_1 +  \bar{\bm{b}}_1)_{k+1:d} \\
			z_0 \\
			\bm{0}_{d_s-1} 
		\end{pmatrix} \\
	& = \begin{pmatrix}
	\bar{\mathcal{R}}_1 \circ \bar{ \mathcal{P}} (\bm{X})[1]_{1:k} \\
	z_0 \bm{1}_{d-k} + \delta 	\bar{\mathcal{R}}_1 \circ \bar{ \mathcal{P}} (\bm{X})[1]_{k+1:d} \\
	z_0 \\ \bm{0}_{d_s-1}
	\end{pmatrix}.
	\end{align*}
	Then assume 
	\begin{align*}
		\mathcal{R}_2^\delta \circ \mathcal{R}_1^\delta \circ \mathcal{P}(\bm{X})[t-1] =
		 \begin{pmatrix}
			\bar{\mathcal{R}}_1\circ \bar{ \mathcal{P}} (\bm{X})[t-1]_{1:k} \\
			z_0 \bm{1}_{d-k} + \delta 	\bar{\mathcal{R}}_1 \circ \bar{ \mathcal{P}} (\bm{X})[t-1]_{k+1:d} \\
			z_0 \\ \bm{0}_{d_s-1}
			\end{pmatrix}.
	\end{align*}
	We can derive the non-recurrent part of $\mathcal{R}_2^\delta \circ \mathcal{R}_1^\delta \circ \mathcal{P}(\bm{X})[t]$ as
	\begin{align*}
		& \begin{pmatrix}
			\delta^{-1} \bm{I}_k & \\ & \bm{I}_{d+d_s-k} 
		\end{pmatrix} 
		\begin{pmatrix}
			z_0 \bm{1}_d + \delta(	\bar{\bm{W}}_{x, 1} \bm{x}_t  + \bar{\bm{b}}_1 ) \\ z_0 \\ \bm{0}_{d_s - 1}
		\end{pmatrix} +
		\begin{pmatrix}
			-\delta^{-1} z_0 \bm{1}_k  - \delta^{-1} z_0 [\bar{\bm{W}}_{h, 1}]_{1:k, k+1:d} \bm{1}_{d-k}  \\ 
			- z_0 [\bar{\bm{W}}_{h, 1}]_{k+1:d, k+1:d} \bm{1}_{d-k} \\
			\bm{0}_{d_s}
		\end{pmatrix} \\
	= &  \begin{pmatrix}
		(\bar{\bm{W}}_{x, 1} \bm{x}_t + \bar{\bm{b}}_1)_{1:k} \\
		z_0 \bm{1}_{d-k} + \delta (\bar{\bm{W}}_{x, 1} \bm{x}_t +  \bar{\bm{b}}_1)_{k+1:d} \\
		z_0 \\
		\bm{0}_{d_s-1} 
	\end{pmatrix} + 
	\begin{pmatrix}
		 - \delta^{-1} z_0 [\bar{\bm{W}}_{h, 1}]_{1:k, k+1:d} \bm{1}_{d-k}  \\ 
		- z_0 [\bar{\bm{W}}_{h, 1}]_{k+1:d, k+1:d} \bm{1}_{d-k} \\
		\bm{0}_{d_s}
	\end{pmatrix}, 
	\end{align*}
	and thus
	\begin{align*}
			& \mathcal{R}_2^\delta \circ \mathcal{R}_1^\delta \circ \mathcal{P}(\bm{X})[t] \\ = & 
			\sigma_h \left(
			\tilde{\bm{W}}_{h, 1} \mathcal{R}_2^\delta \circ \mathcal{R}_1^\delta \circ \mathcal{P}(\bm{X})[t-1] +
			\begin{pmatrix}
				(\bar{\bm{W}}_{x, 1} \bm{x}_t + \bar{\bm{b}}_1)_{1:k} \\
				z_0 \bm{1}_{d-k} + \delta (\bar{\bm{W}}_{x, 1} \bm{x}_t +  \bar{\bm{b}}_1)_{k+1:d} \\
				z_0 \\
				\bm{0}_{d_s-1} 
			\end{pmatrix} + 
			\begin{pmatrix}
				- \delta^{-1} z_0 [\bar{\bm{W}}_{h, 1}]_{1:k, k+1:d} \bm{1}_{d-k}  \\ 
				- z_0 [\bar{\bm{W}}_{h, 1}]_{k+1:d, k+1:d} \bm{1}_{d-k} \\
				\bm{0}_{d_s}
			\end{pmatrix}
			\right) \\
			= & \sigma_h \left(
			\begin{pmatrix}
				[\bar{\bm{W}}_{h, 1}]_{1:k, :} \bar{\mathcal{R}}  \circ \bar{\mathcal{P}} (\bm{X}) [t-1]  + \delta^{-1} z_0 [\bar{\bm{W}}_{h, 1}]_{1:k, k+1:d} \bm{1}_{d-k}\\
				\delta [\bar{\bm{W}}_{h, 1}]_{k+1:d, :} \bar{\mathcal{R}} \circ \bar{\mathcal{P}} (\bm{X}) [t-1] + z_0 [\bar{\bm{W}}_{h, 1}]_{k+1:d, k+1:d} \bm{1}_{d-k} \\
				\bm{0}_{d_s}
			\end{pmatrix} + 
		\begin{pmatrix}
			(\bar{\bm{W}}_{x, 1} \bm{x}_1 + \bar{\bm{b}}_1)_{1:k} \\
			z_0 \bm{1}_{d-k} + \delta (\bar{\bm{W}}_{x, 1} \bm{x}_1 +  \bar{\bm{b}}_1)_{k+1:d} \\
			z_0 \\
			\bm{0}_{d_s-1} 
		\end{pmatrix}  \right.\\  & \hspace{4cm}+ \left.
		\begin{pmatrix}
			- \delta^{-1} z_0 [\bar{\bm{W}}_{h, 1}]_{1:k, k+1:d} \bm{1}_{d-k}  \\ 
			- z_0 [\bar{\bm{W}}_{h, 1}]_{k+1:d, k+1:d} \bm{1}_{d-k} \\
			\bm{0}_{d_s}
		\end{pmatrix}
			\right) \\
		= & 
		\begin{pmatrix}
			\bar{\mathcal{R}}_1 \circ \bar{\mathcal{P}} (\bm{X})[t]_{1:k} \\
			z_0 \bm{1}_{d-k} + \delta \bar{\mathcal{R}}_1 \circ \bar{\mathcal{P}} (\bm{X})[t]_{k+1:d} \\
			z_0 \\ \bm{0}_{d_s - 1}
		\end{pmatrix}.
	\end{align*}
	
	Further, we can construct the position-wise FC layer as
	\begin{align}
		&\mathcal{F}(\bm{X})[t] \notag \\= &\sigma \left(
		\begin{pmatrix}
			\bm{I}_k & & \\
			& \delta \bm{I}_{d-k} & \\
			& & \bm{I}_{d_s}
		\end{pmatrix}
		\begin{pmatrix}
			\bar{\bm{W}}_f & \\
			& \bm{0}_{d_s, d_s}
		\end{pmatrix}
		\begin{pmatrix}
			\bm{I}_k & & & \\
			& \delta^{-1} \bm{I}_{d-k} & -\delta^{-1} \bm{1}_{d-k} &  \\
			& & & \bm{0}_{d_s, d_s-1}
		\end{pmatrix}
		\bm{X}[t]  \right.  \notag \\ 
		&\hspace{8cm} + \left.
		\begin{pmatrix}
			[\bar{\bm{b}}_f ]_{1:k} \\ \delta 	[\bar{\bm{b}}_f ]_{k+1:d} + z_0 \bm{1}_{d-k}  \\ z_0 \\  \bm{0}_{d_s-1}
		\end{pmatrix}
		\right),		
		\label{eq:fnn-construction}
	\end{align}
	which leads to
	\begin{align*}
		\mathcal{F} \circ \mathcal{R}_2^\delta \circ \mathcal{R}_1^\delta \circ  \mathcal{P}(\bm{X})[t] = \begin{pmatrix}
			\bar{\mathcal{F}}\circ \bar{\mathcal{R}}_1 \circ \bar{\mathcal{P}} (\bm{X})[t]_{1:k} \\
			\delta \bar{\mathcal{F}}\circ \bar{\mathcal{R}}_1 \circ \bar{\mathcal{P}} (\bm{X})[t]_{k+1:d} + z_0 \bm{1}_{d-k}  \\
			z_0 \\ \bm{0}_{d_s-1}
		\end{pmatrix}.
	\end{align*}	
	Finally, let the output projection matrix 
	\begin{align}
		\bm{Q} = \begin{pmatrix}
			\bar{\bm{Q}} &  \bm{0}_{d, d_s}
		\end{pmatrix}
		\begin{pmatrix}
			\bm{I}_k & & & \\
			& \delta^{-1} \bm{I}_{d-k} & -\delta^{-1} \bm{1}_{d-k} &  \\
			& & & \bm{0}_{d_s, d_s-1}
		\end{pmatrix}. \label{eq:mrnn-rnn-Q}
	\end{align}
	One can see that 
	\begin{align*}
		\mathcal{Q}\circ\mathcal{F} \circ\mathcal{R}_2^\delta \circ \mathcal{R}_1^\delta\circ \mathcal{P}(\bm{X})[t]
		= 	\bar{\mathcal{Q}}\circ\bar{\mathcal{F}} \circ\bar{ \mathcal{R}}_1  \circ \bar{\mathcal{P}}(\bm{X})[t]
		\hspace{3mm} \text{for } 1\leq t \leq T.
	\end{align*} 
	As $z_0$ is arbitrary, we set $z_0 = \max\{ \max_{1\leq t \leq T, 1 \leq k \leq d} \sup_{\bm{X} \in \Omega} \lvert \bar{\mathcal{R}}_1 \circ \bar{\mathcal{P}} (\bm{X}) [t]_k \rvert, \max_{1\leq t \leq T, 1 \leq k \leq d}$ $ \sup_{\bm{X} \in \Omega} \lvert \bar{\mathcal{F}} \circ \bar{\mathcal{R}}_1 \circ \bar{\mathcal{P}} (\bm{X}) [t]_k \rvert\}$ and $\delta \leq 1$ to ensure that $\sigma(z_0 + \delta x) =z_0 + \delta x$ holds.
	
	We next consider the case that $L=2$. Consider a ParaRNN recurrent layer of the form 
	\begin{align*}
		\mathcal{R}_3^\delta(\bm{X})[t] = \sigma_h
		\left(\delta\begin{pmatrix}
			\bar{\bm{W}}_{x, 2} & \\ & \bm{0}_{d_s, d_s} 
		\end{pmatrix} 
		\begin{pmatrix}
			\bm{I}_k & & & \\
			& \delta^{-1} \bm{I}_{d-k} & -\delta^{-1} \bm{1}_{d-k} &  \\
			& & & \bm{0}_{d_s, d_s-1}
		\end{pmatrix}
		\bm{X}[t] \right. \\ 
		\left.+ \delta \begin{pmatrix}
			\bar{\bm{b}}_2 \\ \bm{0}_{d_s}
		\end{pmatrix}
		+ z_0 \begin{pmatrix}
			\bm{1}_{d+1} \\ \bm{0}_{d_s - 1}
		\end{pmatrix}\right),
	\end{align*}
	and thus
	\begin{align*}
		\mathcal{R}_3^\delta  \circ\mathcal{R}_2^\delta \circ \mathcal{R}_1^\delta\circ \mathcal{P}(\bm{X})[t]
		=  \begin{pmatrix}
			z_0 \bm{1}_d + \delta(\bar{\bm{W}}_{x, 2} \bar{\mathcal{R}}_1 \circ \bar{\mathcal{P}}(\bm{X})[t] + \bar{\bm{b}}_2) \\
			z_0 \\ \bm{0}_{d_s-1}
		\end{pmatrix}.
	\end{align*}
	We further construct another ParaRNN recurrent layer
	\begin{align*}
		&\mathcal{R}_4^\delta(\bm{X})[t] \\ = & \sigma_h \left(
		\tilde{\bm{W}}_{h, 2} \mathcal{R}_4^\delta(\bm{X})[t-1]  + 
		\begin{pmatrix}
			\delta^{-1} \bm{I}_k & \\ & \bm{I}_{d+d_s-k} 
		\end{pmatrix} \bm{X}[t] +
		\begin{pmatrix}
			-\delta^{-1} z_0 \bm{1}_k  - \delta^{-1} z_0 [\bar{\bm{W}}_{h, 2}]_{1:k, k+1:d} \bm{1}_{d-k}  \\ 
			- z_0 [\bar{\bm{W}}_{h, 2}]_{k+1:d, k+1:d} \bm{1}_{d-k} \\
			\bm{0}_{d_s}
		\end{pmatrix}
		\right),
	\end{align*}
	and let the recurrent matrix 
	\begin{align*}
		\tilde{\bm{W}}_{h, 2} = 
		\begin{pmatrix}
			\bm{I}_{k} & \\ & \delta \bm{I}_{d+d_s - k} 
		\end{pmatrix}
		\begin{pmatrix}
			\bar{\bm{W}}_{h, 2} & \\ & \bm{0}_{d_s, d_s}
		\end{pmatrix}
		\begin{pmatrix}
			\bm{I}_k & & \\ & \delta^{-1}\bm{I}_{d-k} & \\ & & \bm{0}_{d_s, d_s}
		\end{pmatrix}
	\end{align*}
	and the initial hidden state $	\mathcal{R}_4^\delta(\bm{X})[0] = \begin{pmatrix}
		\bm{0}_k \\ z_0 \bm{1}_{d-k} \\ \bm{0}_{d_s}
	\end{pmatrix}$. 
	This gives rise to 
	\begin{align*}
		\mathcal{R}_4^\delta \circ \mathcal{R}_3^\delta  \circ\mathcal{R}_2^\delta \circ \mathcal{R}_1^\delta\circ \mathcal{P}(\bm{X})[t] 
		= \begin{pmatrix}
			\bar{\mathcal{R}}_2 \circ \bar{\mathcal{R}}_1 \circ \bar{\mathcal{P}}(\bm{X})[t]_{1:k} \\
			z_0 \bm{1}_{d-k} + \delta (\bar{\mathcal{R}}_2 \circ \bar{\mathcal{R}}_1 \circ \bar{\mathcal{P}}(\bm{X})[t]_{k+1:d}) \\
			z_0 \\ \bm{0}_{d_s -1}
		\end{pmatrix}
	\end{align*}
	Then applying the same FC layer as in \eqref{eq:fnn-construction} and the same output projection matrix as in \eqref{eq:mrnn-rnn-Q}, we have 
	\begin{align*}
		\mathcal{Q}\circ\mathcal{F} \circ \mathcal{R}_4^\delta \circ \mathcal{R}_3^\delta  \circ\mathcal{R}_2^\delta \circ \mathcal{R}_1^\delta\circ \mathcal{P}(\bm{X})[t]
		= 	\bar{\mathcal{Q}}\circ\bar{\mathcal{F}} \circ \bar{ \mathcal{R}}_2 \circ\bar{ \mathcal{R}}_1  \circ \bar{\mathcal{P}}(\bm{X})[t]
		\hspace{3mm} \text{for } 1\leq t \leq T.
	\end{align*} 

	For a general $L$ and any modified ParaRNN $\bar{\mathcal{N}} \in \mathcal{PRNN}_{d_\mathrm{in}, d_\mathrm{out}}(d, d_s, L, U)$, we can find $2L$ ParaRNN recurrent layers $\mathcal{R}_1, \dots, \mathcal{R}_{2L}$, a position-wise FC layer $\mathcal{F}$, and projection maps $\mathcal{P}$ and $\mathcal{Q}$ such that 
	\begin{align*}
		\bar{\mathcal{N}}(\bm{X}) &= \bar{\mathcal{Q}} \circ \bar{\mathcal{F}} \circ \bar{\mathcal{R}}_L \circ \cdots \circ \bar{\mathcal{R}}_1 \circ \bar{\mathcal{P}}(\bm{X}) \\
		 &=  \mathcal{Q} \circ \mathcal{F} \circ  (\mathcal{R}_{2L} \circ \mathcal{R}_{2L-1}) \circ \cdots \circ (\mathcal{R}_{2} \circ \mathcal{R}_{1}) \circ {\mathcal{P}} (\bm{X}):= \mathcal{N}(\bm{X})
	\end{align*}
	for $\bm{X} \in [0, 1]^{d_\mathrm{in} \times T}$ based on previous discussion. It can be seen that  $\mathcal{N}(\bm{X}) \in \mathcal{PRNN}_{d_\mathrm{in}, d_\mathrm{out}} (d+d_s, d_s, 2L, U)$, and the proof is thus finished.
\end{proof}

	\begin{lemma}
		\label{lemma:mprnn-fnn}
		Let $t_0 \in \{1, \dots, T\}$.
		For any FNN $\tilde{\mathcal{N}} \in \mathcal{FNN}_{d_\mathrm{in}t_0, d_\mathrm{out}}(d, L)$, there exists a modified ParaRNN $\bar{\mathcal{N}} \in \mathcal{MPRNN}_{d_\mathrm{in}, d_\mathrm{out}}((d_\mathrm{in} + 1)d, {d_s}, T+L-1, U)$ with $d_s$ being any factor of $(d_\mathrm{in} + 1)d$ such that
		\begin{align*}
			\tilde{\mathcal{N}}(\bm{X}_{t_0}) = \bar{\mathcal{N}}(\bm{X})[t_0],
		\end{align*}
		where $\bm{X} = (\bm{x}_1, \bm{x}_2, \cdots, \bm{x}_T) \in \mathbb{R}^{d_\mathrm{in}\times T}$.
	\end{lemma}
	
	\begin{proof}[Proof of Lemma \ref{lemma:mprnn-fnn}]
		This proof adapts from Lemma 8 of \cite{jiao2024approximation}, and it contains three steps.
		\paragraph{Step 1.} Suppose $\bm{A}[T-t_0+1], \bm{A}[T-t_0+2], \dots,  \bm{A}[T] \in \mathbb{R}^{1 \times d_{\mathrm{in}}}$ are given matrices. There exists $\mathcal{N} = \mathcal{R}_T \circ \cdots \circ \mathcal{R}_1 \circ \mathcal{P}: \mathbb{R}^{d_\mathrm{in} \times T} \mapsto \mathbb{R}^{(d_\mathrm{in} + 1) \times T}$, where $\mathcal{R}_l$'s are modified ParaRNN recurrent layers of width $d_\mathrm{in} + 1$, such that
		\begin{align*}
			\mathcal{N}(\bm{X})[t] = \begin{pmatrix}
				\bm{x}_t \\ \star
			\end{pmatrix} \hspace{2mm} \text{for } t \neq t_0, \hspace{3mm}
			\mathcal{N}(\bm{X})[t_0] = \begin{pmatrix}
				\bm{x}_{t_0} \\ \sum_{j=1}^{t_0} \bm{A}[T-t_0 + j] \bm{x}_j
			\end{pmatrix}, 
		\end{align*}
		where $\star$ denotes a quantity that we do not care for the proof.
		\begin{proof}[Proof of Step 1]
			Following Lemma 8 of \cite{jiao2024approximation}, we construct $\mathcal{N}_n = \mathcal{R}_n \circ \cdots \circ \mathcal{R}_1 \circ \mathcal{P}$ where the $l$-th modified ParaRNN recurrent layer is defined as
			\begin{align}
				\mathcal{R}_l (\bm{X})[t] = \begin{pmatrix}
					\bm{0}_{d_\mathrm{in}, d_\mathrm{in}}  & \bm{0}_{d_\mathrm{in}, 1} \\
					 \bm{0}_{1, d_\mathrm{in}} & 1
				\end{pmatrix}
				\mathcal{R}_l (\bm{X})[t-1] + \begin{pmatrix}
					\bm{I}_{d_\mathrm{in}} &  \bm{0}_{d_\mathrm{in}, 1} \\
					\bm{b}_l & 1
				\end{pmatrix} \bm{x}[t] \label{eq:fnn-modified-rnn-layer}
			\end{align}
		for some $\bm{b}_l \in \mathbb{R}^{1 \times d_\mathrm{in}}$ which we will determine soon; and we construct the input projection map $\mathcal{P}$ as 
		\begin{align*}
			\mathcal{P}(\bm{X})[t] = \begin{pmatrix}
				\bm{I}_{d_\mathrm{in}} \\ \bm{0}_{1, d_\mathrm{in}} 
			\end{pmatrix} \bm{x}_t = \begin{pmatrix}
			\bm{x}_t \\ 0
		\end{pmatrix}.
		\end{align*}
		By Lemma 24 of \cite{song2023minimal}, we have 
		$\mathcal{N}_n(\bm{X})[m] = \begin{pmatrix}
			\bm{x}_m \\ \sum_{i=1}^{\infty} \sum_{j=1}^\infty 
				{n + m - i - j \choose n-i}
			\bm{b}_i \bm{x}_j
		\end{pmatrix}$, where $n \choose k$ denotes the binomial coefficient for $n \geq k$ and it equals to zero whenever $k > n$ or $n < 0$. Therefore, $\mathcal{N}(\bm{X})[t] = \begin{pmatrix}
		\bm{x}_t \\ \star
		\end{pmatrix}$ for $t \neq t_0$ is proved, and it remains to discuss the case where $n=T, m=t_0$.
		Define a matrix $\bm{\Lambda}_T = \left\{ {2T - i - j \choose T-i } \right\}_{1 \leq i, j \leq T}$, and by Lemma 15 of \cite{jiao2024approximation}, it has an inverse $\bm{\Lambda}_T^{-1} = \left\{\lambda_{i, j}\right\}_{1 \leq i, j \leq T}$. Let $\bm{b}_i = \sum_{k=1}^T \lambda_{k, i} \bm{A}[k]$. Then we have
		\begin{align*}
			\sum_{i=1}^{T} \sum_{j=1}^{t_0} 
			{T + t_0 - i - j \choose T-i}
			\bm{b}_i \bm{x}_j
			=& \sum_{i=1}^{T} \sum_{j=1}^{t_0}  \sum_{k=1}^T 
			{T + t_0 - i - j \choose T-i}
			\lambda_{k, i} \bm{A}[k] \bm{x}_j \\
			= &  \sum_{j=1}^{t_0}  \sum_{k=1}^T \left\{\sum_{i=1}^T
			{T + t_0 - i - j \choose T-i}
			\lambda_{k, i}\right\} \bm{A}[k] \bm{x}_j \\
			= &  \sum_{j=1}^{t_0}  \sum_{k=1}^T \delta_{k, T-t_0+j} \bm{A}[k] \bm{x}_j \\ 
			= & \sum_{j=1}^{t_0} \bm{A}[T-t_0+j] \bm{x}_j 
		\end{align*}
		with $\delta$ being the Kronecker delta function. Thus
		\begin{align*}
			\mathcal{N}(\bm{X})[t_0] = \begin{pmatrix}
				\bm{x}_{t_0}\\ \sum_{i=1}^{T} \sum_{j=1}^{t_0} 
				{T + t_0 - i - j \choose T-i}
				\bm{b}_i \bm{x}_j
			\end{pmatrix} = 
			\begin{pmatrix}
				\bm{x}_{t_0}\\ \sum_{j=1}^{t_0} \bm{A}[T-t_0+j] \bm{x}_j 
			\end{pmatrix}.
		\end{align*}
		\end{proof}
		
		\paragraph{Step 2.} Suppose $\bm{A}_k[T-t_0+1], \bm{A}_k[T-t_0+2], \dots,  \bm{A}_k[T] \in \mathbb{R}^{1 \times d_{\mathrm{in}}}$ with $1\leq k \leq d$ are given matrices and $\bm{c} \in \mathbb{R}^d$ is the given bias vector. Then there exists $\mathcal{N}_{T+1} = \mathcal{R}_{T+1} \circ  \mathcal{R}_T \circ \cdots \circ \mathcal{R}_1 \circ \mathcal{P}: \mathbb{R}^{d_\mathrm{in} \times T} \mapsto \mathbb{R}^{(d_\mathrm{in} + 1)d \times T}$, where $\mathcal{R}_l$'s are modified ParaRNN recurrent layers of width $(d_\mathrm{in} + 1)d$, such that 
		\begin{align*}
			\mathcal{N}_{T+1}(\bm{X})[t_0] =
			\begin{pmatrix}
				\sigma\left( \sum_{j=1}^{t_0} \bm{A}_1[T-t_0+j] \bm{x}_j + c_1  \right) \\ 
				\vdots \\
				\sigma\left( \sum_{j=1}^{t_0} \bm{A}_d[T-t_0+j] \bm{x}_j + c_d  \right) \\ \bm{0}_{dd_\mathrm{in}, 1}
			\end{pmatrix}.
		\end{align*}
		\begin{proof}
			By concatenating $d$ architectures in \eqref{eq:fnn-modified-rnn-layer}, i.e., let the weight matrices
			\begin{align*}
				\bm{W}_{h, l} = \oplus_{k=1}^d \begin{pmatrix}
					\bm{0}_{d_\mathrm{in}, d_\mathrm{in}}  & \bm{0}_{d_\mathrm{in}, 1} \\
					\bm{0}_{1, d_\mathrm{in}} & 1
				\end{pmatrix},
				\bm{W}_{x, l} = \oplus_{k=1}^d  \begin{pmatrix}
					\bm{I}_{d_\mathrm{in}} &  \bm{0}_{d_\mathrm{in}, 1} \\
					\bm{b}_l^{(k)} & 1
				\end{pmatrix}, 
				\bm{P} = \begin{pmatrix}
						\bm{I}_{d_\mathrm{in}} \\ \bm{0}_{1, d_\mathrm{in}}  \\ \vdots \\ \bm{I}_{d_\mathrm{in}} \\ \bm{0}_{1, d_\mathrm{in}}
				\end{pmatrix},
			\end{align*}
		where $\bm{b}_l^{(k)}  = \sum_{s=1}^T \lambda_{s, l} \bm{A}_k[s]$ and $\bm{P} \in \mathbb{R}^{(d_\mathrm{in} + 1)d \times d_\mathrm{in} }$,
		we can construct $\mathcal{N}_T = \mathcal{R}_T \circ \cdots \circ \mathcal{R}_1 \circ \mathcal{P} : \mathbb{R}^{d_\mathrm{in} \times T} \mapsto \mathbb{R}^{(d_\mathrm{in} + 1)d \times T}$ with $\mathcal{R}_l (\bm{X})[t] = \bm{W}_{h, l}
		\mathcal{R}_l (\bm{X})[t-1] + \bm{W}_{x, l} \bm{x}[t]$ and $\mathcal{P}$ having its corresponding projection matrix as the above. Note that all $\bm{W}_{h, l}$'s are block diagonal. Then by Step 1, we have
		\begin{align*}
			\mathcal{N}_T(\bm{X})[t_0] = \begin{pmatrix}
				\bm{x}_{t_0}\\ \sum_{j=1}^{t_0} \bm{A}_1[T-t_0+j] \bm{x}_j  \\ \vdots \\ \bm{x}_{t_0}\\ \sum_{j=1}^{t_0} \bm{A}_d[T-t_0+j] \bm{x}_j 
			\end{pmatrix}
			\in \mathbb{R}^{(d_\mathrm{in} + 1)d}.
		\end{align*}
		We further define a modified ParaRNN recurrent layer $\mathcal{R}_{T+1}$ by $\mathcal{R}_{T+1}(\bm{X})[t] = \sigma(\bm{W}_{x, T+1}\bm{x}[t] +\bm{b}_{T+1})$, where $\bm{b}_{T+1} = \begin{pmatrix}
			\bm{c} \\ \bm{0}_{d_\mathrm{in}d, 1}
		\end{pmatrix}$ and $[\bm{W}_{x, T+1}]_{i, j} = \begin{cases}
		1,  & j = (d_\mathrm{in} + 1)i, 1\leq i \leq d \\
		0, & \text{otherwise}.
	\end{cases}$.
	  It can be easily verified that 
	  \begin{align*}
	  	\mathcal{N}_{T+1}(\bm{X})[t_0] = \mathcal{R}_{T+1} \circ \mathcal{N}_{T}(\bm{X})[t_0] =
	  	\begin{pmatrix}
	  		\sigma \left(\sum_{j=1}^{t_0} \bm{A}_1[T-t_0+j] \bm{x}_j + c_1 \right) \\ \vdots \\	\sigma \left(\sum_{j=1}^{t_0} \bm{A}_d[T-t_0+j] \bm{x}_j  +c_d \right) \\   \bm{0}_{dd_\mathrm{in}, 1}.
	  	\end{pmatrix}
	  \end{align*}
	\end{proof}

	\paragraph{Step 3.} For any FNN $\tilde{\mathcal{N}} \in \mathcal{FNN}_{d_\mathrm{in}t_0, d_\mathrm{out}}(d, L)$, there exists a modified ParaRNN $\bar{\mathcal{N}} \in \mathcal{MPRNN}_{d_\mathrm{in}, d_\mathrm{out}}((d_\mathrm{in} + 1) d, {d_s}, T+L-1, U)$ such that
	\begin{align*}
		\tilde{\mathcal{N}}(\bm{X}_{t_0}) = \bar{\mathcal{N}}(\bm{X})[t_0],
	\end{align*}
	where $\bm{X} = (\bm{x}_1, \bm{x}_2, \cdots, \bm{x}_T) \in \mathbb{R}^{d_\mathrm{in}\times T}$.
	\begin{proof}
		Without loss of generality, assume $L \geq 3$. Denote the weight matrices and bias vectors for $l$-th layer of $\tilde{\mathcal{N}}$ as $\tilde{\bm{A}}_l$ and $\tilde{\bm{c}}_l$, where $\tilde{\bm{A}}_1 \in \mathbb{R}^{d \times d_\mathrm{in} t_0}$, $\tilde{\bm{A}}_l \in \mathbb{R}^{d\times d}$ for $2\leq l \leq L-1$, $\tilde{\bm{A}}_L \in \mathbb{R}^{d_\mathrm{out} \times d}$, $\tilde{\bm{c}}_l \in \mathbb{R}^d$ for $1 \leq l \leq L-1$ and $\tilde{\bm{c}}_L \in \mathbb{R}^{d_\mathrm{out}}$.
		
		By Step 2, there exists a network $\mathcal{N}_{T+1}= \mathcal{R}_{T+1} \circ  \mathcal{R}_T \circ \cdots \circ \mathcal{R}_1 \circ \mathcal{P}: \mathbb{R}^{d_\mathrm{in} \times T} \mapsto \mathbb{R}^{(d_\mathrm{in} + 1)d \times T}$, where $\mathcal{R}_l$'s are modified ParaRNN recurrent layers of width $(d_\mathrm{in} + 1)d$ such that
		\begin{align*}
			\mathcal{N}_{T+1}(\bm{X})[t_0] = \begin{pmatrix}
				\sigma\left( \tilde{\bm{A}}_1 \vectorize(\bm{x}[{1:t_0}]) + \tilde{\bm{c}}_1 \right) \\
				\bm{0}_{dd_\mathrm{in}, 1}
			\end{pmatrix},
		\end{align*}
		as we may let $\bm{A}_k[T-t_0+j]$ equal to $[\tilde{\bm{A}}_1 ]_{k, (j-1)d_\mathrm{in} + 1 : jd_\mathrm{in}}$ in Step 2.
	\end{proof}
	
	 The 2nd to $L-1$-th of $\tilde{\mathcal{N}}$ can be achieved by the token-wise FNN, which is a special type of modified ParaRNN. Define the modified ParaRNN layers $ \mathcal{R}_{T+l} (\bm{X})[t] = \sigma(\bar{\bm{A}}_l \bm{x}[t] + \bar{\bm{b}}_l )   \text{ for } 2\leq l \leq L-1 $, and the modified position-wise FC layer as $\mathcal{F}(\bm{X})[t] = \bar{\bm{A}}_L \bm{x}[t] + \bar{\bm{b}}_L$, where
	 $\bar{\bm{A}}_l = \begin{pmatrix}
	 	\tilde{\bm{A}}_l & \bm{0}_{d, d_\mathrm{in}d} \\
	 	 \bm{0}_{d_\mathrm{in}d, d} & \bm{0}_{d_\mathrm{in}d, d_\mathrm{in}d}
	 \end{pmatrix}$, $\bar{\bm{b}}_l = \begin{pmatrix}
	 \tilde{\bm{c}}_l \\ \bm{0}_{d_\mathrm{in}d, 1}
	 \end{pmatrix}$ for $l=2, \dots, L-1$,  and
	 	$\bar{\bm{A}}_L = \begin{pmatrix}
	 		\tilde{\bm{A}}_L & \bm{0}_{d_\mathrm{out}, d_\mathrm{in}d} \\
	 		\bm{0}_{(d_\mathrm{in}+1)d - d_\mathrm{out}, d} & \bm{0}_{(d_\mathrm{in}+1)d - d_\mathrm{out}, d_\mathrm{in}d}
	 	\end{pmatrix}$,
	 $\bar{\bm{b}}_L = \begin{pmatrix}
	 		\tilde{\bm{c}}_L \\ \bm{0}_{(d_\mathrm{in} + 1)d - d_\mathrm{out}, 1}
	 \end{pmatrix}$.
	Further let the output projection matrix $\bm{Q} = \begin{pmatrix}
		\bm{I}_{d_\mathrm{out}} & \bm{0}_{d_\mathrm{out}, (d_\mathrm{in} + 1)d - d_\mathrm{out}}
	\end{pmatrix}$. Then it can be verified that $	\mathcal{N}_{T+L-1} = \mathcal{Q} \circ \mathcal{F} \circ \mathcal{R}_{T+L-1} \circ \cdots \circ \mathcal{R}_1 \circ \mathcal{P} \in \mathcal{MPRNN}_{d_\mathrm{in}, d_\mathrm{out}}((d_\mathrm{in} + 1) d, {d_s}, T+L-1, U)$ satisfies
	\begin{align*}
		\mathcal{N}_{T+L-1}(\bm{X})[t_0] = \tilde{\mathcal{N}}(\bm{x}[{1:t_0}]),
	\end{align*}
 	which finishes the proof.
\end{proof}

\subsubsection{Proof of Theorem \ref{proposition:approx-rate-single-step}} \label{subsec:approx-proof}
\begin{proof}
	This proof adapts from Lemma 10 of \cite{jiao2024approximation}.
	Note that to prove Theorem \ref{proposition:approx-rate-single-step}, it suffices to show that for any $I, J \in \mathbb{N}^+$, there exists a ParaRNN $\mathcal{N}  \in \mathcal{PRNN}^{(T)}_{d_\mathrm{in}, d_\mathrm{out}}(d, d_s,$ $ L, U)$ with width $d = 
	76 (\lfloor \beta \rfloor + 1)^2 3^{d_\mathrm{in}T} d_\mathrm{in}^{\lfloor \beta \rfloor + 2} d_\mathrm{out}$ $ T^{\lfloor \beta \rfloor + 1} J \lceil \log_2(8J)\rceil + d_s$ and depth $L= 42( \lfloor \beta \rfloor + 1)^2 I \lceil \log_2(8I) \rceil + 6d_\mathrm{in} T$
	such that 
	\begin{align*}
		\sup_{\bm{X}_{T} \in [0, 1]^{d_\mathrm{in} \times T}}
	\norm{\mathcal{N}(\bm{X}_{T})[T]  - f^{(T)}(\bm{X}_{T})}_\infty \leq 19 U (\lfloor \beta \rfloor + 1)^2 (d_\mathrm{in}T)^{\lfloor \beta \rfloor + (\beta  \vee 1) / 2} (JI)^{-2\beta / (d_\mathrm{in}T)}.
	\end{align*}
	
	Note that $f^{(T)} = (f_1^{(T)}, \dots, f_i^{(T)}, \dots, f_{d_\mathrm{out}}^{(T)})^\top$. By Corollary 3.1 of \cite{jiao2023deep}, for any $i \in \{1, \dots, d_\mathrm{out}\}$ and $I, J \in \mathbb{N}^+$, there exists a feedforward neural network $\tilde{\mathcal{N}}_i \in \mathcal{FNN}_{d_\mathrm{in}T, 1}(d, L)$ with width $d = 38(\lfloor \beta \rfloor + 1)^2 3^{d_\mathrm{in}T} (d_\mathrm{in}T)^{\lfloor \beta \rfloor + 1} J \lceil \log_2(8J)\rceil$ and depth $L = 21( \lfloor \beta \rfloor + 1)^2 I \lceil \log_2(8I) \rceil + 2d_\mathrm{in} T$ such that 
	\begin{align*}
		\sup_{\bm{X}_{T} \in [0, 1]^{d_\mathrm{in} \times T}}
		\lvert{\tilde{\mathcal{N}}_i(\bm{X}_{T})- f_i^{(T)}(\bm{X}_{T})}\rvert \leq
		19 U (\lfloor \beta \rfloor + 1)^2 (d_\mathrm{in}T)^{\lfloor \beta \rfloor + (\beta  \vee 1) / 2} (JI)^{-2\beta / (d_\mathrm{in}T)}.
	\end{align*}
	By concatenating $\tilde{\mathcal{N}}_1, \dots, \tilde{\mathcal{N}}_{d_\mathrm{out}}$, we can say there exists a FNN $\tilde{\mathcal{N}} = (\tilde{\mathcal{N}}_1, \dots, \tilde{\mathcal{N}}_{d_\mathrm{out}})^\top \in \mathcal{FNN}_{d_\mathrm{in}T, d_\mathrm{out}}(d^\prime, L)$ with width $d^\prime = d_\mathrm{out} d$ such that
	\begin{align*}
		\sup_{\bm{X}_{T} \in [0, 1]^{d_\mathrm{in} \times T}}
		\norm{\tilde{\mathcal{N}}(\bm{X}_{T})- f^{(T)}(\bm{X}_{T})}_\infty \leq
		19 U (\lfloor \beta \rfloor + 1)^2 (d_\mathrm{in}T)^{\lfloor \beta \rfloor + (\beta  \vee 1) / 2} (JI)^{-2\beta / (d_\mathrm{in}T)}.
	\end{align*}

	By Proposition \ref{prop:equivalence-pararnn-fnn}, there exists a ParaRNN $\mathcal{N} \in \mathcal{PRNN}_{d_\mathrm{in}, d_\mathrm{out}}(d^{\prime\prime}, d_s, L^{\prime\prime}, U)$ with width
	\begin{align*}
		d^{\prime\prime} =& (d_\mathrm{in} + 1)d^\prime + d_s =  (d_\mathrm{in} + 1)d_\mathrm{out}d + d_s \\ 
	    =& 38 (d_\mathrm{in} + 1)d_\mathrm{out} (\lfloor \beta \rfloor + 1)^2 3^{d_\mathrm{in}T} (d_\mathrm{in}T)^{\lfloor \beta \rfloor + 1} J \lceil \log_2(8J)\rceil + d_s \\
	   \leq &  {
	    	76 (\lfloor \beta \rfloor + 1)^2 3^{d_\mathrm{in}T} d_\mathrm{in}^{\lfloor \beta \rfloor + 2} d_\mathrm{out} T^{\lfloor \beta \rfloor + 1} J \lceil \log_2(8J)\rceil + d_s
	   }	 
	\end{align*}
	and depth 
	\begin{align*}
		L^{\prime\prime} = 2T + 2L - 2 
		=& 42( \lfloor \beta \rfloor + 1)^2 I \lceil \log_2(8I) \rceil + 4d_\mathrm{in}T + 2T - 2 \\
		\leq & {42( \lfloor \beta \rfloor + 1)^2 I \lceil \log_2(8I) \rceil + 6d_\mathrm{in} T}
	\end{align*}
such that $\mathcal{N}(\bm{X})[T] = \tilde{\mathcal{N}}(\bm{X})$ for $\bm{X} \in [0, 1]^{d_\mathrm{in} \times T}$.
Note that for any ParaRNN $\mathcal{N} \in \mathcal{PRNN}_{d_\mathrm{in}, d_\mathrm{out}}(d^{\prime\prime}, d_s, L^{\prime\prime}, U)$,  we can always find a $\phi \in  \mathcal{PRNN}^{(T)}_{d_\mathrm{in}, d_\mathrm{out}}(d^{\prime\prime}, d_s, L^{\prime\prime}, U)$ such that $\phi(\bm{X})[t] = \mathcal{N}(\bm{X})[t]$ for any $1\leq t\leq T$ and $\bm{X} \in [0, 1]^{d_\mathrm{in} \times T}$. It is because in $\mathcal{PRNN}$, both $\mathcal{P}$, $\mathcal{Q}$ and $\mathcal{F}$ operate position-wise and recurrent layers $\mathcal{R}_l$'s are past-dependent, and we can simply construct $\phi$ by using exactly the same weights from $\mathcal{N}$. Hence we can say that there exists a ParaRNN $\mathcal{N}  \in \mathcal{PRNN}^{(T)}_{d_\mathrm{in}, d_\mathrm{out}}(d^{\prime\prime}, d_s, L^{\prime\prime}, U)$ such that
\begin{align*}
	\sup_{\bm{X}_{T} \in [0, 1]^{d_\mathrm{in} \times T}}
	\norm{\mathcal{N}(\bm{X}_{T})[T]  - f^{(T)}(\bm{X}_{T})}_\infty \leq
		19 U (\lfloor \beta \rfloor + 1)^2 (d_\mathrm{in}T)^{\lfloor \beta \rfloor + (\beta  \vee 1) / 2} (JI)^{-2\beta / (d_\mathrm{in}T)}.
\end{align*}
\end{proof}

\subsection{Proof of Estimation Error Bound in Theorem \ref{proposition:estimation-error}}\label{subsec:estimation-proof}
This subsection includes the proof of estimation error bound in Theorem \ref{proposition:estimation-error}. As one of the terms in the upper bound involves the metric entropy of the function class $\mathcal{F}^{(T)}_{d_\mathrm{in}, 1, d, d_s, L, U}$. We begin with introducing the concept of uniform covering number, which is a measure for metric entropy of function classes. The uniform covering number for $\mathcal{F}^{(T)}_{d_\mathrm{in}, 1, d, d_s, L, U}$ is derived in Lemma \ref{lemma:covering-number}.

\begin{definition}[Uniform covering number]\label{def:covering-number} 
	For any fixed $\mathcal{X} = \{\bm{X}_{i, T}\}_{i=1}^N$, let $\mathcal{F}_{|\mathcal{X}}$ be the subset of $\mathbb{R}^{N}$ given by
	\begin{align*}
		\mathcal{F}_{|\mathcal{X}} = \left\{\left(f\left(\bm{X}_{1, T}\right),
		\dots, 
		f\left(\bm{X}_{N, T}\right)\right), f \in \mathcal{F}
		\right\}.
	\end{align*}
	Consider a metric 
	\begin{align*}
		d_{\mathcal{X}, \infty}(f, g) = \max_{1\leq i \leq N} \left\lvert 
		f\left(\bm{X}_{i, T}\right) - 
		g\left(\bm{X}_{i, T}\right) 
		\right\rvert.
	\end{align*}
	For a positive number $\delta$, we say that $\bar{\mathcal{F}}_{|\mathcal{X}} \subseteq\mathbb{R}^N$ is a $\delta$-cover for $\mathcal{F}_{|\mathcal{X}}$ under the metric $d_{\mathcal{X}, \infty}$ if $\bar{\mathcal{F}}_{|\mathcal{X}} \subseteq \mathcal{F}_{|\mathcal{X}}$ and for every $f \in  \mathcal{F}_{|\mathcal{X}}$ there is a $g \in \bar{\mathcal{F}}_{|\mathcal{X}}$ such that $d_{\mathcal{X}, \infty}(f, g) < \delta$. Then the $d_{\mathcal{X}, \infty}$ $\delta$-covering number of $\mathcal{F}_{|\mathcal{X}}$,  $\mathcal{N}(\delta, \mathcal{F}_{|\mathcal{X}}, d_{\mathcal{X}, \infty})$, is defined to be the minimum cardinality of a $d_{\mathcal{X}, \infty}$ $\delta$-cover for $\mathcal{F}_{|\mathcal{X}}$.
	Further, we define the \textit{uniform covering number} under the infinity norm as
	\begin{align*}
		\mathcal{N}_\infty(\delta, \mathcal{F}, N) = \max_{|\mathcal{X}| = N} \mathcal{N}(\delta, \mathcal{F}_{|\mathcal{X}}, d_{\mathcal{X}, \infty}).
	\end{align*}
\end{definition}

\begin{lemma}[Covering number of ParaRNN class] 
	\label{lemma:covering-number}
	Let $\mathcal{F}^{(T)}_\phi$ be the abbreviation of $\mathcal{F}^{(T)}_{d_\mathrm{in}, 1, d, d_s, L, U}$. For any fixed $\mathcal{X} = \{\bm{X}_{i, {T}}\}_{i=1}^N$, define a distance metric $$d_{\mathcal{X}, \infty} (f, g) =  \max_{1\leq i \leq N} \left\lvert 
	f\left(\bm{X}_{i, {T}}\right) - 
	g\left(\bm{X}_{i, {T}}\right) 
	\right\rvert.$$ Then we have
	\begin{align*}
		\log \mathcal{N}_\infty(\delta, \mathcal{F}_{\phi}^{(T)}, N) \lesssim d^2 L^2 \log \max \{d, L\} \log \frac{NU}{\delta}.
	\end{align*}
	 for $N \gtrsim {d}^2 {L}^2 \log \max\{d, L\}$.
\end{lemma}

\begin{proof}[Proof of Lemma \ref{lemma:covering-number}]
	This proof derives the covering number of ParaRNN class by using its equivalency with the FNN class. By Theorem 7 of \cite{bartlett2019nearly} and Theorem 12.2 of \cite{martin2009neural}, we have an upper bound of the covering number for the FNN class $\bar{\mathcal{F}}^{(T)} = \{ \bar{\mathcal{N}}(\bm{X}): \bar{\mathcal{N}} \in \mathcal{FNN}_{d_\mathrm{in}T, 1}(d^{(T)}, L^{(T)}), |\bar{\mathcal{N}}| \leq U \}$, 
	\begin{align*}
		\log \mathcal{N}_\infty(\delta, \bar{\mathcal{F}}^{(T)}, N) 
		\leq C {d^{(T)}}^2 {L^{(T)}}^2 \log \max\{d^{(T)}, L^{(T)}\}
		\log{ \frac{NU}{\delta}},
	\end{align*}
	for $N \gtrsim {d^{(T)}}^2 {L^{(T)}}^2 \log \max\{d^{(T)}, L^{(T)}\}$, where $C$ is some positive constant.
	Since by Proposition \ref{prop:equivalence-pararnn-fnn}(ii), we have $\mathcal{F}_{\phi}^{(T)} \subseteq  \mathcal{FNN}_{d_\mathrm{in}T, 1}((2t_0-1)d, (T + 1)L+2)$. Hence, with $d^{(T)} = (2T-1)d$ and $L^{(T)} = (T + 1)L+2$,
	\begin{align*}
	\log \mathcal{N}_\infty(\delta, \mathcal{F}_{\phi}^{(T)}, N) \leq
		\log \mathcal{N}_\infty(\delta, \bar{\mathcal{F}}^{(T)}, N) 
		{\lesssim d^2 L^2 \log \max \{d, L\} \log \frac{NU}{\delta}},
	\end{align*}
	for  $N \gtrsim {d}^2 {L}^2 \log \max\{d, L\}$.
\end{proof}

The following presents the proof of Theorem \ref{proposition:estimation-error}.

\begin{proof}[Proof of Theorem \ref{proposition:estimation-error}]
	For the simplicity of notations, the superscript ``$(T)$" is omitted from $\widehat{f}^{(T)}$, $\bar{f}^{(T)}$, ${f}_0^{(T)}$, and $\mathcal{F}_\phi^{(T)}$ whenever there is no confusion.
	
	Since $\widehat{f}$ is the empirical risk minimizer, we have
	\begin{align}
		\mathcal{R}_N(\widehat{f}) -  \mathcal{R}_N({f}_0) \leq  \mathcal{R}_N(\bar{f}) - \mathcal{R}_N({f}_0),
		\label{eq:ERM}
	\end{align}
	where $\bar{f} = \arg\min_{\phi \in \mathcal{F}_\phi} \mathcal{R}(\phi)$. Let $\xi$ be a positive number satisfying $1 < \xi \leq 2$. Take expectations on both sides of \eqref{eq:ERM} and multiply them by $\xi$, we have
	\begin{align*}
		\xi \cdot \mathbb{E}_\mathcal{S}[\mathcal{R}_N(\widehat{f}) -  \mathcal{R}({f}_0)] \leq  \xi \cdot  \left( \mathcal{R}(\bar{f}) - \mathcal{R}({f}_0) \right),
	\end{align*}
	and thus,
	\begin{align}
		\mathbb{E}_\mathcal{S}[\mathcal{R}(\widehat{f}) - \mathcal{R}(\bar{f})]
		\leq &  \mathbb{E}_\mathcal{S}[\mathcal{R}(\widehat{f}) - \mathcal{R}(\bar{f})]
		+ \xi \cdot  \left( \mathcal{R}(\bar{f}) - \mathcal{R}({f}_0) \right) - 
		\xi \cdot \mathbb{E}_\mathcal{S}[\mathcal{R}_N(\widehat{f}) -  \mathcal{R}({f}_0)] \notag\\ 
		\leq &
		(\xi - 1) \left[\mathcal{R}(\bar{{f}})  - \mathcal{R}({{f}}_0) \right]+ 
		\mathbb{E}_\mathcal{S}
		\left[
		\mathcal{R}(\widehat{f}) - \xi \mathcal{R}_N(\widehat{f}) + (\xi - 1) \mathcal{R}({f}_0)
		\right],\notag\\ \label{eq:estimation-error}
	\end{align}
	where the two terms correspond to approximation error and stochastic error, respectively. 
	
	We now proceed to bound the stochastic error term, the proof of which is inpired from Lemma 3.2 of \cite{jiao2023deep}. Recall that $\mathcal{S} = \{w_i = (\bm{X}_{i, T}, {z}_{i, T}), i=1, \dots, N \}$ is a random sample from the distribution of $(\bm{X}_T, {z}_T)$. Let $\tilde{\mathcal{S}} = \{\tilde{w}_i = (\tilde{\bm{X}}_{i, T}, \tilde{{z}}_{i, T}), i=1, \dots, N \}$ be another sample from the same distribution and independent of $\mathcal{S}$.
	Define $g(f, w_i) = (f(\bm{X}_{i, T}) - {z}_{i,T})^2 -  (f_0(\bm{X}_{i,T}) - z_{i,T})^2$ for any $f$ and sample $w_i$. Then the stochastic error can be rewritten as
	\begin{align}
		\mathbb{E}_\mathcal{S}
		\left[
		\mathcal{R}(\widehat{f}) - \xi \mathcal{R}_N(\widehat{f}) + (\xi - 1) \mathcal{R}({f}_0)
		\right]
		=
		\mathbb{E}_\mathcal{S}
		\left\{
		\frac{1}{N} \sum_{i=1}^N
		\left[
		-\xi g(\widehat{f}, w_i) + \mathbb{E}_{\tilde{\mathcal{S}}} g(\widehat{f}, \tilde{w}_i) 
		\right]
		\right\}, \\
		\label{eq:equality-g}
	\end{align}
	and we denote $G(f, w_i) := -\xi g({f}, w_i) + \mathbb{E}_{\tilde{\mathcal{S}}} g({f}, \tilde{w}_i)$ for the later use.
	
	To bound the stochastic error, the truncation technique is applied. We introduce a positive number $\beta_N \geq U \geq 1$ which may depend on the sample size $N$, and a truncation operator $\mathcal{T}_{\beta_N}$, where for any $z_t \in \mathbb{R}$, we have $$\mathcal{T}_{\beta_N} z_t = \begin{cases}
		z_t & \text{if } |z_t| \leq \beta_N \\ 
		\mathrm{sign}(z_t) \cdot \beta_N & \text{otherwise}
	\end{cases}.$$
	Let $f_{\beta_N}(\bm{X}_{T}) = \mathbb{E}\{\mathcal{T}_{\beta_N}z_{T} | \bm{x}_1, \dots, \bm{x}_{T} \}$ be the regression function corresponding to truncated response $\mathcal{T}_{\beta_N}z_{T}$. Define $g_{\beta_N}(f, w_i) := (f(\bm{X}_{i, T}) - \mathcal{T}_{\beta_N} z_{i, T})^2 -  (f_{\beta_N}(\bm{X}_{i, T}) - \mathcal{T}_{\beta_N} z_{i, T})^2$ and 
	$G_{\beta_N}(f, w_i) := -\xi g_{\beta_N}({f}, w_i) + \mathbb{E}_{\tilde{\mathcal{S}}} g_{\beta_N}({f}, \tilde{w}_i)$.
	Then by assuming a sub-exponential distribution on $z_{t}$'s in Assumption \ref{assump:Y} and using (A.4) in Lemma 3.2 of \cite{jiao2023deep}, we have
	\begin{align}
		\mathbb{E}_\mathcal{S} \left[
		\frac{1}{N} \sum_{i=1}^{N} G(\widehat{f}, w_i) 
		\right]
		\leq 
		\mathbb{E}_\mathcal{S} \left[
		\frac{1}{N} \sum_{i=1}^{N} G_{\beta_N}(\widehat{f}, w_i) 
		\right]
		+ c_1 \beta_N \exp(-\sigma_z \beta_N / 2), 
		\label{eq:truncation}
	\end{align}
	where $c_1$ is a constant not dependent on the sample size $N$ or the threshold $\beta_N$ but related to $\xi$.
	To bound the expectation term on the right side of \eqref{eq:truncation}, we first bound the corresponding tail probability:
	\begin{align}
		\mathbb{P}\left\{
		\frac{1}{N} \sum_{i=1}^{N} G_{\beta_N}(\widehat{f}, w_i)  > t
		\right\}
		\leq & 
		\mathbb{P}\left\{
		\exists f \in \mathcal{F}_\phi: 
		\frac{1}{N} \sum_{i=1}^{N} G_{\beta_N}({f}, w_i)  > t
		\right\} \notag \\
		= & 
		\mathbb{P}\left\{
		\exists f \in \mathcal{F}_\phi: 
		-\frac{\xi}{N} \sum_{i=1}^{N} g(\widehat{f}, w_i) + \mathbb{E}_{\tilde{\mathcal{S}}} g(\widehat{f}, \tilde{w}_i)  > t
		\right\} \notag \\
		\leq &
		14 \mathcal{N}_\infty \left(
		\frac{t}{40\xi \beta_N}, \mathcal{F}_\phi, N 
		\right) 
		\exp \left( - \frac{(\xi - 1) tN}{428 (2\xi - 1 )\xi^2 \beta_N^4}\right) \label{eq:gyorfi}.
	\end{align}
	The last inequality in \eqref{eq:gyorfi} holds by applying Theorem 11.4 of \cite{gyorfi2002distribution} with $\epsilon = (\xi - 1 )/ \xi $ and $\alpha = \beta = t / (2 \xi - 2)$; see definitions of $\epsilon, \alpha, \beta$ therein. Consequently, for $a_N > 0$,
	\begin{align*}
		& \mathbb{E}_\mathcal{S} \left[
		\frac{1}{N} \sum_{i=1}^{N} G_{\beta_N}(\widehat{f}, w_i) 
		\right] \\
		\leq & a_N + \int_{a_N}^{\infty}
		\mathbb{P}\left\{
			\frac{1}{N} \sum_{i=1}^{N} G_{\beta_N}(\widehat{f}, w_i)  > t
		\right\} \mathrm{d} t \\
		\leq &
		 a_N + \int_{a_N}^{\infty}
		14 \mathcal{N}_\infty \left(
		\frac{t}{40\xi \beta_N}, \mathcal{F}_\phi, N 
		\right) 
		\exp \left( - \frac{(\xi - 1) tN}{428 (2\xi - 1 )\xi^2 \beta_N^4}\right) 
		 \mathrm{d} t \\
		 \leq &
		 a_N + \int_{a_N}^{\infty}
		 14 \mathcal{N}_\infty \left(
		 \frac{a_N}{40\xi \beta_N}, \mathcal{F}_\phi, N 
		 \right) 
		 \exp \left( - \frac{(\xi - 1) tN}{428 (2\xi - 1 )\xi^2 \beta_N^4}\right) 
		 \mathrm{d} t \\
		 \leq &
		 a_N +
		  14 \mathcal{N}_\infty \left(
		 \frac{a_N}{40\xi \beta_N}, \mathcal{F}_\phi, N 
		 \right) 
		 \frac{428 (2\xi - 1 )\xi^2 \beta_N^4}{(\xi - 1) N}
		  \exp \left( - \frac{(\xi - 1)a_N N}{428 (2\xi - 1 )\xi^2 \beta_N^4}\right). 
	\end{align*}
	Choose $a_N = \log \left(14 \mathcal{N}_\infty \left(
	N^{-1}, \mathcal{F}_\phi, N 
	\right)\right) \cdot
	{428 (2\xi - 1 )\xi^2 \beta_N^4} / \left({(\xi - 1) N}\right)$, which leads to $\mathcal{N}_\infty (
	N^{-1},$ $ \mathcal{F}_\phi, N 
	) \geq \mathcal{N}_\infty (
	a_N / (40\xi \beta_N), \mathcal{F}_\phi, N 
	)$ as $a_N / (40\xi \beta_N) \geq 1 / N$. Hence we have
	\begin{align*}
	 \mathbb{E}_\mathcal{S} \left[
		\frac{1}{N} \sum_{i=1}^{N} G_{\beta_N}(\widehat{f}, w_i) 
		\right]  \leq 
		 \frac{428 (2\xi - 1 )\xi^2 \beta_N^4}{\xi - 1 }
		\cdot
		\frac{\log \left(14 \mathcal{N}_\infty (N^{-1}, \mathcal{F}_\phi, N)\right) + 1}{N}.
	\end{align*}
 	Setting $\xi = 3/2$ and $\beta_N = c_2 U \log N$ results in 
 	\begin{align*}
 		\mathbb{E}_\mathcal{S} \left[
 		\frac{1}{N} \sum_{i=1}^{N} G_{\beta_N}(\widehat{f}, w_i) 
 		\right]  \leq 
 		3852 c_2^4 U^4 (\log N)^4 \cdot
 		\frac{\log \left(14 \mathcal{N}_\infty (N^{-1}, \mathcal{F}_\phi, N)\right) + 1}{N},
 	\end{align*}
 	and then combined with \eqref{eq:equality-g} and \eqref{eq:truncation}, we have
 	\begin{align*}
 		\mathbb{E}_\mathcal{S}
 		\left[
 		\mathcal{R}(\widehat{f}) - \frac{3}{2}\mathcal{R}_N(\widehat{f}) + \frac{1}{2} \mathcal{R}({f}_0)
 		\right] \leq 
 		c_3 U^4  (\log N)^4  \frac{1}{N} \log \mathcal{N}_\infty (N^{-1}, \mathcal{F}_\phi, N).
 	\end{align*}
 	Further combined with \eqref{eq:estimation-error} and the upper bound of covering number in Lemma \ref{lemma:covering-number}, one can obtain
 	\begin{align*}
 		\mathbb{E}_\mathcal{S}[\mathcal{R}(\widehat{f}) - \mathcal{R}(\bar{f})]
 		\leq 
 		c_4 U^5  (\log N)^5  \frac{1}{N} d^2 L^2 \log \max\{d, L\}
 		+
 		\frac{1}{2} \left[\mathcal{R}(\bar{{f}})  - \mathcal{R}({{f}}_0) \right],
 	\end{align*}
 	for $N \gtrsim d^2 L^2 \log \max\{d, L\}$.
 	This finishes the proof.
\end{proof}

\subsection{Proof of Prediction Error Bounds in Theorem \ref{thm:prediction-error} and Corollary \ref{corollary:prediction-error}}
This subsection gives the proofs of prediction error bounds in Theorem \ref{thm:prediction-error} and Corollary \ref{corollary:prediction-error}. Note that Theorem \ref{thm:prediction-error} can be easily obtained by combining Theorems \ref{proposition:approx-rate-single-step} and \ref{proposition:estimation-error}. Corollary \ref{corollary:prediction-error} is a direct consequence of Theorem \ref{thm:prediction-error}.

\begin{proof}[Proof of Theorem \ref{thm:prediction-error}]
	Denote $\mathcal{F}_\phi^{(T)} = \mathcal{F}^{(T)}_{d_\mathrm{in}, 1, d, d_s, L, U}$ and $\bar{f}^{(T)}  = \arg \min_{\phi \in \mathcal{F}_\phi^{(T)}  }\mathcal{R}(\phi)$. Since $f_0^{(T)}$ is the minimizer of the $L_2$ risk $\mathcal{R}(\phi)$, it can be easily obtained that
	\begin{align*}
		 \mathbb{E}_\mathcal{S}[	\mathcal{R}(\widehat{f}^{(T)}) - \mathcal{R}(f_0^{(T)}) ]  
		= & \mathbb{E}_\mathcal{S}[	\mathcal{R}(\widehat{f}^{(T)}) -  \mathcal{R}(\bar{f}^{(T)})] + 
		\mathbb{E}_\mathcal{S}[ \mathcal{R}(\bar{f}^{(T)}) -  \mathcal{R}(f_0^{(T)})] \\
		\leq & 
		C U^5  (\log N)^5  \frac{1}{N} d^2 L^2 \log \max\{d, L\} + 
		\frac{3}{2} \left[ \mathcal{R}(\bar{f}^{(T)}) -  \mathcal{R}(f_0^{(T)})\right] 
	\end{align*}
	for $N \gtrsim  d^2 L^2 \log \max\{d, L\}$,
	where the last inequality holds by applying Theorem \ref{proposition:estimation-error}.
	
	Further, from the approximation error bound given in Theorem \ref{proposition:approx-rate-single-step}, we have 
	\begin{align*}
		\mathcal{R}(\bar{f}^{(T)}) -  \mathcal{R}(f_0^{(T)})
		= &  \inf_{\phi \in \mathcal{F}_\phi^{(T)}} 
		\mathbb{E}_{\bm{X}_{T}} \left[
		\left(\phi (\bm{X}_{T}) - f_0^{(T)}(\bm{X}_{T})\right)^2
		\right] \\
		\leq & 
		\inf_{\phi \in \mathcal{F}_\phi^{(T)}} 
		\left(
		\sup_{\bm{X}_{T} \in [0, 1]^{d_\mathrm{in} \times T}}
		\left\lvert \phi(\bm{X}_{T}) - f_0^{(T)}(\bm{X}_{T}) \right\rvert
		\right)^2 \\
		\leq &  361 U^2 (\lfloor \beta \rfloor + 1)^4 (d_\mathrm{in}T)^{2 \lfloor \beta \rfloor + (\beta  \vee 1)} (JI)^{-4\beta / (d_\mathrm{in}T)}. 
	\end{align*}
	Hence,
	\begin{align*}
		& \mathbb{E}_\mathcal{S}[	\mathcal{R}(\widehat{f}^{(T)}) - \mathcal{R}(f_0^{(T)}) ]  \\
		\leq & C U^5  (\log N)^5  \frac{1}{N} d^2 L^2 \log \max\{d, L\} + 
		542 U^2 (\lfloor \beta \rfloor + 1)^4 (d_\mathrm{in}T)^{2 \lfloor \beta \rfloor + (\beta  \vee 1)} (JI)^{-4\beta / (d_\mathrm{in}T)}.
	\end{align*}
\end{proof}

\begin{proof}[Proof of Corollary \ref{corollary:prediction-error}]
	From Theorem \ref{thm:prediction-error}, by plugging the width $d = 
	76 (\lfloor \beta \rfloor + 1)^2 3^{d_\mathrm{in}T} d_\mathrm{in}^{\lfloor \beta \rfloor + 2}$ $ T^{\lfloor \beta \rfloor + 1} J \lceil \log_2(8J)\rceil + d_s$ and depth 
	$L = {42( \lfloor \beta \rfloor + 1)^2 I \lceil \log_2(8I) \rceil} + 6d_\mathrm{in} T$ into \eqref{eq:expected-excess-risk} of the main paper, we have
	\begin{align}
		\mathbb{E}_\mathcal{S}[	\mathcal{R}(\widehat{f}^{(T)}) - \mathcal{R}(f_0^{(T)}) ] 
		\lesssim 
		\frac{1}{N} (\log N)^5 J^2I^2(\log J)^2(\log I)^2 \log\max \{J, I\} +
		(JI)^{-4\beta / (d_\mathrm{in}T)}.
	\end{align} 
	For any $\eta \in [0, {d_\mathrm{in}T} / (2d_\mathrm{in}T + 4\beta)]$, taking $J \asymp N^\eta$ and $I \asymp N^{ {d_\mathrm{in}T} / (2d_\mathrm{in}T + 4\beta) - \eta}$ will result in
	\begin{align*}
		\mathbb{E}_\mathcal{S}[	\mathcal{R}(\widehat{f}^{(T)}) - \mathcal{R}(f_0^{(T)}) ]  \lesssim N^{-\frac{2\beta}{d_\mathrm{in} T + 2\beta}} (\log N)^{10}.
	\end{align*}
\end{proof}

\section{More Details of Visualization and Experiments} \label{appendix-sec:exp}
\subsection{Details on Visualization of Recurrent Dynamics Patterns}
The recurrent dynamic patterns depicted in Figure \ref{fig:derivative}(a) are further elucidated in this subsection. The training procedure involves a two-layer vanilla RNN with a hidden size of 128 and a prediction window size of 48. We employ the Tanh activation function and use the WTH data set. Note that two recurrent matrices are trained, one for each layer. To showcase the dependence between two hidden states at a specific time lag $j$, we leverage the power of the corresponding eigenvalues. In Figure \ref{fig:derivative}(a), we select eigenvalues from the recurrent matrix $\bm{W}_h$ in the first layer to illustrate two elementary patterns, Type R and Type C, captured by the RNN model as the time lag $j$ increases.

\begin{table}[ht]
	\centering
	\caption{Hyperparameter settings: number of transformer layers and recurrent blocks in vanilla BRT and ParaBRT in different data sets.}
	\resizebox{0.65\textwidth}{!}
	{\begin{tabular}{c|c|c|c}
			\specialrule{1pt}{0.1pc}{0.1pc}
			Data set & Prediction window size & No. of layers & No. of blocks \\
			\specialrule{1pt}{0.1pc}{0.1pc}
			\multirow{4}{*}{ETTh$_1$} & 24 & 3 & 2 \\
			\cline{2-4}
			& 48 & 3 & 2 \\
			\cline{2-4}
			& 168 & 3 & 3 \\
			\cline{2-4}
			& 336 & 3 & 4 \\
			\specialrule{1pt}{0.1pc}{0.1pc}
			\multirow{4}{*}{ETTh$_2$} & 24 & 3 & 2 \\
			\cline{2-4}
			& 48 & 3 & 2 \\
			\cline{2-4}
			& 168 & 5 & 3 \\
			\cline{2-4}
			& 336 & 5 & 3 \\
			\specialrule{1pt}{0.1pc}{0.1pc}
			\multirow{4}{*}{ETTm$_1$ } & 24 & 3 & 2 \\
			\cline{2-4}
			& 48 & 3 & 2 \\
			\cline{2-4}
			& 96 & 3 & 3 \\
			\cline{2-4}
			& 288 & 3 & 3 \\
			\specialrule{1pt}{0.1pc}{0.1pc}
			\multirow{4}{*}{WTH} & 24 & 3 & 2 \\
			\cline{2-4}
			& 48 & 3 & 2 \\
			\cline{2-4}
			& 168 & 5 & 2 \\
			\cline{2-4}
			& 336 & 5 & 2 \\
			\specialrule{1pt}{0.1pc}{0.1pc}
			permuted MNIST & - & 3 & 2 \\
			\specialrule{1pt}{0.1pc}{0.1pc}
			Pixel-by-pixel CIFAR-10 & - & 4 & {2} \\
			\specialrule{1pt}{0.1pc}{0.1pc}
			Noise-padded CIFAR-10 & - & 4 & {2} \\
			\specialrule{1pt}{0.1pc}{0.1pc}
			EigenWorms & - & 4 & {18} \\
			\specialrule{1pt}{0.1pc}{0.1pc}
	\end{tabular}}
	\label{tab:setting}
\end{table}

\subsection{More Architecture Details}
\label{sec:appendix-brt}
The following is the hyperparameter settings for vanilla BRT and ParaBRT.
For all the settings of time series forecasting task, we set the hidden dimension of model, attention head and the state to 512, 64, and 512, respectively. The number of attention heads and states is set to 8 and 128, respectively. In the permuted sequential MNIST, pixel-by-pixel CIFAR-10, noise-padded CIFAR-10, and EigenWorms classification tasks, we adjust the above hyperparameters to 256, 32, 256, 8, and 32. The numbers of transformer layers and recurrent blocks are listed in Table \ref{tab:setting}. Note that we always insert the recurrent layer before the last transformer layer except for the EigenWorms classification task.

For other models in Section \ref{sec:result-mnist}, the number of recurrent layers is 2, 3, and 2 for permuted sequential MNIST, pixel-by-pixel CIFAR-10, and noise-padded CIFAR-10 tasks, respectively. All the hidden sizes are set to 128.

\subsection{Training Scheme in Experiments}
\label{sec:appendix-training}
For the first simulation in Section \ref{sec:result-simulation}, we divide the data set into train, validation, and test sets, each comprising 40,000, 10,000 and 10,000 samples, respectively.
The Adam optimizer is used for training, with the initial learning rate set as $0.001$ and dropped by a factor of 2 when the performance on the validation set no longer improves.
For the time series forecasting in Section \ref{sec:result-ts}, the train/validation/test split of ETT data set is 12/4/4 months, and that of WTH data set is 28/10/10 months.
For the image classification tasks in Section \ref{sec:result-mnist}, the MNIST and CIFAR-10 datasets are divided into 57,000/3,000/10,000 and 47,000/3,000/10,000 samples, respectively, as train, validation and test sets. 
The Adam optimizer is used with the learning rate initialized at 0.002, which is reduced by a factor of 2 when the validation loss no longer improves. The training is terminated if the learning rate becomes smaller than $10^{-6}$ or if the maximum number of epochs, set at 800, is reached.
For the adding problem in Section \ref{sec:appendix-adding}, we use the MSE as the objective function and the Adam optimizer with an initial learning rate of 0.002. The total number of training iterations is set to 60,000 and the learning rate decays by a factor of 10 every 20,000 iterations.

\subsection{Additional Results for Simulations}
\label{sec:appendix-simulation}

\begin{figure}[t]
	\centering
	\includegraphics[width=0.55 \linewidth]{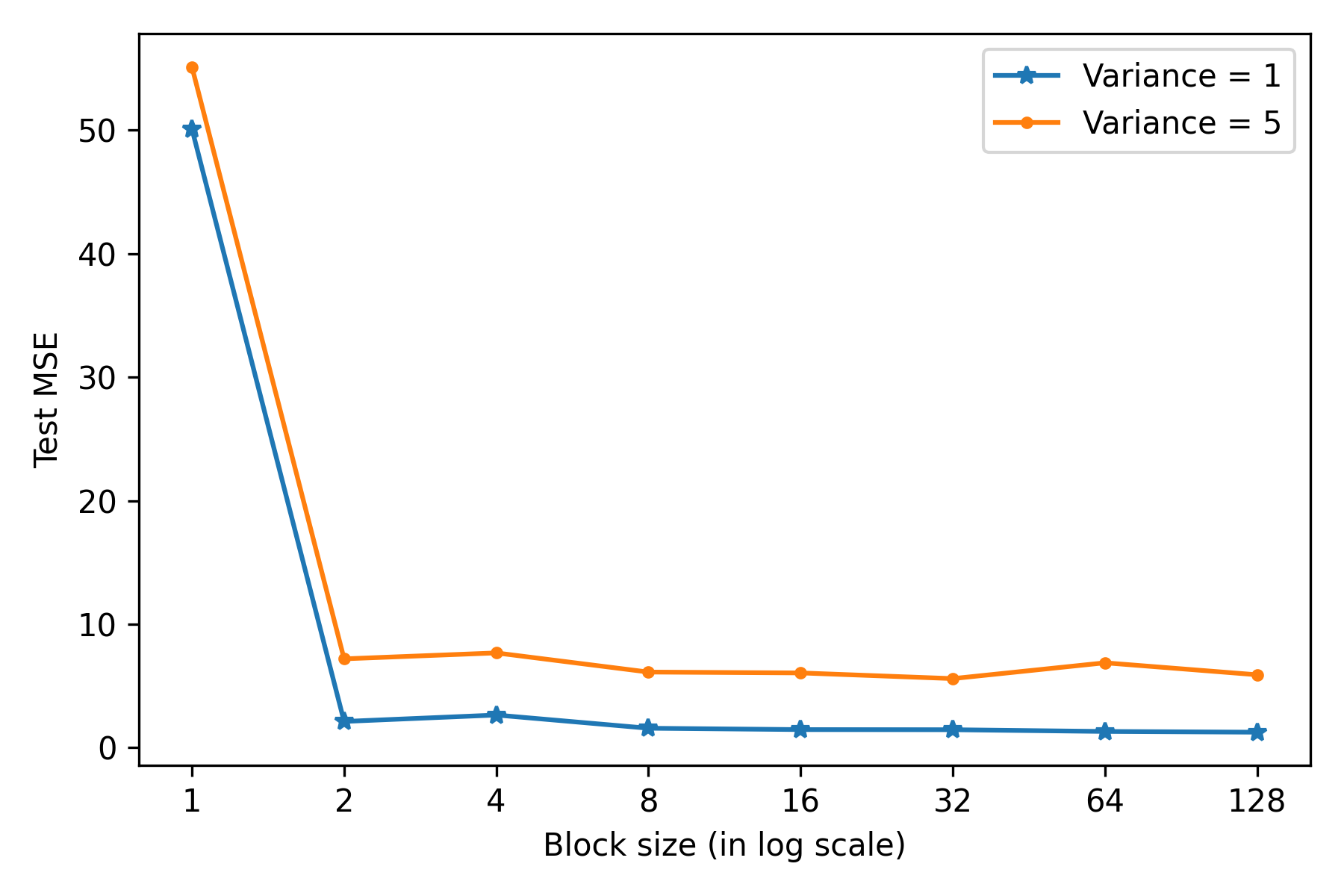}
	\caption{For the first simulation in Section \ref{sec:result-simulation} of the manuscript, the test MSEs averaged over 25 replicates are reported for two noise variances, 1 and 5.}
	\label{fig:mse}
\end{figure}
This subsection provides more results for the two simulations in Section  \ref{sec:result-simulation}. 
For the first simulation, to examine the robustness of our findings, we have also explored a larger noise variance, i.e., the elements of additive errors $\bm{\epsilon}_i$ are now independently sampled from $N(0, 5)$ instead of $N(0, 1)$. The simulation is rerun with the new variance, and the corresponding test mean squared errors (MSEs) averaged over 25 replicates are displayed in Figure \ref{fig:mse}. It can be observed that, while the test MSEs increase with larger noise variances, the overall trend remains the same: a substantial drop in test MSE arises when the block size $d_s$ alters from 1 to 2, whereas the change in performances for $d_s \geq 2$ is not significant. The robustness of our claim is thus verified. 

Furthermore, we also report the standard deviations of these test MSEs for both noise variance settings in Table \ref{tab:variance}. Overall, the standard deviations of test MSEs increase with the noise variance, aligning with our anticipations. The test MSEs for ParaRNNs with $d_s=1$ exhibit notably higher variations than other cases, possibly because they can only provide Type R-1 recurrence features and lose important features that capture periodic patterns. Hence they may struggle to learn the task, resulting in some extremely poor performances.

For the second simulation, we also report the standard deviations of the forward, backward and total time over 100 replications in Table \ref{tab:variance-time}. The corresponding mean values can be found in Figure \ref{fig:two}(c) of the manuscript. 

\begin{table}
	\centering
	\caption{Standard deviations of test MSEs over 25 replicates reported in Figure \ref{fig:mse}.}
	\renewcommand{\arraystretch}{0.8}
	\begin{tabular}{cccccccccc}
		\toprule
		\multirow{2}{*}{}&\multirow{2}{*}{} & \multicolumn{8}{c}{Block size} \\
		\cmidrule(lr){3-10}
		& & 1 & 2 & 4 & 8 & 16 & 32 & 64 & 128 \\
		\midrule
		\multirow{2}{*}{Variance} 
		& 1    & 101.00 & 1.53 & 2.86 & 0.59 & 0.52 & 0.57 & 0.40 & 0.29 \\
		& 5    & 107.50 & 2.93 & 4.01 & 1.39 & 1.87 & 0.95 & 6.31 & 1.78 \\
		\bottomrule
	\end{tabular}
	\label{tab:variance}
\end{table}

\begin{table}
	\centering
	\caption{Standard deviations of execution time over 100 replicates reported in Figure \ref{fig:two}(c) of the manuscript.}
	\vspace{3mm}
	\renewcommand{\arraystretch}{0.8}
	\begin{tabular}{ccccccccc}
		\toprule
		& \multicolumn{8}{c}{Block size} \\
		\cmidrule(lr){2-9}
		& 1 & 2 & 4 & 8 & 16 & 32 & 64 & 128 \\
		\midrule
		Forward propagation    & 0.09 & 0.12 & 0.09 & 0.14 & 0.18 & 0.24 & 0.29 & 0.41 \\
		Backward propagation    & 0.40 & 0.34 & 0.34 & 0.33 & 0.48 & 0.43 & 0.51 & 0.41 \\
		Total time   & 0.42 & 0.38 & 0.35 & 0.36 & 0.50 & 0.45 & 0.57 & 0.55 \\
		\bottomrule
	\end{tabular}
	\label{tab:variance-time}
\end{table}

\subsection{Additional Results for Time Series Forecasting}
\label{sec:appendix-ts}
Table \ref{tab:time-series} lists the mean squared error (MSE) and mean absolute error (MAE) obtained from 3 replicates. 

\begin{table}[t]
	\centering
	\caption{Multivariate long sequence time-series forecasting results using MSE and MAE as evaluation metrics. The results are reported on different prediction window sizes and averaged over 3 replicates. The better performances are in bold and the last row shows the winning counts for each model.}
	\resizebox{1.\textwidth}{!}{\begin{tabular}{c|c|cc|cc|cc|cc|cc|cc|cc|cc}
			\specialrule{1pt}{0.1pc}{0.1pc}  
			\multicolumn{2}{c|}{Methods}&\multicolumn{2}{c|}{Vanilla RNN}&\multicolumn{2}{c|}{ParaRNN}&\multicolumn{2}{c|}{Vanilla LSTM}&\multicolumn{2}{c|}{ParaLSTM}&\multicolumn{2}{c|}{Vanilla GRU}&\multicolumn{2}{c|}{ParaGRU}&\multicolumn{2}{c|}{Vanilla BRT}&\multicolumn{2}{c}{ParaBRT}\\
			\midrule
			\multicolumn{2}{c|}{Metric}&MSE&MAE&MSE&MAE&MSE&MAE&MSE&MAE&MSE&MAE&MSE&MAE&MSE&MAE&MSE&MAE\\
			\specialrule{1pt}{0.1pc}{0.1pc} 
			\parbox[t]{2mm}{\multirow{4}{*}{\rotatebox[origin=c]{90}{ETTh1}}}
			&	24	& 0.548 & 0.530 & \textbf{0.547} & \textbf{0.522}	& 0.642	& 0.602 & \textbf{0.621} & \textbf{0.580} &	0.513 &	0.527 &	\textbf{0.484 }&	\textbf{0.504} & 0.520 & 0.506 & \textbf{0.503} &	\textbf{ 0.496}\\
			&	48	& 0.527  & \textbf{0.519 }&	\textbf{0.525} &	0.521 & \textbf{0.598}	& \textbf{0.577}	& 0.616 & 0.595 &	\textbf{0.523 }&	0.526 &	0.529 &	\textbf{0.522} & \textbf{0.533} & \textbf{0.518} & 0.537 &	0.522\\
			&	168 & 0.752  & 0.662  &	\textbf{0.747} & \textbf{0.648} & 0.974	& 0.774	& \textbf{0.909 }& \textbf{0.739}	 &	0.794 &	0.684 &	\textbf{0.764}	&	\textbf{0.666}& \textbf{0.799}& 0.675 & 0.800 &	\textbf{0.671}\\
			&	336	& 0.989  & 0.790   & \textbf{0.900 }& \textbf{0.755	} & 1.120	& 0.841	& \textbf{0.925} & \textbf{0.761} & 0.967 &	0.789 &	\textbf{0.950} & \textbf{0.785 }& 1.004& \textbf{0.790}& \textbf{1.003} & 0.793\\
			
			\midrule
			\parbox[t]{2mm}{\multirow{4}{*}{\rotatebox[origin=c]{90}{ETTh2}}}
			&	24	&  0.595 & 0.598  &	\textbf{0.550} &	\textbf{0.569} &	0.966 &	0.786 & \textbf{0.804} &	\textbf{0.719 } &	0.614 & 0.633 &	\textbf{0.516} &	\textbf{0.549}& 0.954& 0.778& \textbf{0.891}& \textbf{0.758}	\\
			&	48	& 1.144  & 0.852  &	\textbf{1.025 }&	\textbf{0.826} & 1.437	& 0.953	& \textbf{1.322} & \textbf{0.949}	 &	1.149 &	\textbf{0.868} &	\textbf{1.119}	&	0.873& 1.466 &	0.987 & \textbf{1.374}&\textbf{0.956}\\
			&	168 & 4.599  & 1.787  &	\textbf{3.325 }& \textbf{1.519} & 3.428	& 1.578	&\textbf{ 3.399} &	 \textbf{1.515 }&	3.347 & 1.542 &	\textbf{2.478}	& \textbf{1.239}& \textbf{2.852} & 1.311	& 2.873 &\textbf{1.310}\\
			&	336	& 3.576  & 1.618  &	\textbf{2.772} &	\textbf{1.420} & 3.025	& 1.523	&  \textbf{2.856}&	\textbf{1.445} &	2.598 &	 1.375 & \textbf{2.354}	& \textbf{1.278}& \textbf{2.550} & \textbf{1.239} & 2.599&1.271\\
			
			\midrule
			\parbox[t]{2mm}{\multirow{4}{*}{\rotatebox[origin=c]{90}{ETTm1}}}
			&	24	& 0.780  &  0.581 &\textbf{	0.766} & \textbf{0.577}	& 0.814	& 0.624	& \textbf{0.756 }&	\textbf{0.582 } &	0.753 &	0.585 &	\textbf{0.674} &	\textbf{0.557} &0.778 &0.601	& \textbf{0.745}&\textbf{0.585}\\
			&	48	&  0.972 &  0.698 & \textbf{0.924 }&	\textbf{0.697} & \textbf{0.968}	& 0.719	& \textbf{0.968} &	\textbf{0.699} &	0.768 & 0.623 &	\textbf{0.744 }&	\textbf{0.602}& 0.581&	\textbf{0.554} & \textbf{0.578} &\textbf{0.554}\\
			&	96 & 0.570  & 0.548  &	\textbf{0.519} &\textbf{0.509} & 0.605	&	0.570 & \textbf{0.535 }&	 \textbf{0.526} &	0.510 &	0.529 &	\textbf{0.428}	&	\textbf{0.458}&0.869 &0.739& \textbf{0.786}&\textbf{0.694}	\\
			&	228	&  0.684 &  0.623 & \textbf{0.654} &	\textbf{0.607} & 0.965	&	0.788 & \textbf{0.770} & \textbf{0.679 }&	0.770 &	0.691 &	\textbf{0.739} &	\textbf{0.671}	&\textbf{1.034} & 0.821& 1.043&\textbf{0.819}\\
			
			\midrule
			\parbox[t]{2mm}{\multirow{4}{*}{\rotatebox[origin=c]{90}{Weather}}}
			&	24	& \textbf{0.365}  &  \textbf{0.411} &	 0.368 & \textbf{0.411	}&	0.376 &	0.420 & \textbf{0.374} &	 \textbf{0.417} &	\textbf{0.360}&	0.407&	\textbf{0.360}	&\textbf{0.403}	&\textbf{0.384} &	\textbf{0.428} &	\textbf{0.384}	&	\textbf{0.428}	\\
			&	48	& \textbf{0.479}   & \textbf{0.492}  & 0.481	 &	\textbf{0.492} &0.489	&	0.499& \textbf{0.484}&	\textbf{0.494 }& 0.468	& \textbf{0.484	}&	\textbf{0.464}	&	\textbf{0.484}&	\textbf{0.486 }&	\textbf{0.497} &	\textbf{0.486}	&	0.498	\\
			&	168 &  0.569 & 0.558 &	 \textbf{0.564}&	\textbf{0.556}&	\textbf{0.570}& \textbf{0.558}	& 0.571&	 0.559&	\textbf{0.549} &	\textbf{0.546} &		0.556&	0.552&	\textbf{0.742}&	\textbf{0.641} &		0.750 &	0.645	\\
			&	336	& \textbf{0.582}  & \textbf{0.565}  &	 0.583 & 0.567	&	\textbf{0.583}&	\textbf{0.565} & 0.586& 0.566 &	\textbf{0.580}&	\textbf{0.563}&		0.583&	0.565&	0.842 &	0.694 &	\textbf{0.829}	&	\textbf{0.686}	\\
			
			\specialrule{1pt}{0.1pc}{0.1pc}  
			\multicolumn{2}{c|}{count}&\multicolumn{2}{c|}{7}&\multicolumn{2}{c|}{27}&\multicolumn{2}{c|}{7}&\multicolumn{2}{c|}{26}&\multicolumn{2}{c|}{8}&\multicolumn{2}{c|}{26}&\multicolumn{2}{c|}{15}&\multicolumn{2}{c}{21}\\
			\specialrule{1pt}{0.1pc}{0.1pc} 
	\end{tabular}}
	\label{tab:time-series}
\end{table} 

\subsection{Additional Results for Adding Problem}
\label{sec:appendix-adding}
The adding problem \citep{li2018indrnn} serves as a standard benchmark for assessing the effectiveness of RNNs. It takes two sequences of length $T = 100$ or $500$ as inputs and makes a single prediction based on the sequences. Specifically, the first sequence is generated by uniformly sampling values from the range $(0,1)$. The second sequence consists of two entries being 1 (one in the first half and the other in the second half) and the rest entries being 0. The desired output is the sum of the two entries in the first sequence that correspond to the two entries of 1 in the second sequence. 

Considering that the traditional RNN with the Tanh activation function is known to be ineffective for this task \citep{li2018indrnn}, we exclude it from our comparison and focus solely on comparing ParaLSTM and ParaGRU with their respective counterparts.
We use only one recurrent layer and a total hidden size of 128. 
Ten replications with different data and random initializations have been conducted.
The visualizations of test MSEs for LSTM models with input sequence length $T=100$ and $500$ can be found in Figures \ref{fig:lstm100-all} and \ref{fig:lstm500-all}, respectively; and those for GRU models are provided in Figures \ref{fig:gru100-all} and \ref{fig:gru500-all}. 
Both ParaLSTM and ParaGRU converge to a similar MSE as their baseline models in most cases, regardless of the sequence length. There are some exceptions in Figure \ref{fig:lstm500-all}, where the vanilla LSTM fails to converge while our model rapidly approaches near-zero test MSEs.
Besides, ParaLSTM requires fewer iterations to converge while ParaGRU uses a larger number of iterations compared to their respective vanilla versions. Note that if considering the overall execution time for convergence which depends on both the number of iterations to converge and the execution time per iteration, our models demonstrate advantages due to its parallel implementation and faster speed. 

\begin{figure}
	\centering
	\includegraphics[width=0.9\linewidth]{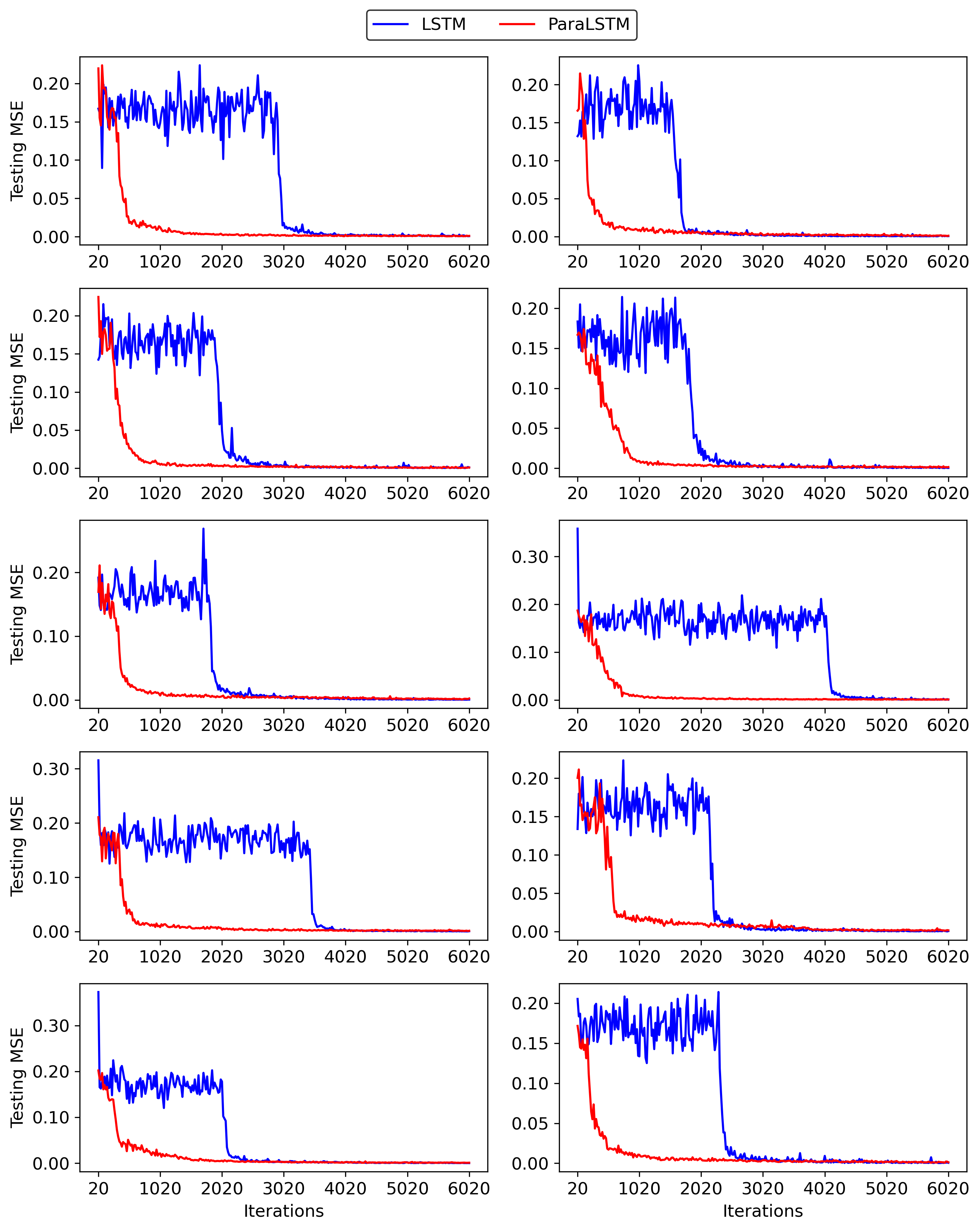}
	\caption{MSEs of ten replicates for LSTMs with time step $T=100$ in the adding problem. The blue and red lines represent the vanilla and our models, respectively.}
	\label{fig:lstm100-all}
\end{figure}

\begin{figure}
	\centering
	\includegraphics[width=0.9\linewidth]{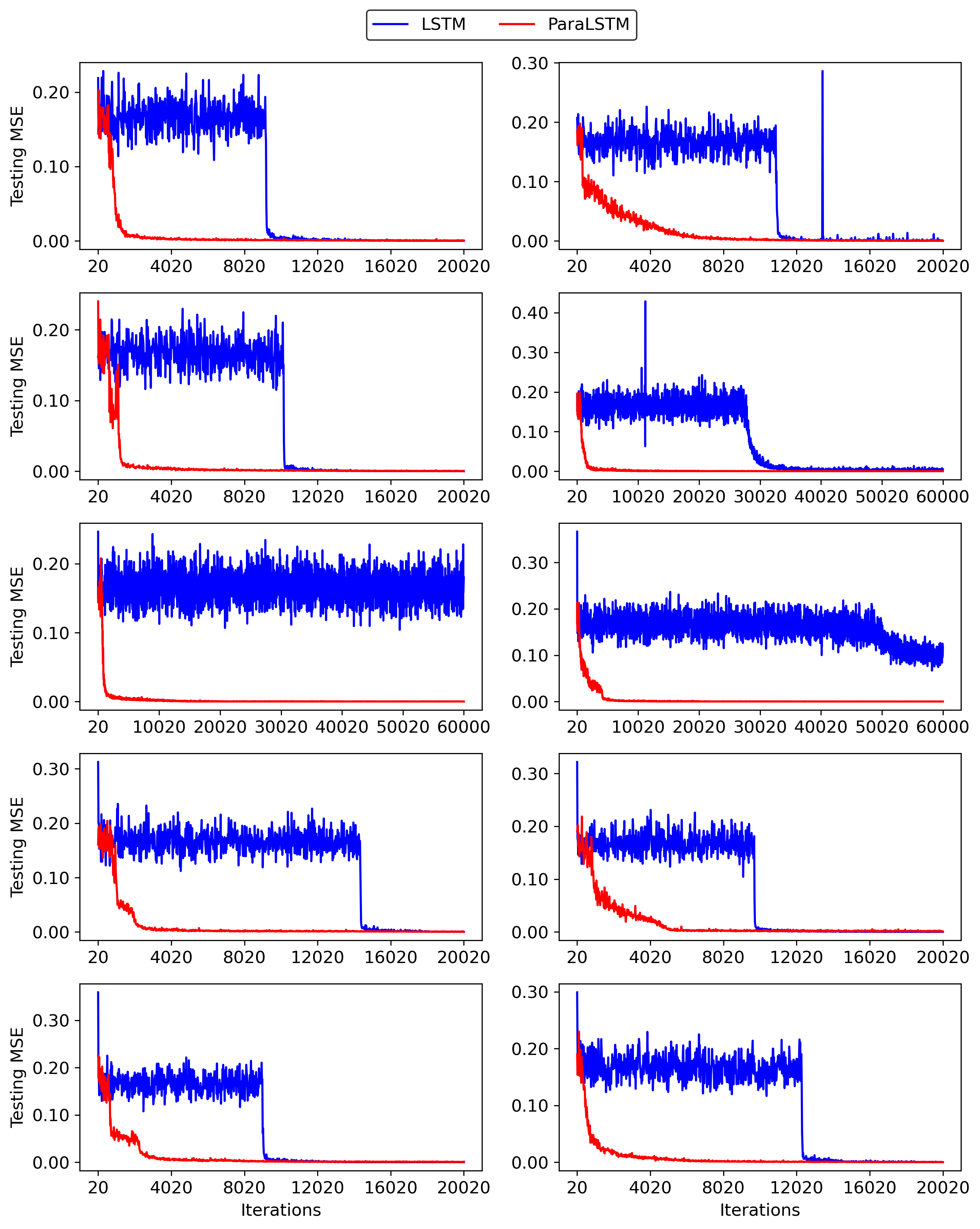}
	\caption{MSEs of ten replicates for LSTMs with time step $T=500$ in the adding problem. The blue and red lines represent the vanilla and our models, respectively. Vanilla LSTM fails to converge after 60,000 iterations in two replications, as shown in the third row.}
	\label{fig:lstm500-all}
\end{figure}

\begin{figure}
	\centering
	\includegraphics[width=0.9\linewidth]{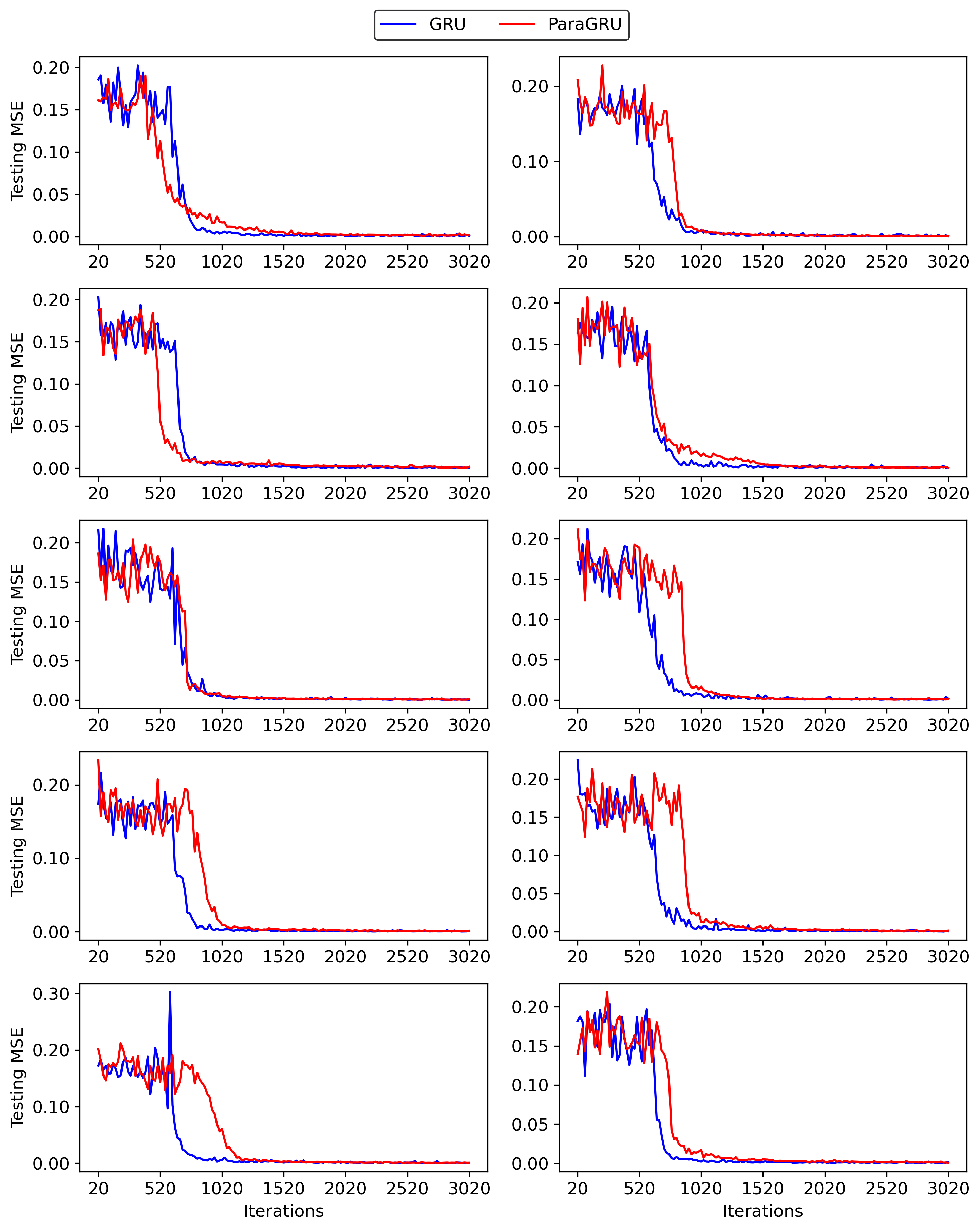}
	\caption{MSEs of ten replicates for GRUs with time step $T=100$ in the adding problem. The blue and red lines represent the vanilla and our models, respectively.}
	\label{fig:gru100-all}
\end{figure}

\begin{figure}
	\centering
	\includegraphics[width=0.9\linewidth]{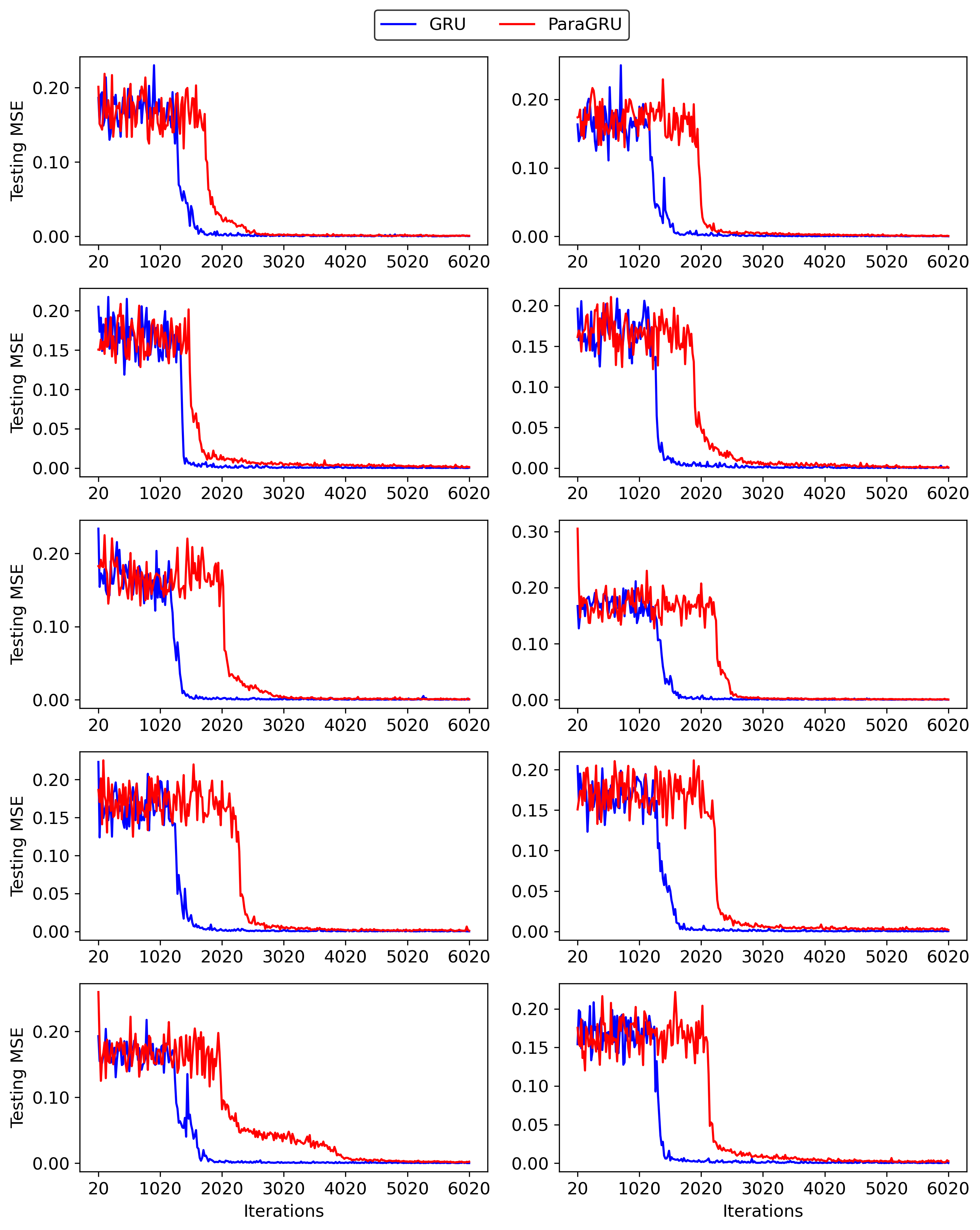}
	\caption{MSEs of ten replicates for GRUs with time step $T=500$ in the adding problem. The blue and red lines represent the vanilla and our models, respectively.}
	\label{fig:gru500-all}
\end{figure}

\setlength{\bibsep}{5.5pt}

\end{document}